\documentclass{article}

\usepackage{microtype}
\usepackage{graphicx}
\usepackage{subfigure}
\usepackage{booktabs} 
\usepackage{makecell}
\usepackage{threeparttable}
\usepackage{multirow} 
\usepackage{hyperref}

\usepackage[accepted]{icml2024}

\usepackage{amsmath}
\usepackage{amssymb}
\usepackage{mathtools}
\usepackage{amsthm}

\usepackage[capitalize,noabbrev]{cleveref}

\def\argmin{\mathop{\rm argmin}}

\def\Var{\mathop{\rm Var}}

\def\bfx{\mathop{\bf x}}
\def\bfx{\mathop{\bf X}}

\def\bfx{\boldsymbol{x}}
\def\bfX{\boldsymbol{X}}
\def\bfu{\boldsymbol{u}}
\def\bfz{\boldsymbol{z}}

\newcommand {\bfphi} {\mbox{\boldmath $\phi$}}

\newcommand {\bfbeta} {\mbox{\boldmath $\beta$}}

\newcommand {\bfomega} {\mbox{\boldmath $\omega$}}

\theoremstyle{plain}
\newtheorem{theorem}{Theorem}[section]
\newtheorem{proposition}[theorem]{Proposition}
\newtheorem{lemma}[theorem]{Lemma}
\newtheorem{corollary}[theorem]{Corollary}
\theoremstyle{definition}
\newtheorem{definition}[theorem]{Definition}
\newtheorem{assumption}[theorem]{Assumption}
\newtheorem{example}[theorem]{Example}
\theoremstyle{remark}
\newtheorem{remark}[theorem]{Remark}
\usepackage[textsize=tiny]{todonotes}

\icmltitlerunning{Optimal Kernel Quantile Learning with Random Features}

\begin{document}

\twocolumn[
\icmltitle{Optimal Kernel Quantile Learning with Random Features}



\icmlsetsymbol{equal}{*}

\begin{icmlauthorlist}
\icmlauthor{Caixing Wang}{sch}
\icmlauthor{Xingdong Feng}{sch}
\end{icmlauthorlist}

\icmlaffiliation{sch}{School of Statistics and Management, Shanghai University of Finance and Economics}

\icmlcorrespondingauthor{Xingdong Feng}{feng.xingdong@mail.shufe.edu.cn}

\icmlkeywords{Kernel methods, learning theory, quantile regression, Lipschitz continuous, agnostic learning}

\vskip 0.3in
]



\printAffiliationsAndNotice{}  

\begin{abstract}
The random feature (RF) approach is a well-established and efficient tool for scalable kernel methods, but existing literature has primarily focused on kernel ridge regression with random features (KRR-RF), which has limitations in handling heterogeneous data with heavy-tailed noises. This paper presents a generalization study of kernel quantile regression with random features (KQR-RF), which accounts for the non-smoothness of the check loss in KQR-RF by introducing a refined error decomposition and establishing a novel connection between KQR-RF and KRR-RF. Our study establishes the capacity-dependent learning rates for KQR-RF under mild conditions on the number of RFs, which are minimax optimal up to some logarithmic factors. Importantly, our theoretical results, utilizing a data-dependent sampling strategy, can be extended to cover the agnostic setting where the target quantile function may not precisely align with the assumed kernel space. By slightly modifying our assumptions, the capacity-dependent error analysis can also be applied to cases with Lipschitz continuous losses, enabling broader applications in the machine learning community. To validate our theoretical findings, simulated experiments and a real data application are conducted.  
\end{abstract}

\section{Introduction}\label{sec:intr}
Kernel methods are pivotal in statistical analysis and have seen extensive applications in nonparametric regression \citep{wahba1990spline,vapnik1999nature} and classification \citep{scholkopf2002learning,steinwart2008support}. Despite their empirical success, typical kernel algorithms struggle with large-scale datasets, hindered by substantial computational cost, scaling as $\mathcal{O}(|D|^3)$, and considerable storage requirements, scaling as $\mathcal{O}(|D|^2)$, where $|D|$ is the sample size of the dataset. This limitation has motivated significant research efforts towards scalable kernel methods, such as distributed learning \citep{zhang2015divide,lin2017distributed,lin2020distributed,lian2022}, Nystr{\"o}m subsampling \citep{williams2001using,rudi2015less,li2023optimal}, random features \citep{rahimi2007random,rahimi2008weighted,rudin2017fourier,rudi2017generalization}, stochastic gradient methods \citep{lin2016optimal,lin2020optimal}, Falkon \citep{rudi2017falkon} and EigenPro 3.0 \citep{abedsoltan2023toward}. 	

Among these popular accelerated methods, random features \citep{rahimi2007random} is a kernel approximation technique that maps the attribute space to a finite and low-dimensional space through Bochner's theorem \citep{bochner2005harmonic,rudin2017fourier}. Recent attention in theoretical analysis has been directed toward kernel methods employing random features \citep{sutherland2015error,rudi2017generalization,avron2017faster,bach2017equivalence,carratino2018learning}. In the context of kernel ridge regression with random features, significant efforts have focused on achieving optimal learning rates \citep{rudi2017generalization,avron2017faster,bach2017equivalence,lioptimalRF2023,liuon2023}, aligning with the minimax optimal rates of exact KRR \citep{caponnetto2007optimal} under mild conditions on the number of random features. However, it is worth pointing out that these works mainly focus on the least square loss which exclusively estimates the conditional mean of the response given the covariate variables. Broader investigations encompass losses that are Lipschitz continuous, as in support vector machine (SVM) and logistic regression \citep{bach2017equivalence,sun2018but,li2021towards}. It is worth noting that their statistical guarantees are capacity-independent and rely on the stringent assumption that the true target function $f_{\rho}$ lies in the assumed kernel space, i.e., $f_{\rho}\in {\cal H}_K$, known as the realizable setting. However, the agnostic setting, where the true target function $f_{\rho}$ is not in the assumed kernel space, i.e., $f_{\rho}\notin {\cal H}_K$, is more common in practice. \textit{This leads to a motivating question: can the capacity-dependent optimal rates for some general losses using random features be maintained in the agnostic settings?}

This paper primarily investigates the statistical and computational trade-offs inherent in random feature approximation for nonparametric quantile regression within a reproducing kernel Hilbert space (RKHS),  also known as \textit{kernel quantile regression with random features (KQR-RF)}. In contrast to KRR-RF, KQR-RF models the entire conditional quantiles of the response, enhancing robustness against outliers and handling data heterogeneity more effectively \citep{koenker2005quantile,takeuchi2006nonparametric,li2007quantile,lian2022,wang2024communication}. Our objective is to establish the capacity-dependent optimal rates for KQR-RF applicable to both realizable and agnostic settings. Based on this special check loss, we extend the theoretical framework to a broader family of Lipschitz continuous loss functions. This expansion encompasses various commonly employed methodologies as specialized instances, including mean regression, quantile regression, likelihood-based classification, and margin-based classification.

\subsection{Our Contributions}
The main contributions of this paper are multi-folds.

\noindent \textbf{Comprehensive theoretical assessments.} We propose a comprehensive theoretical analysis of KQR-RF,  offering deep insights into the impact of random features on kernel quantile learning. To the best of our knowledge, this is the first work to provide generalization analysis for random features in kernel quantile learning. Moreover, the optimal learning rates we derived can be directly extended to the general Lipschitz loss functions. Compared to the existing results, which are either capacity-independent \citep{rahimi2008weighted,bach2017equivalence,li2021towards} or suboptimal \citep{sun2018but}, we provide the capacity-dependent optimal learning rates (Theorem \ref{thm2}) for KQR-RF (and RF for Lipschitz loss) for both realizable and agnostic settings. 

\noindent \textbf{Efficient computational improvement.} For both uniform sampling and data-dependent sampling schemes, we obtain, to the best of our knowledge, the minimum number of random features required for the optimal learning rates in the literature. Specifically, we reduce the number of random features from ${\cal O}(|D|), r=1/2$ \citep{rahimi2008weighted,li2021towards} to ${\cal O}(|D|^{\frac{1}{2r+\gamma}}) \vee {\cal O}(|D|^{\frac{(2r-1)\gamma+1}{2r+\gamma}}),~r \in (0,1]$ (Theorem \ref{thm1}) for the uniform sampling scheme; and ${\cal O}(|D|^{\frac{2\gamma}{2\gamma+1}}),~r=1/2$ \citep{sun2018but} to ${\cal O}(|D|^{\frac{\gamma}{2r+\gamma}}) \vee {\cal O}(|D|^{\frac{2r+\gamma-1}{2r+\gamma}}), ~r \in (0,1]$ (Corollary \ref{cor1}) for the data-dependent sampling scheme. Here, $|D|$ is the sample size, $r$, and $\gamma$ are some key parameters defined in Section \ref{sec:the}. The improvement notably enhances computational efficiency.

\noindent \textbf{Novel proof skills.} In contrast to existing related work on KRR and its RF variants  \citep{caponnetto2007optimal,rudi2017generalization,lioptimalRF2023}, the estimator of KQR-RF (random features for Lipschitz loss) lacks an explicit solution, posing challenges in deriving the learning rates. In our proof, we first provide a novel error decomposition including a least square approximation (LS-approximation) error term (Lemma \ref{lemA.1.1}). Leveraging the empirical process and a self-calibration assumption, we successfully establish a connection between the KQR-RF estimator $f_{M,D,\lambda}$ and its KRR-RF approximation $f_{M,D,\lambda}^{\diamond}$. The theoretical extension to the regularity setting when $r\in (0,1]$ is achieved by using the nontrivial Cauthy-Schwarz and Young's inequalities, along with sharper estimates for the differences between the operators. A more detailed proof sketch will be provided in Section \ref{sec:com}.
    
\noindent \textbf{Diverse numerical verification.} Another contribution of this work is the exploration of KQR-RF's efficacy across diverse synthetic and real-world examples, further validating the theoretical findings in this paper.

\subsection{Related work}
Some most related works are presented below.

\noindent \textbf{Random features approximation.} \citet{rahimi2007random,sutherland2015error,sriperumbudur2015optimal} have investigated the approximation error between the approximated kernel function $K_M(\bfx,{\bfx}^{\prime})$ and the original kernel Gram-matrix $K(\bfx,{\bfx}^{\prime})$, requiring ${\cal O}(|D|)$ features to guarantee the accuracy of the approximation, i.e., $\sup_{\bfx, {\bfx}^{\prime}}|K_M(\bfx,{\bfx}^{\prime})-K(\bfx,{\bfx}^{\prime})| = {\cal O}(|D|^{-1/2})$. Another line of research delves into the generalization properties of random features in various supervised learning tasks, such as kernel ridge regression \citep{bach2017equivalence,avron2017faster,rudi2017generalization}, kernel support vector machine (KSVM) \citep{sun2018but}, and kernel learning with Lipschitz loss \citep{rahimi2008weighted,li2021towards,li2022sharp}.  However, the success of these works depends on the realizable setting where the true function satisfies $f_{\rho} \in {\cal H}_K$.

\noindent \textbf{Agnostic kernel learning.} Recent studies have established the capacity-dependent optimal learning rates in the agnostic kernel learning, such as KRR \citep{smale2007learning,zhang23mis}, along with its variations including random features \citep{lioptimalRF2023,Liijcai2023} and Nystr{\"o}m subsampling \citep{li2023optimal}. However, these studies primarily concentrate on the least square loss, while our focus lies on the KQR-RF with a non-smooth check loss function (and Lispschitz loss functions), posing additional challenges due to the lack of explicit solutions (refer to Section \ref{sec:com} for a detailed discussion). 

\noindent \textbf{Data-dependent sampling.}
Data-dependent sampling based on an importance ratio was initially introduced by \citet{alaoui2015fast} for Nystr{\"o}m subsampling and has been integrated into random features \citep{bach2017equivalence,avron2017faster,rudi2017generalization,li2021towards}, facilitating faster learning rates with fewer random features. Despite its computational efficiency, there remains an open question regarding its impact on the theoretical results for KQR-RF (and RF for Lipschitz loss), particularly in the agnostic settings.

\section{Methodology}\label{sec:met}

\subsection{Preambles}
Consider a standard supervised learning problem that we have a sample $D=\{({\bfx}_i,y_i)\}_{i=1}^{|D|}$, which are the independent copies of $(\bfx, y)$ drawn from an unknown joint distribution $\rho(\bfx,y)$ over $\cal X \times \mathbb{R}$. The $\tau$-th conditional quantile of the response $y$ is the minimizer of the expected quantile risk across all measurable functions, given by:
\begin{align}\label{expected_risk}
f_{\tau}^*=\argmin_{f \in L^2_{\rho_{\cal X}}} \int_{\cal X \times \mathbb{R}}\rho_{\tau}\big(y-f(\bfx)\big)d\rho(\bfx,y), 
\end{align}
where $\rho_{\tau}(u)=u(\tau-I\{u\leq 0\})$ denotes the check loss function at $\tau$-th quantile level with the indicator function $I(\cdot)$, and $L^2_{\rho_{\cal X}}=\{f: {\cal X} \rightarrow \mathbb{R} | \int_{\cal X}f^2(\bfx)d\rho_{\cal X}< \infty \}$ is the space of square integral function with respect to the distribution of the covariates $\rho_{\cal X}$. We also denote the  $L^2_{\rho_{\cal X}}$-norm of $f$ as $\|f\|_{\rho}^2=\langle f, f\rangle_{\rho}=\int_{\cal X}f^2(\bfx)d\rho_{\cal X}$ for $ f \in L^2_{\rho_{\cal X}}$. From the definition of quantile regression model, we have $P(\epsilon\leq 0|\bfx)=\tau$, where $\epsilon=y-f^*_{\tau}(\bfx)$ is the noise term.

\subsection{Kernel Quantile Regression}
Kernel methods are one of the most powerful nonparametric tools, particularly for learning predictive functions within an RKHS \citep{vapnik1999nature}. Let ${\cal H}_K$ denotes the RKHS induced by a symmetric, positive and semi-definite kernel function $K:{\cal X} \times {\cal X} \rightarrow \mathbb{R}$, and we define its equipped norm as $\|\cdot\|^2_K={\langle \cdot, \cdot\rangle_K}$  with the endowed inner product $\langle\cdot,\cdot\rangle_K$. 

KQR estimates a function in the RKHS ${\cal H}_K$ by minimizing the check loss function combined with a penalty based on the squared Hilbert norm
\begin{align}\label{loss_function}
f_{D,\lambda}=\underset{f \in {\cal H}_K}{\operatorname{argmin}}\frac{1}{|D|} \sum_{i=1}^{|D|}\rho_{\tau}\big(y_i-f(\bfx_i)\big)+\lambda\|f\|_K^2  ,
\end{align}
where $|D|$ is the cardinality of $D$ and $\lambda$ is the regularization parameter. According to the representer theorem \citep{wahba1990spline}, the solution of this optimization task \eqref{loss_function} is of finite form as given by $f_{D,\lambda}(\bfx)=\sum_{i=1}^{|D|}\alpha_i K(\bfx,{\bfx}_i)=\boldsymbol{\alpha}^T\mathbf{K}_N({\bfx})$, where $\boldsymbol{\alpha}=(\alpha_1,\ldots,\alpha_{|D|})^{T} \in \mathbb{R}^{|D|}$ are the representer coefficients and $\mathbf{K}_N({\bfx})=(K(\bfx_1,\bfx),\ldots,K(\bfx_{|D|},\bfx))^{T} \in \mathbb{R}^{|D|}$. With this solution plugged into \eqref{loss_function}, the optimization problem can be reformulated as 
\begin{align*}
\hat{\boldsymbol{\alpha}}=\argmin_{\boldsymbol{\alpha} \in \mathbb{R}^{|D|}}\frac{1}{|D|} \sum_{i=1}^{|D|}\rho_{\tau}\big(y_i-\boldsymbol{\alpha}^T\mathbf{K}_N({\bfx_i})\big)+\lambda  \boldsymbol{\alpha}^T\mathbf{K}\boldsymbol{\alpha},
\end{align*}
where $\mathbf{K}=\{K(\bfx_i,\bfx_j)\}_{i,j=1}^{|D|}$ is the Gram matrix. In literature, this problem can be solved by dual optimization \citep{takeuchi2006nonparametric,feng2024towards}, path-following algorithm \citep{li2007quantile}, and ADMM algorithm \citep{boyd2011distributed,wang2024communication}. However, its scalability for large datasets is limited due to the expensive computational complexity $({\cal O}(|D|^3)$ and storage requirements $({\cal O}(|D|^2)$ when $|D|$ is large. Consequently, a surge in research investigating scalable kernel methods and analyzing their performance has surfaced \citep{lin2017distributed,rudi2015less,rudi2017generalization,li2021towards}. 

\subsection{KQR with Random Features}
Random features prove advantageous in kernel approximation. Assuming the kernel $K$ has an integral representation,
\begin{align}\label{in_re}
K(\bfx, \bfx^{\prime})=\int_{\Omega}\phi(\bfx,\bfomega)\phi(\bfx^{\prime},\bfomega)d\pi(\bfomega),
\end{align}
for any $\bfx,\bfx^{\prime} \in \cal X$, where $(\Omega,\pi)$ is a probability space and $\phi: {\cal X} \times \Omega \rightarrow \mathbb{R}$, 
it is thus clear that we can adopt the standard Monte Carlo sampling method \citep{rahimi2007random} to estimate 
$K(\bfx, \bfx^{\prime})$ by 
\begin{align*}
K_{M}(\bfx, \bfx^{\prime})= \langle \bfphi_{M}( \bfx, \bfomega), \bfphi_{M}(\bfx^{\prime}, \bfomega)  \rangle,
\end{align*}
where $\bfphi_{M}(\bfx, \bfomega)= \frac{1}{\sqrt{M}}\big ( \phi (\bfx, \bfomega_1 ), \ldots, \phi (\bfx, \bfomega_M ) \big )^T$ is the feature map and $\bfomega_1,\ldots,\bfomega_M$ are independently sampled  with respect to $\pi$. Henceforth,  we use $\bfphi_{M}(\bfx)$ to denote $\bfphi_{M}(\bfx, \bfomega)$ for notation simplicity. Consequently, the solution of \eqref{loss_function} with random features can be written as 
\begin{align}\label{solution_RF}
f_{M,D,\lambda}(\bfx)=\hat{\bfu}^T \bfphi_{M}(\bfx), 
\end{align}
and the optimization problem becomes
\begin{align}\label{kqr_rf}
\hat{\bfu}= \argmin_{\bfu \in \mathbb{R}^M}\frac{1}{|D|} \sum_{i=1}^{|D|}\rho_{\tau}\big(y_i-\bfu^T\bfphi_{M}(\bfx_i)\big)+\lambda \bfu^T\bfu .
\end{align}
Notably, leveraging random features allows us to reformulate the initial problem into linear quantile regression augmented by a ridge penalty, reducing the number of parameters to be $M \ll |D|$. In our simulation study, we utilize the ADMM algorithm with the proximal operator \citep{boyd2011distributed, gu2018admm} to solve \eqref{kqr_rf}. Although random features can achieve significant success in both computation and storage by approximating the kernel, the detailed trade-off between the number of features required and the statistical prediction accuracy is still an open question, especially when the non-smooth check loss is considered and the true quantile function lies outside of the exact RKHS ${\cal H}_K$. This paper aims to answer these theoretical questions of KQR-RF in subsequent sections.

\section{Theoretical Guarantee}\label{sec:the}
In this section, we first present an existing bound for KQR-RF \citep[Theorem \ref{exist_thm}]{li2021towards}, where they focus on the Lipschitz continuous loss family including the check loss. Subsequently, we provide our capacity-dependent and shaper learning rates for KQR-RF (Theorem \ref{thm1}), which can not only recover those of \citet{li2021towards}, but also can be applied to the case with the agnostic settings where the true quantile functions may not lie in the considered function space. Furthermore, we consider the data-dependent sampling strategy, which achieves the same rates (Corollary \ref{cor1}) with fewer random features and pertains its applicability to the agnostic settings. At last, we extend our theoretical results to a wider array of Lipschitz continuous losses with a modified local strong convexity assumption (Assumption \ref{ass7}).

The objective of KQR-RF is to find an estimator that minimizes the following expected risk
$$
{\cal E}(f)=\int_{\cal X \times \mathbb{R}}\rho_{\tau}\big(y-f(\bfx)\big)d\rho(\bfx,y),
$$
and we evaluate the performance of KRR-RF by the excess risk ${\cal E}(f)-{\cal E}(f_{\tau}^*)$,
or the $L_{\rho_{\cal X}}^2$-norm of the difference $\|f-f_{\tau}^*\|_{\rho}^2$. The following are some standard definitions and assumptions needed to establish the theoretical results.

\begin{definition}[Integral operators]\label{def1}
For any $f \in L_{\rho_{\cal X}}^2$, we define the integral operators by the kernel $K$ and $K_M$ as 
\begin{align*}
L_Kf&=\int_{\cal X}K(\bfx,\cdot)f(\bfx)d\rho_{\cal X},\\
L_Mf&=\int_{\cal X}K_M(\bfx,\cdot)f(\bfx)d\rho_{\cal X}.
\end{align*}
\end{definition}

\begin{definition}[Effective dimension]\label{def2}
For $\lambda >0$, we define the effective dimension of kernel $K$ and $K_M$ as 
\begin{align*}
{\cal N}(\lambda) &=\text{Tr}((L_K+\lambda I)^{-1}L_K),\\
{\cal N}_M(\lambda) &=\text{Tr}((L_M+\lambda I)^{-1}L_M).
\end{align*}
\end{definition}

The effective dimension $\cal N(\lambda)$ serves as a common metric in kernel learning theory literature, measuring the complexity of the RKHS ${\cal H}_K$ \citep{caponnetto2007optimal,smale2007learning,rudi2015less,rudi2017generalization}. Similarly, we introduce ${\cal N}_M(\lambda)$ as the effective dimension induced by the approximation kernel $K_M$. As indicated in Lemma \ref{lemD.8} \citep{rudi2017generalization} in the appendix, ${\cal N}_M(\lambda)$ has been shown to be equivalent to  ${\cal N}(\lambda)$ under mild conditions on the number of random features.

\begin{assumption}[Bounded and continuous random features]\label{ass4}
Assume kernel $K$ has the integral representation defined in \eqref{in_re} with $\phi$ bounded and continuous in both variables, that is, there exists some constant $\kappa \geq 1$ such that $|\phi (\bfx, \bfomega)| \leq \kappa$ for any $\bfx \in \mathcal{X}$ and $\bfomega \in \Omega$. The associated RKHS ${\cal H}_K$ is separable.
\end{assumption}

Assumption \ref{ass4} is a common condition in the literature of random features \citep{rudi2017generalization,liu2020effective,li2021towards}, which can be satisfied when the random features are continuous and bounded and $\cal X$ is separable. Note that this assumption  implies that $\sup_{\bfx,\bfx^{\prime}\in \cal_{X}}K(\bfx,\bfx^{\prime})\leq \kappa^2$ and $\sup_{\bfx,\bfx^{\prime}\in \cal_{X}}K_M(\bfx,\bfx^{\prime})\leq \kappa^2$.

\begin{assumption}[Source condition]\label{ass1}
Suppose there exists $R>0$, $r >0$ and $h_{\tau} \in L^2_{\rho_{\cal X}}$ such that
\begin{align}\label{ass3_e}
f_{\tau}^*=L_K^r h_{\tau} ,   
\end{align}
where $\|h_{\tau}\|_{\rho}\leq R$ and $L_K^r$ is the $r$-th power of $L_K$.
\end{assumption}

The parameter $r$ controls the size of the functional class of $f_{\tau}^*$, denoted as ${\cal F}=L_K^r(L^2_{\rho_{\cal X}})$. According to \citet{steinwart2008support,lin2016optimal}, we have ${\cal H}_K = L_K^{1/2}(L^2_{\rho_{\cal X}})$, and $L_K^{r_1}(L^2_{\rho_{\cal X}}) \subseteq L_K^{r_2}(L^2_{\rho_{\cal X}})$ if $r_1 \geq r_2$. When $r \in [1/2,1]$, the functional class $\cal F$ is a subset of the assumed RKHS ${\cal H}_K$, so we have $f_{\tau}^* \in {\cal H}_K$. When  $r \in (0,1/2)$,  the functional class $\cal F$ is larger than the assumed RKHS ${\cal H}_K$, and there exists some cases where $f_{\tau}^* \notin {\cal H}_K$. Existing literature on KQR and kernel methods with Lipschitz continuous loss functions often assumes that $r=1/2$ \citep{bach2017equivalence,sun2018but,li2021towards} or $r\in [1/2,1]$ \citep{lian2022}, corresponding to the realizable setting $f_{\tau}^* \in {\cal H}_K$. However, our analysis further allows $r\in (0,1/2)$, relating to the agnostic setting $f_{\tau}^* \notin {\cal H}_K$. This is a non-trivial extension since we consider a non-smooth loss with random feature approximation.

\begin{assumption}[Capacity condition]\label{ass2}
For $\lambda>0$, there exists $Q>0$ and $\gamma \in [0,1]$ such that
\begin{align}\label{ass2_e}
{\cal N}({\lambda}) \leq Q^2\lambda^{-\gamma}.   
\end{align}    
\end{assumption}

Note that this assumption always holds
when $\gamma=1$ by taking $Q=\operatorname{Tr}(L_K) \leq \kappa^2$, and $\gamma=0$ corresponds to some more benign cases.  It is more general than the eigenvalue decay assumption \citep{li2021towards,li2022sharp,lian2022}, since it is satisfied when the eigenvalues $\{\mu^{\prime}\}_{i\geq 1}$ of $L_K$ have a polynomial decay, i.e., $i^{-1/\gamma}$. For KRR and KRR-RF, the minimax optimal capacity-dependent rate has been shown to be ${\cal O}(|D|^{\frac{2r}{2r+\gamma}})$ \citep{caponnetto2007optimal,rudi2017generalization}. In the case of KQR, \citet{lian2022} also derive the same capacity-dependent rate ${\cal O}(|D|^{\frac{2r}{2r+\gamma}})$. We want to emphasize that these works mainly focus on the realizable setting with $r\in [1/2,1]$, while our result first extends the capacity-dependent rate analysis of KQR-RF to the agnostic setting. 

\begin{assumption}[Adaptive self-calibration condition]\label{ass6}
Let $f_{y|\bfx}(\cdot)$ denote the conditional density function of $y$ given $\bfx$. Suppose that $\sup_{t \in \mathbb{R}}f_{y|\bfx}(t) \leq c_1$ for $c_1>0$. Furthermore,  there exist some universal constants $\varepsilon, \varepsilon^{\prime}, c_2 >0$ that are independent with $\bfx$ and $y$, such that for any $y\in \mathcal{B}(f_{\tau}^*(\bfx),\varepsilon)$ and $|\delta|\leq \varepsilon^{\prime}$, the following inequality holds almost surely,
\begin{align}\label{ass6_e}
|F_{y|\bfx}(y+\delta)-F_{y|\bfx}(y)|   \geq c_2|\delta|,
\end{align}
 where $\mathcal{B}(f_{\tau}^*(\bfx),\varepsilon)=\{y\mid |y-f_{\tau}^*(\bfx)|\leq \varepsilon\}$ denotes the ball centered at $f_{\tau}^*(\bfx)$ with radius $\varepsilon$, and $F_{y|\bfx}(\cdot)$ is the cumulative distribution function of $y$ given $\bfx$.
\end{assumption}

Assumption \ref{ass6} serves as an adaptive self-calibration condition for the conditional distribution of $y$ given $\bfx$. It is a mild condition intended to hold for most realistic sequences of distributions. For example, if $y$ has a density that is bounded away from zero on some compact interval around $f_{\tau}^*(\boldsymbol{x})$, then Assumption 3.6 holds. More importantly, we do not impose any moment condition on the distribution of $y$.  It is also worth noting that Assumption 3.6 is weaker than Condition 2 in \citet{he1994convergence} where the density function of $y$ is lower bounded everywhere by some positive constant. It is also weaker than Condition D.1 in \citet{belloni2011} requiring the conditional density of $Y$ given $\boldsymbol{x}$ to be continuously differentiable and bounded away from zero uniformly for all $\tau \in (0,1)$ and all $\boldsymbol{x}$ in the support ${\cal X}$.
The special case when $\varepsilon=0$ aligning with the self-calibration condition also appeared in \citet{shen2021deep,madrid2022risk}. 

\begin{remark}
This adaptive self-calibration condition plays a pivotal role in our novel error decomposition as shown in Lemma \ref{lemA.1.1} of the appendix, which leads to an adaptive local strong convexity condition of the expected check loss in a small ball around $f_{\tau}^*$. It is worth noting that the self-calibration condition is weaker than Assumption (A2') of \citet{lian2022} and Assumption (B2) of \citet{li2021towards} where the conditional density of $y$ given $\bfx$ is assumed to be bounded away from zero across all quantile levels and $\bfx \in {\cal X}$. 
Under this assumption, we derive a tight bound for a novel least square approximation (LS-approximation) error between the KQR-RF estimator $f_{M,D,\lambda}$ and its KRR-RF approximation estimator $f_{M,D,\lambda}^{\diamond}$, detailed in Lemma \ref{lemA.2.4} of the appendix.    
\end{remark}

\subsection{Existing Learning Rates for KQR-RF}

To facilitate a clear comparison between our findings and existing results, we first introduce the best learning rates so far for KQR-RF \citep{li2021towards}.

\begin{theorem}[Existing learning rates for KQR-RF (random features with Lipschitz loss),  Theorem 19 of \citet{li2021towards}]\label{exist_thm}
Assume there exists a function $f_{\cal H}$ such that $f_{\cal H}=\argmin_{f \in {\cal H}_K}{\cal E}(f)$. Under some technical assumptions\footnote{Assumption 3.3, Assumption 3.4 with $r=1/2$, eigenvalue decaying assumption (stronger than Assumption 3.5), and the local strongly convex assumption which can be derived from Assumption 3.6.}, and $\lambda={\cal O}(|D|^{-1})$, when the number of random features satisfies
$$
M \gtrsim |D|^{\frac{\gamma}{2}}\log|D|,
$$
and $|D|$ is sufficiently large,  there holds
$$
{\cal E}(f_{M,D,\lambda})-{\cal E}(f_{\cal H}) \asymp \|f_{M,D,\lambda}-f_{\cal H}\|_{\rho}^2 ={\cal O}(|D|^{-\frac{1}{2}}),
$$
with probability near to 1.
\end{theorem}

Theorem \ref{exist_thm} establishes an upper bound for KQR-RF in the worst case, requiring only the existence of $f_{\cal H}$. In this scenario, if the number of random features scales as $|D|^{\frac{\gamma}{2}}\log|D|$, KQR-RF can achieve the capacity-independent optimal generalization properties. This represents a significant improvement over previous work, which required a larger number of random features to guarantee similar learning rates.  Note that  \citet{rahimi2008weighted} proved ${\cal O}(|D|)$ random features to guarantee the learning rates at ${\cal O}(|D|^{-\frac{1}{2}})$. However, these results are capacity-independent and can not apply to the agnostic setting when the size of RKHS is small. In our subsequent analysis, we will present a sharper and capacity-dependent learning rate, allowing $r\in (0,1]$, which covers the entire source condition space. This particularly marks the primary novelty and advancement in the theoretical understanding of KQR-RF.

\subsection{Sharper Learning Rates for KQR-RF}

\begin{theorem}[Worst case]\label{thm1}
Under Assumptions \ref{ass4}-\ref{ass6}, if $r\in (0,1]$, $\gamma \in [0,1]$, $2r+\gamma\geq 1$, and set $\lambda=|D|^{-\frac{1}{2r+\gamma}}$, when the number of random features satisfies 
\begin{align*}
&M \gtrsim  |D|^{\frac{1}{2r+\gamma}},\quad \text{for} \quad r\in (0,1/2) ;\\
&M \gtrsim  |D|^{\frac{(2r-1)\gamma+1}{2r+\gamma}},   \quad \text{for} \quad r\in [1/2,1],   
\end{align*}
and $|D|$ is sufficiently large, there holds
\begin{align*}
{\cal E}(f_{M,D,\lambda})-{\cal E}(f_{\tau}^*) &\asymp \|f_{M,D,\lambda}-f_{\tau}^*\|_{\rho}^2 \\
&={\cal O}(|D|^{-\frac{2r}{2r+\gamma}}\log^2 |D|),
\end{align*}
with probability near to 1.
\end{theorem}

The capacity-dependent learning rates obtained in Theorem \ref{thm1} align with those of KRR \citep{caponnetto2007optimal} and KRR-RF \citep{rudi2017generalization}, which is minimax optimal and thus can not be improved any further. Specifically, in scenario of highest regularity ($r=1$) and a small RKHS  ($\gamma=0$), it approaches the standard parametric bound ${\cal O}(1/|D|)$. For $r=1/2$ and $\gamma=1$, corresponding to the worst case, our learning rates and the requirements on the number of the random features match those in Theorem \ref{exist_thm}. More interestingly, our results extend the optimal learning rates to the agnostic case where the true quantile function is located outside of the RKHS ${\cal H}_K$. Specifically, we relax the regularity condition from $r \in [1/2,1]$ to $r \in (0,1], 2r+\gamma \geq 1$, covering a wider range of scenarios.

\begin{remark}\label{rem1}
Recent studies have explored the generalization performance of kernel-based methods in the agnostic setting, including kernel ridge regression \citep{zhang23mis}, kernel ridge regression with Nystr{\"o}m subsampling \citep{lu2019analysis,li2023optimal}, and kernel ridge regression with random features \citep{lioptimalRF2023, Liijcai2023}. However, these studies primarily focus on the least square loss, contrasting with our work that delves into more complex non-smooth check loss and a broader Lipschitz loss family. Our theory requires a distinct set of proof techniques compared to the work grounded in the least square loss paradigm which has an implicit solution, necessitating the use of the empirical process. Specifically, we introduce a novel error decomposition including an LS-approximation error term, which bridges the excess risk for the check loss with the $L_{\rho_{\cal X}}^2$ error of an intermediate estimator $f_{M,D,\lambda}^{\diamond}$ (see details in Lemmas \ref{lemA.1.1} and \ref{lemA.2.4} of the appendix). To derive the faster learning rates for both realizable and agnostic settings, we use different technical skills to take the regularity condition into the LS-approximation error term, such as the non-trivial Young's inequality and Cauthy-Schwarz inequality tailored for operators.
\end{remark}

\begin{figure*}
    \centering
    \subfigure[Agnostic case]{
        \includegraphics[scale=0.23]{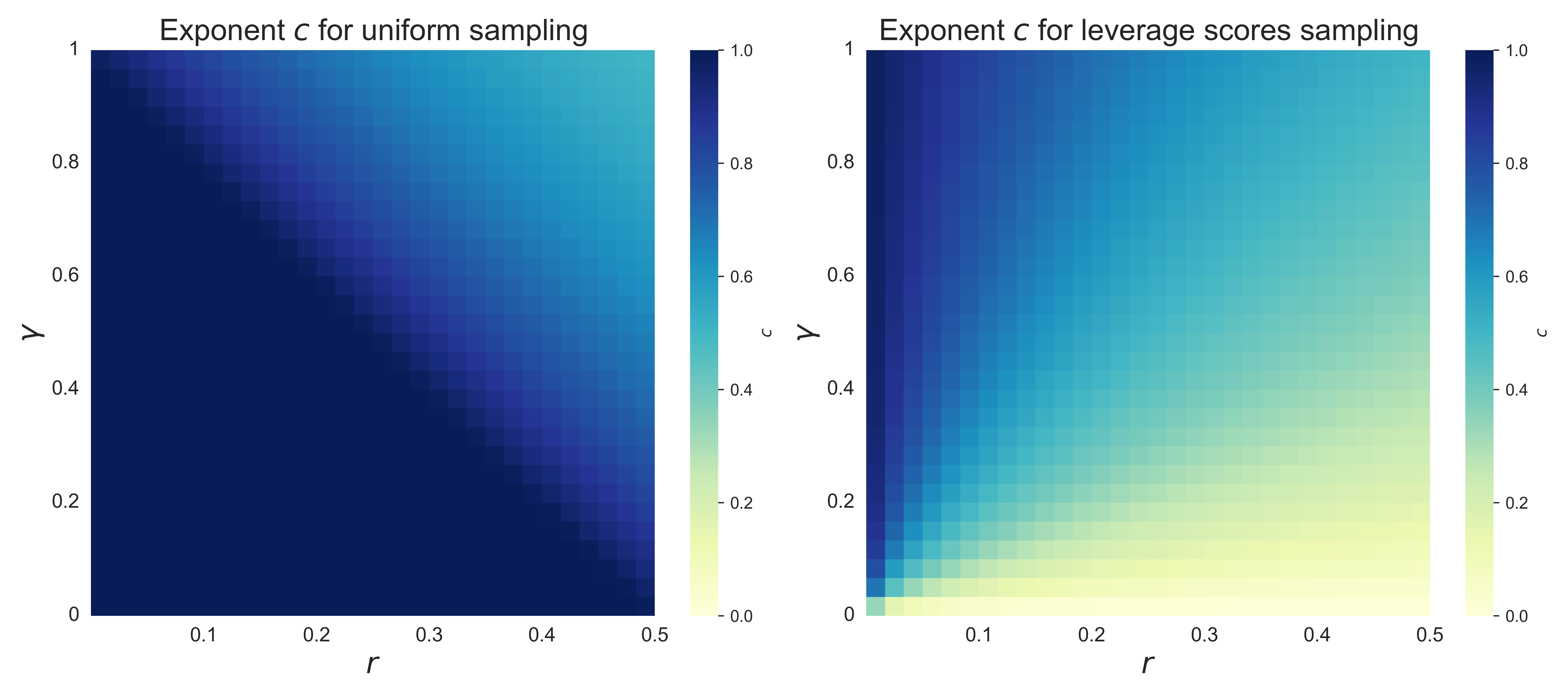}
        \label{agnostic_case}
    }
    \hspace{8mm}
    \subfigure[Realizable case]{
 \includegraphics[scale=0.23]{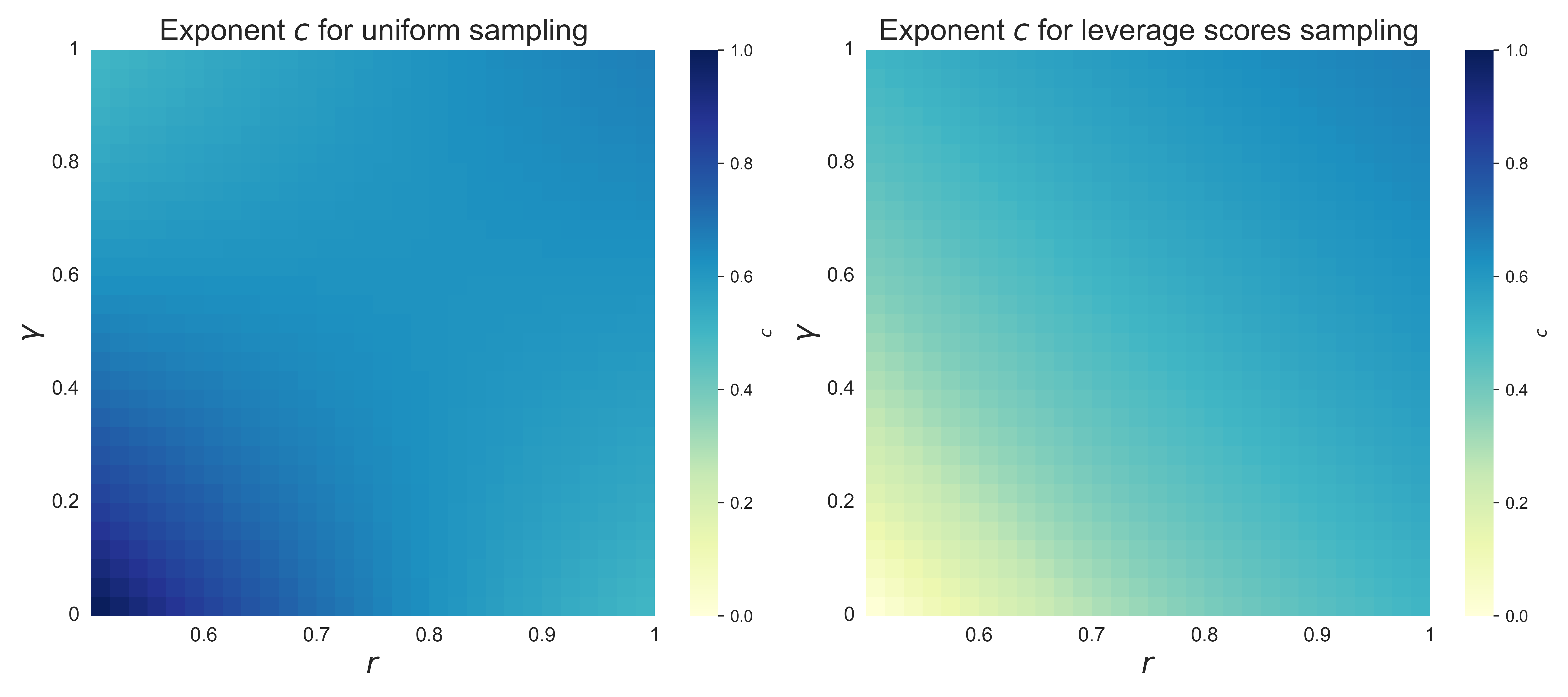}
        \label{realizable_case}
    }
  \caption{Comparison between the number of random features $M={\cal O}(|D|^c)$ required for uniform sampling ($\alpha=1$, left) and leverage scores sampling ($\alpha=\gamma$, right), Figure \ref{agnostic_case} is the agnostic case  and Figure \ref{realizable_case} is the realizable case, respectively.}
  \label{fig1}
\end{figure*} 

\begin{remark}\label{rem2}
Theorem \ref{thm1} broadens the regularity condition for optimal learning rates from $r\in [1/2,1]$ to $(0,1]$, $2r+\gamma \geq 1$. However, it uses the naive uniform sampling strategy for the random features (generate $\phi(\bfx,\bfomega)$ with $\pi(\bfomega)$), which is independent of the training samples. This may lead to an unnecessary burden in computation. Inspired by the data-dependent sampling strategy \cite{bach2017equivalence,avron2017faster,rudi2017generalization}, we aim to demonstrate in the upcoming section how these strategies enable attaining optimal learning rates across the agnostic settings $r\in (0,1]$ with a reduced number of random features in the next section. 
\end{remark}

\subsection{Refined Analysis: Beyond Uniform Sampling}

To obtain sharper learning rates for the setting $r \in (0,1]$ with fewer random features, we first introduce a compatibility condition that is commonly used in the literature \citep{rudi2015less,rudi2017generalization,liu2020effective}.

\begin{assumption}[Compatibility condition]\label{ass5}
Define the maximum dimension of random features as
\begin{align}\label{ass5_e}
\mathcal{N}_{\infty}(\lambda)=\sup _{\bfomega \in \Omega}\left\|(L_K+\lambda I)^{-1 / 2} \phi(\cdot, \bfomega)\right\|_{\rho_{\cal X}}^2,    
\end{align}
where $\lambda>0$. There exist constants $\alpha \in[0,1]$ and $F>0$, such that $\mathcal{N}_{\infty}(\lambda) \leq F \lambda^{-\alpha}$. 
\end{assumption}

The maximum dimension of random features in \eqref{ass5_e} correlates with the data-generating distribution through the integral operator $L_K$, which is always satisfied for $\alpha=1$ and $F=\kappa^2$. Recall the definition of ${\cal N}(\lambda)$ in Definition \ref{def2}. ${\cal N}(\lambda)$ and ${\cal N}_{\infty}(\lambda)$ measure the average and  supreme capacities of ${\cal H}_K$, respectively, so we have ${\cal N}(\lambda)=E_{\bfomega}\left\|(L_K+\lambda I)^{-1 / 2} \phi(\cdot, \bfomega)\right\|_{\rho_{\cal X}}^2\leq \sup _{\bfomega \in \Omega}\left\|(L_K+\lambda I)^{-1 / 2} \phi(\cdot, \bfomega)\right\|_{\rho_{\cal X}}^2= {\cal N}_{\infty}(\lambda)$, where $E_{\bfomega}$ denotes the expectation taking over $\bfomega$.

\begin{theorem}\label{thm2}
Under Assumptions \ref{ass4}-\ref{ass6} and \ref{ass5}, if $r\in (0,1]$, $\gamma \in [0,1]$, $2r+\gamma\geq 1$, and set $\lambda=|D|^{-\frac{1}{2r+\gamma}}$, when the number of random features satisfies 
\begin{align*}
&M \gtrsim  |D|^{\frac{\alpha}{2r+\gamma}},\quad \text{for} \quad r\in (0,1/2) ;\\
&M \gtrsim  |D|^{\frac{(2r-1)(1+\gamma-\alpha)+\alpha}{2r+\gamma}},   \quad \text{for} \quad r\in [1/2,1],   
\end{align*}
and $|D|$ is sufficiently large, there holds
\begin{align*}
{\cal E}(f_{M,D,\lambda})-{\cal E}(f_{\tau}^*) &\asymp \|f_{M,D,\lambda}-f_{\tau}^*\|_{\rho}^2 \\
&={\cal O}(|D|^{-\frac{2r}{2r+\gamma}}\log^2 |D|),
\end{align*}
with probability near to 1.    
\end{theorem}

The above capacity-dependent learning rate is the same as that of Theorem \ref{thm1}, while the required number of random features reduces from ${\cal O}(|D|^{\frac{1}{2r+\gamma}})$ to ${\cal O}(|D|^{\frac{\alpha}{2r+\gamma}})$ when $r \in (0,1/2)$ and ${\cal O}(|D|^{\frac{(2r-1)\gamma+1}{2r+\gamma}})$ to ${\cal O}(|D|^{\frac{(2r-1)(1+\gamma-\alpha)+\alpha}{2r+\gamma}})$ when $r \in [1/2,1]$, owing to the additional imposition of the compatibility condition $\mathcal{N}_{\infty}(\lambda) \leq F \lambda^{-\alpha}$. By adopting a favorable sampling strategy, as demonstrated in Example \ref{example1}, we can further reduce the required number of random features and achieve the optimal learning rates across the range of $r \in (0,1]$ and $2r+\gamma \geq 1$.

\begin{example}[Leverage scores sampling]\label{example1}
Given the integral representation of kernel $K$ as stated in \eqref{in_re}, we adopt the leverage scores sampling strategy \cite{bach2017equivalence,avron2017faster} by employing an importance ratio denoted as $q(\bfomega)=l_{\lambda}(\bfomega)/\int_{\bfomega}l_{\lambda}(\bfomega)d\pi(\bfomega)$, where $l_{\lambda}(\bfomega)=\|(L_K+\lambda I)^{-1 / 2} \phi(\cdot, \bfomega)\|_{\rho_{\cal X}}^2$. Consequently, the random features are computed as $\phi_{l}(\bfx,\bfomega)=[q(\bfomega)]^{-1/2}\phi(\bfx,\bfomega)$ and exhibit a distribution  $\pi_{l}(\bfomega)=q(\bfomega)\pi(\bfomega)$. As pointed out in \citet{rudi2017generalization}, the random features provide the integral representation of $K$ and satisfy Assumption \ref{ass5} with $\alpha=\gamma$ indicating that $\cal{N}(\lambda)=\cal{N}_{\infty}(\lambda)$. 
\end{example}

\begin{remark}
We call $\alpha=1$ as the worst case (Theorem \ref{thm1})
when considering the random features with uniform sampling in \eqref{in_re} which is independent of the training samples, and $\alpha=\gamma$ as the benign case (Corollary \ref{cor1}) when adopting the data-dependent sampling strategy in Example \ref{example1}.     
\end{remark}

\begin{corollary}[Benign case]\label{cor1}
 Under Assumptions \ref{ass4}-\ref{ass6}, if random features are sampled according to the strategy in Example \ref{example1},  $r\in (0,1]$, $\gamma \in [0,1]$, $2r+\gamma\geq 1$, and set $\lambda=|D|^{-\frac{1}{2r+\gamma}}$, when the number of random features satisfies 
\begin{align*}
&M \gtrsim  |D|^{\frac{\gamma}{2r+\gamma}},\quad \text{for} \quad r\in (0,1/2) ;\\
&M \gtrsim  |D|^{\frac{2r+\gamma-1}{2r+\gamma}},   \quad \text{for} \quad r\in [1/2,1],   
\end{align*}
and $|D|$ is sufficiently large, then there holds
\begin{align*}
{\cal E}(f_{M,D,\lambda})-{\cal E}(f_{\tau}^*) &\asymp \|f_{M,D,\lambda}-f_{\tau}^*\|_{\rho}^2 \\
&={\cal O}(|D|^{-\frac{2r}{2r+\gamma}}\log^2 |D|),
\end{align*}
with probability near to 1.    
\end{corollary}

Theorem \ref{thm1} is the worst case of Theorem \ref{thm2} with $\alpha=1$, while Corollary \ref{cor1} is the benign case of Theorem \ref{thm2} with $\alpha=\gamma$. This distinction arises from the choice of the uniform sampling strategy $\pi(\bfomega)$, which typically yields an approximate estimate where $\alpha$ tends to $1$. Conversely, employing data-dependent random features assures a favorable scenario where $\alpha=\gamma$. To better illustrate the computational improvement for different cases, we depict a comparison in Figure \ref{agnostic_case} and \ref{realizable_case} between the number of random features required to ensure the optimal learning rates using uniform sampling (left panel) and data-dependent sampling (right panel) for the agnostic case when $r\in (0,1/2)$ and the realizable case when $r \in [1/2,1]$, respectively.

\subsection{Extension to Lipschitz Loss}
Note that the check loss belongs to the family composed of Lipschitz continuous losses.  We aim to extend our theoretical results to the general Lipschitz continuous loss family, including other kernel-based methods, such as kernel support vector machines \citep[KSVM]{sun2018but} and kernel logistic regression \citep[KLR]{keerthi2005fast}. Similar to the quantile regression estimation in \eqref{kqr_rf}, we formulate the following general learning problem
\begin{align*}
\widetilde{\bfu}= \argmin_{\bfu \in \mathbb{R}^M}\frac{1}{|D|} \sum_{i=1}^{|D|}L\big(y_i, \bfu^T\bfphi_{M}(\bfx_i)\big)+\lambda \bfu^T\bfu, 
\end{align*}
where $L(y,\cdot)$ is a Lipschitz continuous loss such that for some $V \geq 0$, there exists a constant $c_{L}>0$ such that $|L(y, x)-L(y, x^{\prime})|\leq c_{L}|x-x^{\prime}|$ holds for all pairs $x, x^{\prime} \in [-V, V]$ and $y \in \mathbb{R}$. We can refer to \citet{feng2024towards} for more specific examples satisfying this property. 

Our objective is to replace the check loss $\rho_{\tau}$ with some general Lipschitz continuous loss and construct a unified theoretical framework. In our proof of the main theorems for KQR-RF, a pivotal step involves controlling the LS-approximation error therm in Lemmas \ref{lemA.2.4} and \ref{lemA.2.5}. To facilitate this, we merely need to substitute Assumption \ref{ass6} with the following substantial assumption.

\begin{assumption}[Local strong convexity]\label{ass7}
There exist some constants $u,u^{\prime},c_3,c_4>0$ such that for any $f$ and $f^{\prime}$ satisfying $\|f-f^{\prime}\|_{\rho}\leq u$ and $\|f^{\prime}-f^*\|_{\rho}\leq u^{\prime}$, there holds
\begin{align}\label{local_con}
 {\cal E}_{L}(f)-{\cal E}_{L}(f^{\prime})\geq c_3  \|f-f^{\prime}\|^2_{\rho},
\end{align}
or 
\begin{align}\label{adap_local_con}
 {\cal E}_{L}(f)-{\cal E}_{L}(f^{\prime})+\|f^{\prime}-f^*\|^2_{\rho}\geq c_4  \|f-f^{\prime}\|^2_{\rho},
\end{align}
where ${\cal E}_{L}(f)=E(L(y,f(\bfx)))$ and $f^*=\argmin_f {\cal E}_L(f)$. Here, we refer to \eqref{local_con} as the local strong convexity of $L(y,f)$, and \eqref{adap_local_con} as the adaptive local strong convexity of $L(y,f)$. 
\end{assumption}

It is worth pointing out that we can verify conditions \eqref{local_con} and \eqref{adap_local_con} for the check loss $\rho_{\tau}(\cdot)$ by using \eqref{ass6_e} in Assumption \ref{ass6} with $\varepsilon=0$ and $\varepsilon \neq 0$, respectively. With Assumption \ref{ass6} replaced by Assumption \ref{ass7} and keeping all other conditions unchanged, we can similarly establish the same learning rates for the Lipschitz loss $L$. Specifically,
$$
\|f^L_{M,D,\lambda}-f^*\|^2_{\rho}={\cal O}(|D|^{-\frac{2r}{2r+\gamma}}\log^2 |D|),
$$
with probability near to 1, where $f^L_{M,D,\lambda}=\widetilde{\bfu}^T\phi_M$. The detailed results for Lipschitz continuous loss and their proofs are deferred to Appendix \ref{proof_lip} due to the space limit.

\section{Comparisons to Related Work}\label{sec:com}

\begin{table*}
\centering
\caption{Summary of conditions for derived learning rates in different methods.}
\begin{threeparttable}
 \scalebox{0.85}{\begin{tabular}{c|c|c|c|c}
\hline
Methods              & Regularity condition & \makecell{Capacity \\condition} & Random centers M & Learning rate \\ \hline
KRR \citep{caponnetto2007optimal}                   &    $r\in [1/2,1]$                  &  $\gamma\in[0,1]$                  &     $\times$             &  $|D|^{-\frac{2r}{2r+\gamma}}$             \\ \hline
KRR \citep{zhang23mis}                   &            $r\in (0,1]$           &              $\gamma\in[0,1]$      &             $\times$      &       $|D|^{-\frac{2r}{2r+\gamma}}$        \\ \hline
KRR-RF-Uniform \citep{rudi2017generalization}        &               $r\in [1/2,1]$       &               $\gamma\in[0,1]$      &              $|D|^{-\frac{(2r-1)\gamma+1}{2r+\gamma}}$     &          $|D|^{-\frac{2r}{2r+\gamma}}$      \\ \hline
KRR-RF-Leverage  \citep{rudi2017generalization}       &                $r\in [1/2,1]$       &            $\gamma\in[0,1]$        &  $|D|^{-\frac{2r+\gamma-1}{2r+\gamma}}$                 &    $|D|^{-\frac{2r}{2r+\gamma}}$            \\ \hline
KRR-RF-Uniform  \citep{lioptimalRF2023}       &       $r\in (0,1], 2r+\gamma \geq 1$                &      $\gamma\in[0,1]$               &  $|D|^{-\frac{1}{2r+\gamma}}$                 &    $|D|^{-\frac{2r}{2r+\gamma}}$\\ \hline
KRR-RF-Leverage    \citep{lioptimalRF2023}       &       $r\in (0,1]$                &      $\gamma\in[0,1]$               &  $|D|^{-\frac{\gamma}{2r+\gamma}}$                 &    $|D|^{-\frac{2r}{2r+\gamma}}$\\ \hline
KQR    \citep{lian2022}        &    $r\in [1/2,1]$                  &  $\gamma\in[0,1]$                  &     $\times$             &  $|D|^{-\frac{2r}{2r+\gamma}}$             \\ \hline
Lip-RF-Uniform \citep{rahimi2008weighted}  &  $r=1/2$                    &  $\gamma\in[0,1]$                &     $|D|$             &      $|D|^{-1/2}$         \\ \hline
Lip-RF-Leverage \cite{bach2017equivalence} &  $r=1/2$                    &  $\gamma\in[0,1]$                &     $|D|^{\frac{\gamma}{2}}$             &      $|D|^{-1/2}$         \\ \hline
Lip-RF-Uniform \citep{li2021towards}  &  $r=1/2$                    &  $\gamma\in[0,1]$                &     $|D|$             &      $|D|^{-1/2}$         \\ \hline
Lip-RF-Leverage \citep{li2021towards}  &  $r=1/2$                    &  $\gamma\in[0,1]$                &     $|D|^{\frac{\gamma}{2}}$             &      $|D|^{-1/2}$  \\ \hline
KSVM-RF   \citep{sun2018but}             &             $r=1/2$          &  $\gamma\in[0,1]$                   &               $|D|^{\frac{2\gamma}{2\gamma+1}}$    &     $|D|^{-\frac{1}{2\gamma+1}}$         \\ \hline
\textbf{KQR-RF (Theorem \ref{thm2})}                &      $r\in(0,1], 2r+\gamma\geq 1$                &   $\gamma\in[0,1]$                   &     \makecell{$|D|^{\frac{\alpha}{2r+\gamma}}, r \in (0,1/2)$\\$|D|^{\frac{(2r-1)(1+\gamma-\alpha)+\alpha}{2r+\gamma}}, r\in[1/2,1]
$}             &      $|D|^{-\frac{2r}{2r+\gamma}}$         \\ \hline
\textbf{KQR-RF-Uniform  (Theorem \ref{thm1})}       & $r\in(0,1], 2r+\gamma\geq 1$                &   $\gamma\in[0,1]$                   &     \makecell{$|D|^{\frac{1}{2r+\gamma}}, r \in (0,1/2)$\\$|D|^{\frac{(2r-1)\gamma+1}{2r+\gamma}}, r\in[1/2,1]
$}             &      $|D|^{-\frac{2r}{2r+\gamma}}$   \\ \hline
\textbf{KQR-RF-Leverage (Corollary \ref{cor1})}&      $r\in(0,1], 2r+\gamma\geq 1$                &   $\gamma\in[0,1]$                   &     \makecell{$|D|^{\frac{\gamma}{2r+\gamma}}, r \in (0,1/2)$\\$|D|^{\frac{2r+\gamma-1}{2r+\gamma}}, r\in[1/2,1]
$}             &      $|D|^{-\frac{2r}{2r+\gamma}}$         \\ \hline
\end{tabular}} 
\end{threeparttable}
\label{tab.1}
\end{table*}

In this section, we compare the conditions and learning rates of our method with related existing approaches including KRR, KRR-RF, KQR, and Lipschitz loss with RF, which are summarized in Table \ref{tab.1}, where ``Uniform" and ``Leverage" stand for the uniformly and leverage scores sampling strategies, respectively, and ``Lip" is short for Lipschitz continuous loss. Note that under Assumption \ref{ass7}, the results in the last three lines of Table \ref{tab.1} can also be directly extended to the cases with the general Lipschitz continuous losses.

\noindent \textbf{Compared to KRR and its RF variants.} Previous studies have extensively pursued the optimal learning rates for KRR \citep{caponnetto2007optimal,smale2007learning} and KRR-RF \citep{rudi2017generalization,avron2017faster}. Recent extensions \citep{zhang23mis,lioptimalRF2023,Liijcai2023} have enlarged the regularity condition to the agnostic setting when the regression function lies outside of the RKHS. However, we focus on KQR-RF with the non-smooth check loss, which is more challenging since we have no explicit solutions. Notably, deriving a capacity-dependent learning rate for $r\in (0,1/2)$ demands distinct technical skills compared to those required for KRR and its RF variations. Moreover, our results can be easily extended to the Lipschitz losses with a modified assumption, signifying the added novelty of our analysis.

\noindent \textbf{Compared to kernel methods with Lipschitz loss and their RF variants.} Existing literature for random features with Lipschitz loss \citep{rahimi2008weighted,sun2018but,li2021towards,li2022sharp} only consider the ideal case when $r=1/2$, and their learning rates are either capacity-independent \citep{rahimi2008weighted,li2021towards} or suboptimal \citep{sun2018but}. \citet{lian2022} studied the capacity-dependent learning rate for KQR when $r\in [1/2,1]$. However, their work can not be directly applied to the random feature setting. In contrast, our study offers a comprehensive analysis of the capacity-dependent learning rates for KQR-RF (Lip-RF), exhibiting broader applicability across scenarios where the true regression function resides in the  agnostic setting.

We also provide a brief proof sketch to emphasize the theoretical contributions of this paper. 

\noindent \textbf{A novel error decomposition and least square approximation $\|f_{M,D,\lambda}-f_{M,D,\lambda}^{\diamond}\|_{\rho}$.} Unlike existing RF work \cite{rudi2017generalization,li2021towards,lioptimalRF2023}, we first introduce a novel error decomposition in Lemma \ref{lemA.1.1}. Except for the standard empirical, RF approximation, and kernel approximation errors, we have an extra least square approximation (LS-approximation) error term $\|f_{M,D,\lambda}-f_{M,D,\lambda}^{\diamond}\|_{\rho}$. By the adaptive self-calibration assumption, we build an adaptive local strong convexity condition of the expected loss on a small neighborhood of $f^*_{\tau}$. This promises the convergence of the LS-approximation error term. We also use the non-trivial Cauthy-Schwarz and Young's inequalities to take into account the source index when $r\in (0,1/2)$ and $r\in [1/2,1]$, respectively. 

\noindent \textbf{Sharper analysis for $\|f_{M,D,\lambda}^{\diamond}-f^{*}_{\tau}\|_{\rho}$.} As indicated in Lemma \ref{lemA.1.1}, we divide the $\|f_{M,D,\lambda}^{\diamond}-f^{*}_{\tau}\|_{\rho}$ into three terms. To get tighter bounds of the empirical errors $\|f_{M,D,\lambda}^{\diamond}-f_{M,\lambda}\|_{\rho}$ and $\|f_{M,\lambda}-f_{\lambda}\|_{\rho}$, we utilize the compatibility condition to Bernstein's inequalities among operators $L_K, L_M$, and $C_M, C_{M,D}$. This refined procedure helps to relax the conditions on $M$ and  $|D|$, and further enlarges the regularity condition to $r \in (0,1]$. In fact, the convergence of term $\|f_{M,D,\lambda}^{\diamond}-f^{*}_{\tau}\|_{\rho}$ is also an important premise for the convergence of the LS-approximation error  $\|f_{M,D,\lambda}-f_{M,D,\lambda}^{\diamond}\|_{\rho}$.

\section{Numerical Experiments}\label{sec:num}
Inspired by the simulation setup in \citet{rudi2017generalization,li2021towards}, we also consider the spline kernel of order $q$, defined as $\Lambda_q\left(x, x^{\prime}\right)=\sum_{k=-\infty}^{\infty} e^{2 \pi i k x} e^{-2 \pi i k x^{\prime}}|k|^{-q}$, where $x, x^{\prime} \in [0,1]$, and $q\in \mathbb{R}$. According to the property of spline kernel, we have $\int_0^1 \Lambda_q(x, z) \Lambda_{q^{\prime}}\left(x^{\prime}, z\right) d z=\Lambda_{q+q^{\prime}}\left(x, x^{\prime}\right)$, for any $q, q^{\prime} \in \mathbb{R}$. Consequently, for $r \in (0,1]$ and $\gamma \in [0,1]$, let $K(x,x^{\prime})=\Lambda_{\frac{1}{\gamma}}(x, x^{\prime})$, and its corresponding random feature is $\phi(x,w)=\Lambda_{\frac{1}{2\gamma}}(x, w)$ with $w \sim U(0,1)$. We consider the model $y=\Lambda_{\frac{r}{\gamma}+\frac{1}{2}}\left(x, 0\right)+\varepsilon$,
where $\varepsilon \sim N(0,0.01)$ and $x \sim U(0,1)$. Then Assumptions 3.3-3.5 and 3.11 are satisfied and $\alpha=\gamma$ \citep{rudi2017generalization}. To graphically show the true and estimated quantile function, we consider three different settings: (1) worst case ($r=0,\gamma=1$); (2) general case ($r=1/2, \gamma=1$); (3) most benign case ($r=1, \gamma=0$). Without loss of generality, we fix $\tau=0.5$. We generate training data with size $N_{tr}=1000$, and testing data with size $N_{te}=10000$. The regularization parameter $\lambda$ is selected via a grid search based on a validation set with 1000 samples, where the grid is set as $\{10^{0.5s}: s = -20, -19,\ldots, 2\}$. The number of random features is selected according to Theorem \ref{thm2}. The estimated and true quantile curves on the testing data are shown in Figure \ref{fig_2}. From the results, we can conclude that KQR-RF can estimate the quantile functions very well both in realizable and agnostic settings.


\begin{figure}
    \centering 
    \includegraphics[scale=0.22]{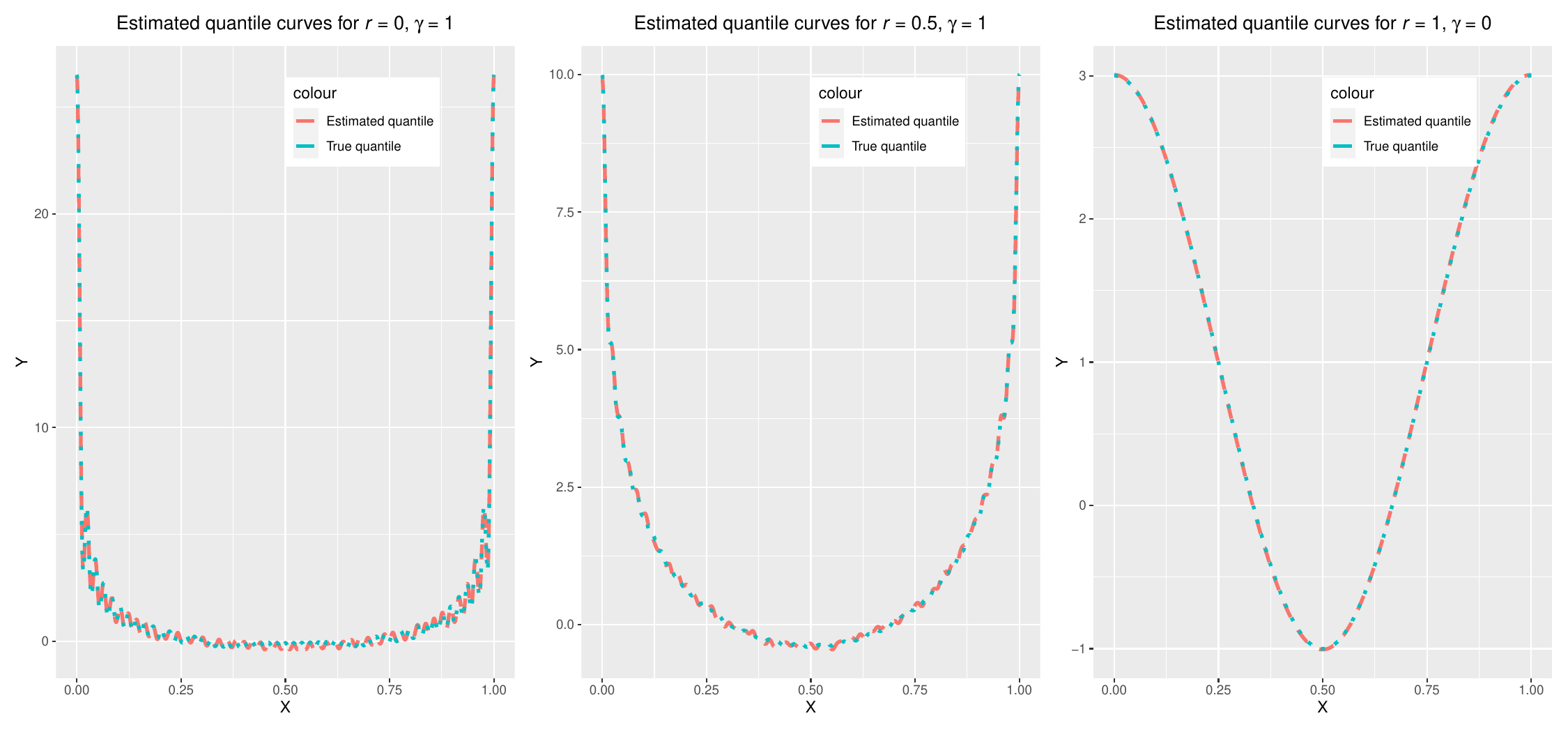} 
    \caption{Estimated and true quantile curves for $r=0,\gamma=1$ (left), $r=1/2, \gamma=1$ (middle), and $r=1, \gamma=0$ (right) when $\tau=0.5$.}
    \label{fig_2}
\end{figure}

To validate the derived learning rates, i.e., ${\cal E}(f_{M,D,\lambda})-{\cal E}(f_{\tau}^*)={\cal O}(|D|^{-\frac{2r}{2r+\gamma}})$, we estimate the excess risk on the testing data and compared it with the theoretical one. We consider two agnostic cases ($r=0.2, \gamma=0.1$ and $r=0.4, \gamma=0.2$) and two realizable cases ($r=0.5, \gamma=0.1$ and $r=0.8, \gamma=0.2$) for better illustration. The setting is the same as the above except that the training data size varies in $\{1000,2000,\ldots,10000\}$. We perform a log transform on the empirical excess risk and the number of training data and plot them in Figure \ref{fig_3}. From the results, we can see that the data points are uniformly distributed on both sides of a straight line, which verifies the derived learning rate. To further investigate the constants in the big-$\mathcal{O}$ bounds, we calculate the slope of each learning curve and compare it to $-\frac{2r}{2r+\gamma}$. The slope constants are $0.81,1.21,1.63,0.95$ in four scenarios. This also highlights our contribution in deriving the sharper and capacity-dependent learning rates. 
\begin{figure}[H]
    \centering 
    \includegraphics[scale=0.45]{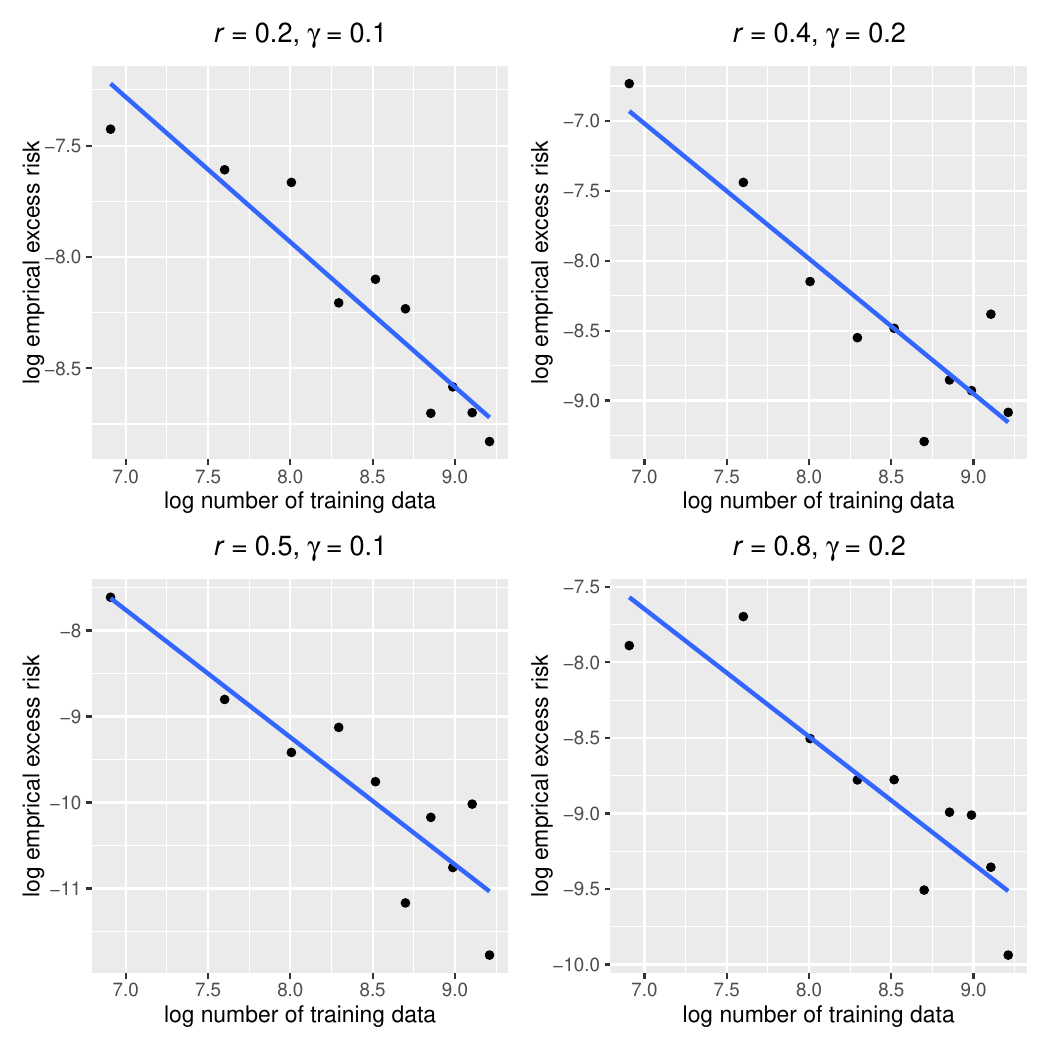} 
    \caption{Log empirical excess risk for $r=0.2,\gamma=0.1$ (left top), $r=0.4, \gamma=0.2$ (right top), $r=0.5, \gamma=0.1$ (left bottom) and $r=0.8, \gamma=0.2$ (right bottom) when $\tau=0.5$.}
    \label{fig_3}
\end{figure}

\section{Conclusion}\label{sec:con}
This paper investigates kernel quantile regression with random features and derives capacity-dependent optimal learning rates for both realizable and agnostic settings. By introducing a modified local strong convexity assumption, our theoretical analysis seamlessly extends to the entire Lipschitz continuous loss family, leading to the sharpest result so far to our best knowledge. Extensive experiments are conducted on both simulated and real case studies, providing empirical evidence that supports the theoretical findings of our paper. Furthermore, it is feasible to extend the theoretical results of random features to incorporate other accelerated approaches, such as stochastic gradient methods or distributed techniques, or consider the parallel problem in deep over-parameterized quantile regression.

\section*{Acknowledgement}
The authors thank the area chair and the anonymous referees for their constructive suggestions, which significantly improved this article. This research is supported in part by NSFC-12371270 and Shanghai Science and Technology Development Funds (23JC1402100).  This
research is also supported by Shanghai Research Center for Data Science and Decision Technology.

\section*{Impact Statement}
This paper presents work whose goal is to advance the field of Machine Learning. There are many potential societal consequences of our work, none which we feel must be specifically highlighted here.

\bibliography{example_paper}
\bibliographystyle{icml2024}

\newpage
\appendix
\onecolumn

\section{Random features in kernel methods and deep neural networks}
Random features mapping is a powerful tool for scaling up kernel methods \citep{rudi2017generalization}, neural tangent kernel \citep{zandieh2021scaling}, graph neural networks \citep{zambon2020graph} and attention in Transformers \citep{peng2021random}.  In fact, random features can be viewed as a class of two-layer neural networks with fixed weights in their first layer \citep{liu2021random}. For example, we consider a two-layer neural network, i.e., $f(\boldsymbol{x},\boldsymbol{\theta})=\sqrt{\frac{2}{M}}\sum_{j=1}^M \alpha_j\sigma(\boldsymbol{\omega}_j^T\boldsymbol{x})$ for some activation function $\sigma$, where $\boldsymbol{x} \in \mathbb{R}^d$ and $\boldsymbol{\omega} \sim N(0,\mathbf{I}_d)$. Its corresponding random features mapping is $k(\boldsymbol{x}, \boldsymbol{x}^{\prime})=E_{\boldsymbol{\omega}}\left[\sigma\left(\boldsymbol{\omega}^{T} \boldsymbol{x}\right) \sigma\left(\boldsymbol{\omega}^{T} \boldsymbol{x}^{\prime}\right)\right]$. If the commonly used ReLU activation $\sigma(x)=\max\{0,x\}$ is adopted, then the kernel is the first order arc-cosine kernel, i.e., $k\left(\boldsymbol{x}, \boldsymbol{x}^{\prime}\right) \equiv \kappa_1(u)=\frac{1}{\pi}\left(u(\pi-\arccos (u))+\sqrt{1-u^2}\right)$ with $u=\left\langle\boldsymbol{x}, \boldsymbol{x}^{\prime}\right\rangle /\left(\|\boldsymbol{x}\|\left\|\boldsymbol{x}^{\prime \prime}\right\|\right)$. This relationships helps to explain phenomena such as the fit the random labels \citep{zhang2021understanding} and double descent \citep{belkin2019reconciling} in the two-layer overparameterized neural networks \citep{arora2019fine}. For a deep neural network with more than two layers and fixed weights except for the output layer, we can also find a compositional kernel with its widths tending to infinity \citep{daniely2016toward}. In view of the connection of random features for kernel methods and neural networks, it is meaningful to study the generalization properties of random features in over-parameterized quantile neural networks by those in KQR, especially in the agnostic setting.

\section{Proofs of the Learning Rate KQR-RF}

To start with, we define a $M$-dimensional function space ${\cal H}_M$ related to $\bfphi_{M}(\bfx)$ as 
$$
{\cal H}_M=\left\{f \mid f(\bfx)=\bfu^T \phi_M(\bfx), \bfx\in \mathcal{X}, \bfu \in \mathbb{R}^M \right\}.
$$
It thus clear that ${\cal H}_M$ is a RKHS induced by kernel function $K_{M}(\bfx, \bfx^{\prime})= \langle \bfphi_{M}( \bfx, \bfomega), \bfphi_{M}(\bfx^{\prime}, \bfomega)  \rangle$. For $f=\bfu^T \phi_M(\bfx)\in {\cal H}_M, g=\bfz^T \phi_M(\bfx) \in {\cal H}_M$, we define their inner product in $ {\cal H}_M$ as $\langle f,g\rangle_{ {\cal H}_M}= \bfu^T\bfz$. And the corresponding norm of $f$ in $ {\cal H}_M$ is $\|f\|_{{\cal H}_M}=\sqrt{\bfu^T\bfu}=\|\bfu\|_2$. 

In the rest of this paper, we denote $\|\cdot\|$ as the operatorial norm, $\|\cdot\|_{HS}$ as the Hilbert-Schmidt norm and $\|\cdot\|_2$ as the Euclidean norm of a vector in $\mathbb{R}^n$. Let $\cal H$ be a Hilbert space, we denote with $\langle \cdot, \cdot \rangle_{\cal H}$ the associated inner product, with $\|\cdot\|_{\cal H}$ the norm and with $\text{Tr}(\cdot)$ the trace.

\subsection{Kernel and Random Feature Operators}

In this section, we provide some popular kernel and random feature operators used in the proofs. 

\begin{definition}\label{defA2}
For any $g \in L_{\rho_{\cal X}}^2$, $\bfbeta \in \mathbb{R}^M$, we define
\begin{itemize}
    \item $S_M:  \mathbb{R}^M \rightarrow L_{\rho_{\cal X}}^2, \ \ (S_M\bfbeta)(\cdot)=  \bfbeta^T\bfphi_M(\cdot),$
    \item $S_M^*: L_{\rho_{\cal X}}^2 \rightarrow \mathbb{R}^M, \ S_M^*g=\int_{\cal X} \bfphi_M(\bfx)g(\bfx)d\rho_{\cal X}(\bfx),$
    \item $S_{M,D}^*: L_{\rho_{\cal X}}^2 \rightarrow \mathbb{R}^M, \
S_{M,D}^*g=\frac{1}{|D|}\sum_{i=1}^{|D|}\bfphi_M({\bfx}_i)g({\bfx}_i),$
    \item $C_M: \mathbb{R}^M \rightarrow \mathbb{R}^M,  \ C_M=\int_{\cal X} \bfphi_M(\bfx)\bfphi_M(\bfx)^T d\rho_{\cal X}(\bfx),$
    \item $C_{M,D}: \mathbb{R}^M \rightarrow \mathbb{R}^M,  \
C_{M,D}=\frac{1}{|D|}\sum_{i=1}^{|D|}\bfphi_M({\bfx}_i)\bfphi_M({\bfx}_i)^T.$
\end{itemize}
\end{definition}

According to Definition \ref{def1}, \ref{defA2} and Assumption \ref{ass4}, we have $L_M, C_M, S_M, C_{M,D}$ are finite dimensional. Moreover, we have $L_M=S_MS_M^*, C_M=S_M^*S_M$ and $C_{M,D}=S_{M,D}^*S_M$. Finally, $L_M, C_M, C_{M,D}$ are self-adjoint and positive operator, with
spectrum is $[0,\kappa^2]$.

\subsection{Error Decomposition}
In this section, we first introduce some intermediate estimators based on check loss and least square loss and then explain the relationship between the estimators. Finally, we give a tight error decomposition for KQR-RF.

\begin{definition}\label{defA1}
We define the following intermediate estimators:
\begin{align*}
f_{M,D,\lambda}&=\argmin_{f\in \mathcal{H}_M} \left\{\frac{1}{|D|} \sum_{(\bfx,y)\in D}\rho_{\tau}\left(y-f({\bfx})\right)+\lambda\|f\|_{{\cal H}_M}^2\right\},\\
f_{M,D,\lambda}^{\diamond}&=\argmin_{f\in \mathcal{H}_M} \left\{\frac{1}{|D|} \sum_{(\bfx,y)\in D}\left(f({\bfx})-f_{\tau}^*(\bfx)\right)^2+\lambda\|f\|_{{\cal H}_M}^2\right\},\\
f_{M,\lambda}&=\argmin_{f \in {\cal H}_M}\left\{\int_{\cal X}\big(f({\bfx})-f_{\tau}^*(\bfx)\big)^2d \rho_{\cal X}(\bfx)+\lambda\|f\|_{{\cal H}_M}^2\right\}, \\
f_{\lambda}&=\argmin_{f \in {\cal H}_K}\left\{\int_{\cal X}\big(f({\bfx})-f_{\tau}^*(\bfx)\big)^2d \rho_{\cal X}(\bfx)+\lambda\|f\|_K^2\right\}. 
\end{align*}
We can also write $f_{M,D,\lambda}^{\diamond}=\bfphi_M(\cdot)^T{\bfomega}_{M,D,\lambda}^{\diamond}$ and ${\bfomega}_{M,D,\lambda}^{\diamond}=\argmin_{\bfomega \in \mathbb{R}^M}\frac{1}{|D|} \sum_{(\bfx,y)\in D}(\bfphi_M(\bfx)^T{\bfomega}-f_{\tau}^*(\bfx))^2+\lambda \|\bfomega\|_2^2$, $f_{M,\lambda}=\bfphi_M(\cdot)^T{\bfomega}_{M,\lambda}$ and ${\bfomega}_{M,\lambda}=\argmin_{\bfomega \in \mathbb{R}^M}\int_{\cal X}(\bfphi_M(\bfx)^T{\bfomega}-f_{\tau}^*(\bfx))^2d \rho_{\cal X}(\bfx)+\lambda \|\bfomega\|_2^2$.
\end{definition}

Note that $f_{M,D,\lambda}$ is the global estimator of KQR-RF, and it does not have an explicit form due to the non-smoothness of the check loss function, while the other three estimators are defined by the least square loss function. Recall the operators defined in Definition \ref{def1} and \ref{defA2}, we have 
\begin{align*}
&f_{M,D,\lambda}^{\diamond}=S_M(C_{M,D}+\lambda I)^{-1}S_{M,D}^*f_{\tau}^*,\\
&f_{M,\lambda}=(L_M+\lambda I)^{-1}L_Mf_{\tau}^*=S_M(C_{M}+\lambda I)^{-1}S_{M}^*f_{\tau}^*, \\
&f_{\lambda}=(L_K+\lambda I)^{-1}L_Kf_{\tau}^*.
\end{align*}

According to the definition of these estimators, we summarized a relationship chain from the KQR-RF estimator to the true quantile function $f_{\tau}^*$ in $L^2_{\rho_{\cal X}}$: 
$$
f_{M,D,\lambda} \stackrel{\rho_{\tau}\rightarrow ls}{\longrightarrow}f_{M,D,\lambda}^{\diamond}\stackrel{\rho_{\cal X}(\bfx)}{\longrightarrow}f_{M,\lambda}\stackrel{{\cal H}_M \rightarrow {\cal H}_K}{\longrightarrow}f_{\lambda}\stackrel{{\cal H}_K \rightarrow L_2}{\longrightarrow}f_{\tau}^*.
$$
Consequently, we can decompose the error in terms of $L_{\rho_{\cal X}}^2$-norm into four parts: $\|f_{M,D,\lambda}-f_{M,D,\lambda}^{\diamond}\|_{\rho}$ is the LS-approximation error from the non-smooth check loss to the least square loss; $\|f_{M,D,\lambda}^{\diamond}-f_{M,\lambda}\|_{\rho}$ is the empirical error from the sample to the expectation; $\|f_{M,\lambda}-f_{\lambda}\|_{\rho}$ is the approximation error introduced by the random features; and $\|f_{\lambda}-f_{\tau}^*\|_{\rho}$ is the approximation error between ${\cal H}_K$ and $L^2_{\rho_{\cal X}}$.

\begin{lemma}\label{lemA.1.1}
Let $f_{M,D,\lambda}$, $f_{M,D,\lambda}^{\diamond}$, $f_{M,\lambda}$ and $f_{\lambda}$ be defined in Definition \ref{defA1}, we have the following error decomposition for KQR-RF,
\begin{align}\label{eqA.1.1}
 \|f_{M,D,\lambda}-f_{\tau}^*\|_{\rho} \leq  \underbrace{\|f_{M,D,\lambda}-f_{M,D,\lambda}^{\diamond}\|_{\rho}}_{\text{LS-approximation error}} +\underbrace{  \|f_{M,D,\lambda}^{\diamond}-f_{M,\lambda}\|_{\rho}}_{\text{Empirical error}}+\underbrace{
 \|f_{M,\lambda}-f_{\lambda}\|_{\rho}}_{\text{RF error}}+\underbrace{ \|f_{\lambda}-f_{\tau}^*\|_{\rho}}_{\text{Approximation error}}.
\end{align}
\end{lemma}
\begin{proof}
 According to the triangle inequality, we can obtain the result directly.   
\end{proof}

\subsection{Error Bounds}
In this section, we provide the bounds for the four error terms in Lemma \ref{lemA.1.1}. By utilizing the operator representation of $f_{M,D,\lambda}^{\diamond}$, $f_{M,\lambda}$ and $f_{\lambda}$, we first bound the last three error terms. Benefiting from the maximum dimension of random features and Berntein's inequalities,
our refined convergence results allow the source condition index $r\in [0,1]$ which can also explain the agnostic case when $f_{\tau}^* \notin {\cal H}_K$. Thus we can show that $\|f_{M,D,\lambda}^{\diamond}-f_{\tau}^*\|_{\rho}$ is small enough, i.e., $\mathcal{O}_P(\lambda^r)$. Based on this result, we finally bound the LS-approximation error by using the empirical process and some properties of the check loss function.

\subsubsection{Approximation Errors}

\begin{lemma}\label{lemA.2.1}
Let $f_{\lambda}$ and $f_{\tau}^*$ be defined in Definition \ref{defA1} and \eqref{expected_risk}, respectively. Under Assumption \ref{ass1}, for any $\lambda \in (0,1)$ and $r \in (0,1]$, there holds
\begin{align*}
\|f_{\lambda}-f_{\tau}^*\|_{\rho}\leq R\lambda^r.    
\end{align*}
\end{lemma}
\begin{proof}
 Recall that $f_{\lambda}=(L_K+\lambda I)^{-1}L_Kf_{\tau}^*$ and Assumption \ref{ass1} that $f_{\tau}^*=L_K^rh_{\rho}$ with $\|h_{\rho}\|_{\rho}\leq R$, we have 
 \begin{align*}
  \|f_{\lambda}-f_{\tau}^*\|_{\rho}&=\|(L_K+\lambda I)^{-1}L_Kf_{\tau}^*-f_{\tau}^*\|_{\rho}=\|\lambda(L_K+\lambda I)^{-1}f_{\tau}^*\|_{\rho}  \\
  &=\lambda\|(L_K+\lambda I)^{-1}L_K^rh_{\rho}\|_{\rho}\\
  &=\lambda^r\|\lambda^{1-r}(L_K+\lambda I)^{r-1}(L_K+\lambda I)^{-r}L_K^rh_{\rho}\|_{\rho}\\
  &\leq \lambda^r\|\lambda(L_K+\lambda I)^{-1}\|^{1-r}\|(L_K+\lambda I)^{-1}L_K\|^r\|h_{\rho}\|_{\rho}\leq R\lambda^r,
 \end{align*}
where the first inequality is from Lemma \ref{lemD.1} and the fact that $\|\lambda(L_K+\lambda I)^{-1}\|^{1-r}\leq 1$ and $\|(L_K+\lambda I)^{-1}L_K\|^{r} \leq 1$ for $r \in (0,1)$. 
Thus we complete the proof.
\end{proof}

\subsubsection{Random Feature Error}
\begin{lemma}\label{lemA.2.2}
Let $f_{M,\lambda}$ and $f_{\lambda}$ be defined in Definition \ref{defA1}, for any $0<\lambda\leq \|L_K\| $ and $\delta\in (0,1)$, if the number of random features satisfies the following inequalities
\begin{align*}
&M \geq 16({\cal N}_\infty(\lambda)+1)\log (4/\delta),\quad \text{for} \quad r\in (0,1/2);\\
&M \geq 16({\cal N}_\infty(\lambda)+1)\log (4/\delta) \vee  128\kappa^2 \lambda^{1-2r}{\cal N}(\lambda)^{2r-1}{\cal N}_\infty(\lambda)^{2-2r}\log(4/\delta),   \quad \text{for} \quad r\in [1/2,1],
\end{align*}
then under Assumptions \ref{ass4}, \ref{ass1} and \ref{ass5},  there holds 
\begin{align*}
   \|f_{M,\lambda}-f_{\lambda}\|_{\rho} \leq  R\lambda^r,
\end{align*}
with probability at least $1-\delta$.
\end{lemma}
\begin{proof}
 From the definition of $f_{M,\lambda}$ and $f_{\lambda}$, we have 
\begin{align*}
f_{M,\lambda}-f_{\lambda}&=((L_M+\lambda I)L_M-(L_K+\lambda I)L_K)f_{\tau}^*\\
&=\lambda ((L_K+\lambda I)^{-1}-(L_M+\lambda I)^{-1})f_{\tau}^* \\
&=\lambda (L_M+\lambda I)^{-1}(L_M-L_K)(L_K+\lambda I)^{-1}f_{\tau}^*\\
&=\lambda^{1/2}(\lambda^{1/2}(L_M+\lambda I)^{-1/2})((L_M+\lambda I)^{-1/2}(L_K+\lambda I)^{1/2})[(L_K+\lambda I)^{-1/2}(L_M-L_K)(L_K+\lambda I)^{r-1}]\\
&((L_K+\lambda I)^{-r}L_K^r)h_{\rho},
\end{align*}
where the second and third equalities use $(A+\lambda I)^{-1}A=I-\lambda(A+\lambda I)^{-1}$ and $A^{-1}-B^{-1}=B^{-1}(B-A)A^{-1}$, and the last inequality we use  Assumption \ref{ass1} that $f_{\tau}^*=L_K^rh_{\rho}$ with $\|h_{\rho}\|_{\rho}\leq R$. Note that $\|\lambda^{1/2}(L_M+\lambda I)^{-1/2}\|\leq 1$, $\|(L_K+\lambda I)^{-r}L_K^r\|\leq 1$ and $\|(L_M+\lambda I)^{-1/2}(L_K+\lambda I)^{1/2}\|\leq \sqrt{2}$ in \eqref{lemC.2.2} from Lemma \ref{lemC.2} when $M \geq 16({\cal N}_\infty(\lambda)+1)\log (2/\delta)$, thus for any $\delta \in (0,1)$, there holds
\begin{align}\label{lemA.8.1}
\|f_{M,\lambda}-f_{\lambda}\|_{\rho} \leq R\sqrt{2\lambda}\|(L_K+\lambda I)^{-1/2}(L_M-L_K)(L_K+\lambda I)^{r-1}\|   
\end{align}
with probability at least $1-\delta$. We next to bound $\|f_{M,\lambda}-f_{\lambda}\|_{\rho}$ for two cases:

For the case when $r\in (0,1/2)$, according to \eqref{lemA.8.1}, we have 
\begin{align*}
\|f_{M,\lambda}-f_{\lambda}\|_{\rho} &\leq   R\sqrt{2\lambda}\|(L_K+\lambda I)^{-1/2}(L_M-L_K)(L_K+\lambda I)^{-1/2}\|\|(L_K+\lambda I)^{r-1/2}\| \\
&\leq \sqrt{2}R\lambda^r\|(L_K+\lambda I)^{-1/2}(L_M-L_K)(L_K+\lambda I)^{-1/2}\|\|\lambda^{1/2-r}(L_K+\lambda I)^{r-1/2}\|\\
&\leq \frac{\sqrt{2}}{2} R\lambda^r <R\lambda^r ,
\end{align*}
where the third inequality is from $\|\lambda^{1/2-r}(L_K+\lambda I)^{r-1/2}\|\leq 1$ for $r\in (0,1/2)$, and $\|(L_K+\lambda I)^{-1/2}(L_M-L_K)(L_K+\lambda I)^{-1/2}\| \leq 1/2$ in \eqref{lemC.2.1} from Lemma \ref{lemC.2} when $M \geq 16({\cal N}_\infty(\lambda)+1)\log (2/\delta)$.

For the case when $r\in [1/2,1]$, according to \eqref{lemA.8.1}, we apply Lemma \ref{lemD.2} by letting $s=2-2r \in [0,1]$, $X=(L_K+\lambda I)^{-1/2}(L_M-L_K)$ and $A=(L_K+\lambda I)^{-1/2}$,
$$
\|f_{M,\lambda}-f_{\lambda}\|_{\rho} \leq R\sqrt{2\lambda} \|(L_K+\lambda I)^{-1/2}(L_M-L_K)\|^{2r-1}\|(L_K+\lambda I)^{-1/2}(L_M-L_K)(L_K+\lambda I)^{-1/2}\|^{2-2r}.
$$

Note that from Lemmas \ref{lemC.1} and \ref{lemC.2p}, for any $\delta\in (0,1)$, with probability at least $1-\delta$, we have
\begin{equation}\label{lemA.8.2}
 \begin{aligned}
&\|f_{M,\lambda}-f_{\lambda}\|_{\rho} \leq R\sqrt{2\lambda} \|(L_K+\lambda I)^{-1/2}(L_M-L_K)\|^{2r-1}\|(L_K+\lambda I)^{-1/2}(L_M-L_K)(L_K+\lambda I)^{-1/2}\|^{2-2r}\\
\leq& R\sqrt{2\lambda} \left(\frac{4\kappa\sqrt{{\cal N}_\infty(\lambda)}\log (4/\delta)}{M}+\sqrt{\frac{4\kappa^2{\cal N}(\lambda)\log (4/\delta)}{M}} \right)^{2r-1}\left(\frac{2({\cal N}_\infty(\lambda)+1)\log (4/\delta)}{M}+\sqrt{\frac{2{\cal N}_\infty(\lambda)\log (4/\delta)}{M}}\right)^{2-2r}\\
\leq& R\sqrt{2\lambda}   \left[\left(\frac{4\kappa\sqrt{{\cal N}_\infty(\lambda)}\log (4/\delta)}{M}\right)^{2r-1}+\left(\sqrt{\frac{4\kappa^2{\cal N}(\lambda)\log (4/\delta)}{M}} \right)^{2r-1}\right]\left(\sqrt{\frac{4{\cal N}_\infty(\lambda)\log (4/\delta)}{M}}\right)^{2-2r}\\
 \leq& 4\sqrt{2}R \kappa^{2r-1}\left(\frac{\sqrt{\lambda{\cal N}_\infty(\lambda)}(\log(4/\delta))^r}{M^r}+\sqrt{\frac{\lambda{\cal N}(\lambda)^{2r-1}{\cal N}_\infty(\lambda)^{2-2r}\log(4/\delta)}{M}}\right),
\end{aligned}   
\end{equation}
where the second inequality follows from the inequality that $(a+b)^{2r-1}\leq a^{2r-1}+b^{2r-1}$ for $r\in [1/2,1]$.

Next, we need to add a condition on $M$ to bound $\eqref{lemA.8.2}$ with $R\lambda^r$. We consider
$$
M\geq 128\kappa^2 \lambda^{1-2r}{\cal N}(\lambda)^{2r-1}{\cal N}_\infty(\lambda)^{2-2r}\log(4/\delta),
$$
plug this condition into \eqref{lemA.8.2}, we get 
\begin{align*}
\|f_{M,\lambda}-f_{\lambda}\|_{\rho} &\leq  4\sqrt{2}R \kappa^{2r-1} \left(\sqrt{\frac{\lambda^{4r^2-2r+1}{\cal N}_\infty(\lambda)^{4r^2-4r+1}}{128^{2r}\kappa^{4r}{\cal N}(\lambda)^{4r^2-2r}}}+\sqrt{\frac{\lambda^{2r}}{128\kappa^2}}\right)  \\
&\leq 4\sqrt{2}R\left(\sqrt{\frac{\lambda^{2r}}{128^{2r}\kappa^{8r-8r^2}{\cal N}(\lambda)^{4r^2-2r}}}+\sqrt{\frac{\lambda^{2r}}{128\kappa^{4-4r}}}\right)\leq R\lambda^r,
\end{align*}
where the second inequality is obtained from ${\cal N}_\infty(\lambda)=\sup_{\bfomega}\|(L_K+\lambda)^{-1/2}\phi_{\bfomega}\|_{\rho} \leq \kappa^2/\lambda$ due to Assumptions \ref{ass4} and \ref{ass5}, the third inequality follows from: (1) $128^{2r}{\cal N}(\lambda)^{4r^2-2r}\geq 128$, this is from the fact that $2r \geq 1$ and $4r^2-2^r\geq 0$ with $r\in [1/2,1)$, and ${\cal N}(\lambda)=\text{Tr}((L_K+\lambda I)^{-1}L_K)\geq \frac{|L_K|}{|L_K|+\lambda}\geq 1/2$ with $0 \leq \lambda \leq |L_K|$; (2) $\kappa^{8r-8r^2}\geq 1$ and $\kappa^{4-4r}\geq 1$ with $\kappa\geq 1$ from Assumption \ref{ass4} and $r \in [1/2,1]$.

Combining the results of two cases, we complete the proof.

\end{proof}

\subsubsection{Empirical Error}
The empirical error is related to the similarity between $C_M$ and $C_{M,D}$, so we first define two important quantities measuring this similarity
\begin{align*}
&{\cal Q}_{M,D,\lambda}=\|(C_M+\lambda I)^{1/2}(C_{M,D}+\lambda I)^{-1/2}\|,\\
&{\cal R}_{M,D,\lambda}=\|(C_M+\lambda I)^{-1/2}(C_M-C_{M,D})(C_M+\lambda I)^{-1/2}\|.
\end{align*}

\begin{lemma}\label{lemA.2.3}
Let $f_{M,D,\lambda}^{\diamond}$ and $f_{M,\lambda}$ be defined in Definition \ref{defA1}, if the number of random features and the total sample size satisfy  inequalities 
\begin{align*}
&M \geq 16({\cal N}_\infty(\lambda)+1)\log (4/\delta),\quad \text{for} \quad r\in (0,1/2);\\
&M \geq 16({\cal N}_\infty(\lambda)+1)\log (4/\delta) \vee  128\kappa^2 \lambda^{1-2r}{\cal N}(\lambda)^{2r-1}{\cal N}_\infty(\lambda)^{2-2r}\log(4/\delta),   \quad \text{for} \quad r\in [1/2,1],
\end{align*}
and $|D| \geq 16(\kappa^2\lambda^{-1}+1)\log (6/\delta)$, respectively, then under Assumption  \ref{ass4}, \ref{ass1} and \ref{ass5}, for any $\delta \in (0,1)$, there holds
\begin{align*}
\|f_{M,D,\lambda}^{\diamond}-f_{M,\lambda}\|_{{\cal H}_M}\leq \sqrt{2}\widetilde{C}_1\lambda^{r-1/2} , 
\end{align*}  
and 
\begin{align*}
\|f_{M,D,\lambda}^{\diamond}-f_{M,\lambda}\|_{\rho}\leq \sqrt{2} \widetilde{C}_1 \lambda^r,   
\end{align*}  
with probability at least $1-\delta$, where $\widetilde{C}_1$ is a constant defined in the proof.
\end{lemma}
\begin{proof}
From the definition of $f_{M,D,\lambda}^{\diamond}$ and $f_{M,\lambda}$, we have
$$
\|f_{M,D,\lambda}^{\diamond}-f_{M,\lambda}\|_{{\cal H}_M}=\|{\bfomega}_{M,D,\lambda}^{\diamond}-{\bfomega}_{M,\lambda}\|_2,
$$
and by the equality that $A^{-1}-B^{-1}=A^{-1}(B-A)B^{-1}$ for $A$ and $B$ are invertible operators, we have
\begin{align*}
{\bfomega}_{M,D,\lambda}^{\diamond}-{\bfomega}_{M,\lambda}=&(C_{M,D}+\lambda I)^{-1}S^*_{M,D}f_{\tau}^*-   (C_{M}+\lambda I)^{-1}S^*_{M}f_{\tau}^*\\
=&(C_{M,D}+\lambda I)^{-1}(S^*_{M,D}-S^*_{M})f_{\tau}^*+[(C_{M,D}+\lambda I)^{-1}-(C_{M}+\lambda I)^{-1}]S^*_{M}f_{\tau}^*\\
=&(C_{M,D}+\lambda I)^{-1}(S^*_{M,D}-S^*_{M})f_{\tau}^*+(C_{M,D}+\lambda I)^{-1}(C_M-C_{M,D}){\bfomega}_{M,\lambda}\\
=&(C_{M,D}+\lambda I)^{-1}(S^*_{M,D}-S^*_{M})f_{\tau}^*+(C_{M,D}+\lambda I)^{-1}(S_M^*-S_{M,D}^*)S_M{\bfomega}_{M,\lambda}\\
=&(C_{M,D}+\lambda I)^{-1}S^*_{M,D}(f_{\tau}^*-f_{M,\lambda})+(C_{M,D}+\lambda I)^{-1}S_M^*(f_{M,\lambda}-f_{\tau}^*),
\end{align*}
where the fourth equality uses $C_{M,D}=S^*_{M,D}S_M$ and $C_M=S^*_{M}S_M$. Thus we have
\begin{align*}
\|{\bfomega}_{M,D,\lambda}^{\diamond}-{\bfomega}_{M,\lambda}\|_2\leq & \left(\|(C_{M,D}+\lambda I)^{-1}S^*_{M,D}\|+\|(C_{M,D}+\lambda I)^{-1}S_M^*\|\right)\|f_{M,\lambda}-f_{\tau}^*\|_{\rho} \\
\leq & \|(C_{M,D}+\lambda I)^{-1/2}\|\left(\|(C_{M,D}+\lambda I)^{-1/2}S^*_{M,D}\|+\|(C_{M,D}+\lambda I)^{-1/2}S_M^*\|\right)\|f_{M,\lambda}-f_{\tau}^*\|_{\rho}.
\end{align*}
Note that $C_{M,D}$ is self-adjoint and positive operator, we have $\|(C_{M,D}+\lambda I)^{-1/2}\| \leq \lambda^{-1/2}$. On the other hand, it holds that  
$$
\|(C_{M,D}+\lambda I)^{-1/2}S^*_{M,D}\|=\|(C_{M,D}+\lambda I)^{-1/2}C_{M,D}(C_{M,D}+\lambda I)^{-1/2}\|^{1/2}\leq 1,
$$
and 
\begin{align*}
\|(C_{M,D}+\lambda I)^{-1/2}S_M^*\|&=\|(C_{M,D}+\lambda I)^{-1/2}(C_{M}+\lambda I)^{1/2}(C_{M}+\lambda I)^{-1/2}S_M^*\|\\
&\leq \|(C_{M,D}+\lambda I)^{-1/2}(C_{M}+\lambda I)^{1/2}\|\|(C_{M}+\lambda I)^{-1/2}S_M^*\|\\
&= {\cal Q}_{M,D,\lambda}\|(C_{M}+\lambda I)^{-1/2}C_{M}(C_{M}+\lambda I)^{-1/2}\|^{1/2} \leq \sqrt{2},
\end{align*}
where the second equality uses the fact that $\|AB\|=\|BA\|$ for $A$ and $B$ are self-adjoint operators, and the last inequality uses Lemma \ref{lemC.4} that ${\cal Q}_{M,D,\lambda} \leq \sqrt{2}$ and  $\|(C_{M}+\lambda I)^{-1/2}C_{M}(C_{M}+\lambda I)^{-1/2}\|^{1/2} \leq 1$.
Combine these inequalities and Lemma \ref{lemA.2.2}, we get that 
$$
\|{\bfomega}_{M,D,\lambda}^{\diamond}-{\bfomega}_{M,\lambda}\|_2\leq (1+\sqrt{2})\lambda^{-1/2}\|f_{M,\lambda}-f_{\tau}^*\|_{\rho} \leq \widetilde{C}_1\lambda^{r-1/2},
$$
where $\widetilde{C}_1=(1+\sqrt{2})R$. Similarly, we have
\begin{align*}
&f_{M,D,\lambda}^{\diamond}-f_{M,\lambda}=S_M({\bfomega}_{M,D,\lambda}^{\diamond}-{\bfomega}_{M,\lambda})\\
=&S_M(C_{M,D}+\lambda I)^{-1}S^*_{M,D}(f_{\tau}^*-f_{M,\lambda})+S_M(C_{M,D}+\lambda I)^{-1}S_M^*(f_{M,\lambda}-f_{\tau}^*)\\
=&S_M(C_{M}+\lambda I)^{-1/2}(C_{M}+\lambda I)^{1/2}(C_{M,D}+\lambda I)^{-1/2}(C_{M,D}+\lambda I)^{-1/2}S^*_{M,D}(f_{\tau}^*-f_{M,\lambda})+S_M(C_{M}+\lambda I)^{-1/2}\\
&(C_{M}+\lambda I)^{1/2}(C_{M,D}+\lambda I)^{-1/2}(C_{M,D}+\lambda I)^{-1/2}(C_{M}+\lambda I)^{1/2}(C_{M}+\lambda I)^{-1/2}S_M^*(f_{M,\lambda}-f_{\tau}^*).
\end{align*}
Note that 
$$
\|S_M(C_{M}+\lambda I)^{-1/2}\|=\|(C_{M}+\lambda I)^{-1/2}C_M(C_{M}+\lambda I)^{-1/2}\|^{1/2} \leq 1,
$$
and  the above inequalities that $\|(C_{M,D}+\lambda I)^{-1/2}S^*_{M,D}\|\leq 1$ and $\|(C_{M}+\lambda I)^{-1/2}S_M^*\|\leq 1$, we have 
$$
\|f_{M,D,\lambda}^{\diamond}-f_{M,\lambda}\|_{\rho}\leq ({\cal Q}_{M,D,\lambda}+{\cal Q}_{M,D,\lambda}^2)\|f_{M,\lambda}-f_{\tau}^*\|_{\rho}  \leq (2+\sqrt{2})R\lambda^r=\sqrt{2}\widetilde{C}_1\lambda^r.
$$
Thus we complete the proof.
\end{proof}

The following proposition states the convergence rate of $f_{M,D,\lambda}^{\diamond}$ under some mild conditions on $r,\gamma, \lambda$, and $M$.

\begin{proposition}\label{propA.2.1}
Under Assumptions \ref{ass4}-\ref{ass2} and \ref{ass5}, if $r\in [0,1], \gamma\in[0,1]$, $2r+\gamma\geq 1$, and $\lambda=|D|^{-\frac{1}{2r+\gamma}}$, when the number of random features satisfies the following inequalities 
\begin{align*}
&M \gtrsim  |D|^{\frac{\alpha}{2r+\gamma}},\quad \text{for} \quad r\in (0,1/2), \\
&M \gtrsim  |D|^{\frac{(2r-1)(1+\gamma-\alpha)+\alpha}{2r+\gamma}},   \quad \text{for} \quad r\in [1/2,1],
\end{align*}
and $|D|$ is sufficiently large, then with probability near to $1$, there holds 
\begin{align}\label{last_3_error}
\|f_{M,D,\lambda}^{\diamond}-f_{\tau}^*\|_{\rho}\leq \widetilde{C}_2|D|^{-\frac{r}{2r+\gamma}}, 
\end{align}
where $\widetilde{C}_2=2R+\sqrt{2}\widetilde{C}_1$.
\end{proposition}
\begin{proof}
Combining Lemmas \ref{lemA.2.1}-\ref{lemA.2.3}, and setting $\lambda=|D|^{-\frac{1}{2r+\gamma}}$, we can obtain that 
$$
\|f_{M,D,\lambda}^{\diamond}-f_{\tau}^*\|_{\rho}\leq (2R+\sqrt{2}\widetilde{C}_1)\lambda^{r}=(2R+\sqrt{2}\widetilde{C}_1)|D|^{-\frac{r}{2r+\gamma}},    
$$
with probability near to $1$. Now we check the following conditions for $M$ and $|D|$,
\begin{align*}
&M \geq 16({\cal N}_\infty(\lambda)+1)\log (4/\delta),\quad \text{for} \quad r\in (0,1/2);\\
&M \geq 16({\cal N}_\infty(\lambda)+1)\log (4/\delta) \vee  128\kappa^2 \lambda^{1-2r}{\cal N}(\lambda)^{2r-1}{\cal N}_\infty(\lambda)^{2-2r}\log(4/\delta),   \quad \text{for} \quad r\in [1/2,1];\\
& |D| \geq 32(\kappa^2\lambda^{-1}+1)\log (6/\delta).
\end{align*}
Recalling Assumptions \ref{ass2} and \ref{ass5} that ${\cal N}(\lambda) \leq Q^2 \lambda^{-\gamma}$, ${\cal N}_\infty(\lambda) \leq F\lambda^{-\alpha}$, and $\lambda=|D|^{-\frac{1}{2r+\gamma}}$, we have 
$$
M \gtrsim \lambda^{-\alpha} = |D|^{\frac{\alpha}{2r+\gamma}},\quad \text{for} \quad r\in (0,1/2),
$$
and 
$$
M \gtrsim \lambda^{-\alpha} \vee \lambda^{(2r-2)\alpha+(1-2r)(\gamma+1)} =|D|^{\frac{(2r-1)(1+\gamma-\alpha)+\alpha}{2r+\gamma}},\quad \text{for} \quad r\in [1/2,1],
$$
and 
$$
|D|\gtrsim \lambda^{-1}=|D|^{\frac{1}{2r+\gamma}} \longrightarrow |D| \ \text{is sufficiently large \quad and}\quad  2r+\gamma \geq 1.
$$
Thus we complete the proof.
\end{proof}
\subsubsection{LS-approximation Error}
Now we are ready to provide the bound for the LS-approximation error. We first give a lemma that establishes the connection between the $L_{\rho_{\cal X}}^2$ error term $\|f-f_{M,D,\lambda}^{\diamond}\|^2_{\rho}$ and the excess risk error term $E\big[\rho_{\tau}(y-f({\bfx}))-\rho_{\tau}(y-f_{M,D,\lambda}^{\diamond}(\bfx))\big]$ for any $f \in L_{\rho_{\cal X}}^2$. This lemma heavily relies on the adaptive self-calibration condition governing the conditional distribution of $y$ (see Assumption \ref{ass6}). To use this assumption, we need the conclusion on Proposition \ref{propA.2.1} that under mild condition that $f_{M,D,\lambda}^{\diamond}$ lies in the ball center at $f_{\tau}^*$ with radius $ \|f_{M,D,\lambda}^{\diamond}-f_{\tau}^*\|_{\rho} \leq \varepsilon$ for $\varepsilon\leq 1$ when $|D|$ is large enough.

\begin{lemma}\label{lemA.2.4}
Suppose that Assumptions 3.1-3.6 and the conditions in Proposition \ref{propA.2.1} are satisfied, for any $f \in L_{\rho_{\cal X}}^2$, if $|D|\geq \widetilde{C}_3$, then with probability near to 1, there holds 
$$
\|f-f_{M,D,\lambda}^{\diamond}\|^2_{\rho} \leq \frac{4}{c_2}E\big[\rho_{\tau}(y-f({\bfx}))-\rho_{\tau}(y-f_{M,D,\lambda}^{\diamond}(\bfx))\big]+\frac{4c_1^2\widetilde{C}^2_2}{c_2^2}\|f_{M,D,\lambda}^{\diamond}-f_{\tau}^*\|^2_{\rho},
$$    
where $c_1,c_2,\widetilde{C}_2$ and $\widetilde{C}_3$ are some universal positive constants .
\end{lemma}

\begin{proof}
Using Knight's identity that $\rho_{\tau}(u-v)-\rho_{\tau}(u)=-v\big(\tau-I(u\leq 0)\big)+\int_{0}^v\big(I(u\leq t)-I(u \leq 0)\big)dt$, we have 
\begin{align*}
&\rho_{\tau}(y-f({\bfx}))-\rho_{\tau}(y-f_{M,D,\lambda}^{\diamond}(\bfx))  =-(f(\bfx)-f_{M,D,\lambda}^{\diamond}(\bfx)) \big(\tau-I(y\leq f_{M,D,\lambda}^{\diamond}(\bfx))\big)\\
+&\int_{0}^{f({\bfx})-f_{M,D,\lambda}^{\diamond}(\bfx)}\big(I(y\leq f_{M,D,\lambda}^{\diamond}(\bfx)+t)-I( y\leq f_{M,D,\lambda}^{\diamond}(\bfx))\big)dt\\
=&-(f(\bfx)-f_{M,D,\lambda}^{\diamond}(\bfx)) \big(\tau-I(y\leq f_{\tau}^{*}(\bfx))\big)-(f(\bfx)-f_{M,D,\lambda}^{\diamond}(\bfx)) \big(I(y\leq f_{\tau}^{*}(\bfx))-I( y\leq f_{M,D,\lambda}^{\diamond}(\bfx))\big)\\
+&\int_{0}^{f({\bfx})-f_{M,D,\lambda}^{\diamond}(\bfx)}\big(I(y\leq f_{M,D,\lambda}^{\diamond}(\bfx)+t)-I( y\leq f_{M,D,\lambda}^{\diamond}(\bfx))\big)dt.
\end{align*}
Here we take the expectation and using Fubini's theorem, 
\begin{equation}\label{lemA.2.4.1}
\begin{aligned}
 &E\big[\rho_{\tau}(y-f({\bfx}))-\rho_{\tau}(y-f_{M,D,\lambda}^{\diamond}(\bfx))\big] = -E\big[(f(\bfx)-f_{M,D,\lambda}^{\diamond}(\bfx)) E((\tau-I(y\leq f_{\tau}^{*}(\bfx))|\bfx)\big]\\
 -&E\big[(f(\bfx)-f_{M,D,\lambda}^{\diamond}(\bfx)) E((I(y\leq f_{\tau}^{*}(\bfx))-I(y\leq f_{M,D,\lambda}^{\diamond}(\bfx)))|\bfx)\big]\\
 +&E\left[\int_{0}^{f({\bfx})-f_{M,D,\lambda}^{\diamond}(\bfx)}\big[E(I(y\leq f_{M,D,\lambda}^{\diamond}(\bfx)+t)|\bfx)-E(I( y\leq f_{M,D,\lambda}^{\diamond}(\bfx))|\bfx)\big]dt\right].
\end{aligned}    
\end{equation}
The first term on the right side of \eqref{lemA.2.4.1} is $0$ due to the fact that $P(y\leq f_{\tau}^{*}(\bfx)|\bfx )=\tau$. For the second term, 
\begin{align*}
&E\big[(f(\bfx)-f_{M,D,\lambda}^{\diamond}(\bfx)) E((I(y\leq f_{\tau}^{*}(\bfx))-I(y\leq f_{M,D,\lambda}^{\diamond}(\bfx)))|\bfx)\big] \\
\leq & E\big[|f(\bfx)-f_{M,D,\lambda}^{\diamond}(\bfx)| |E((I(y\leq f_{\tau}^{*}(\bfx))-I(y\leq f_{M,D,\lambda}^{\diamond}(\bfx)))|\bfx)|\big]\\
= &E\big[|f(\bfx)-f_{M,D,\lambda}^{\diamond}(\bfx)| |F_{y|\bfx}(f_{M,D,\lambda}^{\diamond}(\bfx))-F_{y|\bfx}(f_{\tau}^{*}(\bfx))|\big]\\
= &E\big[|f(\bfx)-f_{M,D,\lambda}^{\diamond}(\bfx)| |f_{y|\bfx}(\xi)(f_{M,D,\lambda}^{\diamond}(\bfx)-f_{\tau}^{*}(\bfx))|\big]\\
\leq & c_1E\big[|f(\bfx)-f_{M,D,\lambda}^{\diamond}(\bfx)| |f_{M,D,\lambda}^{\diamond}(\bfx)-f_{\tau}^{*}(\bfx)|\big]\\
\leq & c_1\sqrt{E\big[(f(\bfx)-f_{M,D,\lambda}^{\diamond}(\bfx))^2\big]}\sqrt{E\big[(f_{M,D,\lambda}^{\diamond}(\bfx)-f_{\tau}^{*}(\bfx))^2\big]}=c_1\|f-f_{M,D,\lambda}^{\diamond}\|_{\rho}\|f_{M,D,\lambda}^{\diamond}-f_{\tau}^{*}\|_{\rho},
\end{align*}
where the first inequality is from $E(AB)\leq E(|A||B|)$ for any random variable $A$ and $B$, the second equality is from the mean value theorem with $\xi \in [f_{M,D,\lambda}^{\diamond}(\bfx),f_{\tau}^{*}(\bfx)]$ or $\xi \in [f_{\tau}^{*}(\bfx), f_{M,D,\lambda}^{\diamond}(\bfx)]$ together with Assumption \ref{ass6} that the conditional density $f_{y|\bfx}(\cdot)$ is uniformly bounded, and the last inequality is from the Cauchy-Schwarz inequality. Similarly, for the third term on the right side of \eqref{lemA.2.4.1},
\begin{align*}
&E\left[\int_{0}^{f({\bfx})-f_{M,D,\lambda}^{\diamond}(\bfx)}\big[E(I(y\leq f_{M,D,\lambda}^{\diamond}(\bfx)+t)|\bfx)-E(I( y\leq f_{M,D,\lambda}^{\diamond}(\bfx))|\bfx)\big]dt\right] \\
=& E\left[\int_{0}^{f({\bfx})-f_{M,D,\lambda}^{\diamond}(\bfx)}\big[F_{y|\bfx}(f_{M,D,\lambda}^{\diamond}(\bfx)+t)-F_{y|\bfx}(f_{M,D,\lambda}^{\diamond}(\bfx))\big]dt\right]\\
\geq&c_2E\left[\int_{0}^{f({\bfx})-f_{M,D,\lambda}^{\diamond}(\bfx)}tdt\right]=\frac{c_2}{2}\|f-f_{M,D,\lambda}^{\diamond}\|^2_{\rho},
\end{align*}
where the inequality is from Assumption \ref{ass6} and Propsition \ref{propA.2.1} that $\|f_{M,D,\lambda}^{\diamond}-f_{\tau}^*\|_{\rho}\leq \xi$ when $|D|\geq (\widetilde{C}_2/\xi)^{2+\gamma/r}$.

Plug these results into \eqref{lemA.2.4.1}, we get 
\begin{align*}
\frac{c_2}{2}\|f-f_{M,D,\lambda}^{\diamond}\|^2_{\rho}&\leq c_1\|f-f_{M,D,\lambda}^{\diamond}\|_{\rho}\|f_{M,D,\lambda}^{\diamond}-f_{\tau}^{*}\|_{\rho}+E\big[\rho_{\tau}(y-f({\bfx}))-\rho_{\tau}(y-f_{M,D,\lambda}^{\diamond}(\bfx))\big]\\
&\leq \frac{c_1}{4\beta}\|f-f_{M,D,\lambda}^{\diamond}\|_{\rho}^2+c_1\beta\|f_{M,D,\lambda}^{\diamond}-f_{\tau}^{*}\|^2_{\rho}+E\big[\rho_{\tau}(y-f({\bfx}))-\rho_{\tau}(y-f_{M,D,\lambda}^{\diamond}(\bfx))\big],
\end{align*}
then we set $\beta=c_1/c_2$, it holds that 
\begin{align*}
\|f-f_{M,D,\lambda}^{\diamond}\|^2_{\rho} &\leq \frac{4}{c_2}E\big[\rho_{\tau}(y-f({\bfx}))-\rho_{\tau}(y-f_{M,D,\lambda}^{\diamond}(\bfx))\big]+\frac{4c_1^2}{c_2^2}\|f_{M,D,\lambda}^{\diamond}-f_{\tau}^{*}\|^2_{\rho} \\
&\leq \frac{4}{c_2}E\big[\rho_{\tau}(y-f({\bfx}))-\rho_{\tau}(y-f_{M,D,\lambda}^{\diamond}(\bfx))\big]+\frac{4c_1^2\widetilde{C}^2_2}{c_2^2}|D|^{-\frac{2r}{2r+\gamma}},
\end{align*}
with probability near to $1$. Thus we completes the proof.    
\end{proof} 

The following lemma bounds the supremum of the difference between the empirical average dependent on the data $\frac{1}{|D|}\sum_{i=1}^{|D|}[\rho_{\tau}(y_i-f(\bfx_i))-\rho_{\tau}(y_i-f_{M,D,\lambda}^{\diamond}(\bfx_i))]$ and its expectation $E[\rho_{\tau}(y-f(\bfx))-\rho_{\tau}(y-f_{M,D,\lambda}^{\diamond}(\bfx)]$ within a local ball using the  Rademacher complexity function on ${\cal H}_{M}$
$$
\mathcal{R}_{M}(\delta)=\sqrt{\frac{1}{|D|}\sum_{j=1}^{\infty}\min\{\mu_j,\delta^2\}},
$$
where $\mu_j$'s are the eigenvalues of the spectral decomposition of $L_M$.  Recall the definition of the effective dimension of ${\cal H}_{M}$ that ${\cal{N}}_M(\lambda)=\text{Tr}((L_M+\lambda I)^{-1}L_M)=\sum_{j=1}^{\infty}\mu_j/(\mu_j+\lambda)$. It is easy to verify that ${\cal{N}}_M(\lambda) \asymp |D|\mathcal{R}^2_{M}(\sqrt{\lambda})/\lambda$ (using inequality that $\min(a,b)/2 \leq \frac{ab}{a+b} \leq \min(a,b)$ for $a,b \in \mathbb{R}$ ). So these two quantities are equivalent to some extent, and according to Lemma \ref{lemD.7}, under some mild conditions on the number of random features, we have $\mathcal{R}_{M}(\delta)\asymp \mathcal{R}(\delta)$, where $\mathcal{R}(\delta)$ is the  Rademacher complexity function on ${\cal H}_{K}$ defined by
$$
\mathcal{R}(\delta)=\sqrt{\frac{1}{|D|}\sum_{j=1}^{\infty}\min\{\mu^{\prime}_j,\delta^2\}},
$$
where $\mu^{\prime}_j$'s are the eigenvalues of the spectral decomposition of $L_K$.

\begin{lemma}\label{lemA.2.5}
For any $\delta>0$ and $f\in {\cal H}_M$, we define the event $\cal{M}(\delta)$ as 
{\small$$
\left\{\sup_{f \in \Theta(\delta)}\left|\frac{1}{|D|}\sum_{i=1}^{|D|}[\rho_{\tau}(y_i-f(\bfx_i))-\rho_{\tau}(y_i-f_{M,D,\lambda}^{\diamond}(\bfx_i))]-E[\rho_{\tau}(y-f(\bfx))-\rho_{\tau}(y-f_{M,D,\lambda}^{\diamond}(\bfx)]\right|
\leq C\log |D| \mathcal{R}_{M}(\delta)\right\},
$$}
where $\Theta(\delta):=\{f \in \mathcal{H}_M \mid \|f-f_{M,D,\lambda}^{\diamond}\|_{\rho} \leq \delta, \text{and}\  \|f-f_{M,D,\lambda}^{\diamond}\|_{{\cal H}_M} \leq 1\}$, then $\cal{M}(\delta)$ holds with probability near to 1.
\end{lemma}

\begin{proof}
For the notation simplify, we denote
$$
A=\left|\frac{1}{|D|}\sum_{i=1}^{|D|}[\rho_{\tau}(y_i-f(\bfx_i))-\rho_{\tau}(y_i-f_{M,D,\lambda}^{\diamond}(\bfx_i))]-E[\rho_{\tau}(y-f(\bfx))-\rho_{\tau}(y-f_{M,D,\lambda}^{\diamond}(\bfx)]\right|,
$$
and $C$ is a universal positive constant that may be different from line to line in this lemma.

We first use the standard symmetrization argument in the empirical process \citep{pollard2012convergence} to bound $E[A]$ such that
\begin{equation}\label{lemA.2.5.2}
\begin{aligned}
E[A]&\leq 2E\left[\sup_{f \in \Theta(\delta)}\left |\frac{1}{|D|}\sum_{i=1}^{|D|}\sigma_i\big (\rho_{\tau}(y_i-f({\bfx}_i))-\rho_{\tau}(y_i-f_{M,D,\lambda}^{\diamond}({\bfx}_i))\big)\right | \right]  \\
& \leq  4E\left[\sup_{f \in \Theta(\delta)}\left |\frac{1}{|D|}\sum_{i=1}^{|D|}\sigma_i\big (f({\bfx}_i)-f_{M,D,\lambda}^{\diamond}({\bfx}_i)\big)\right | \right],
\end{aligned}     
\end{equation}
where $\{\sigma_i\}'s$ denote the Rademacher variables taking values in $\{-1,1\}$ with equal probability, the second inequality follows from the fact that $\rho_{\tau}(\cdot)$ is $1$-Lipschitz continuous and the Ledoux–Talagrand contraction
inequality \citep{wainwright2019high}.

For any $f \in \Theta(\delta)$, we denote $g=f-f_{M,D,\lambda}^{\diamond} \in {\cal H}_M$, and $g=\sum_{j=1}^{\infty}g_j\psi_j$ with $g_j=\int_{\cal X}f(\bfx)\psi_j(\bfx)\rho_{\cal{X}}(\bfx)d\bfx$. Note that $\|g\|_{\rho}\leq \delta$ and $\|g\|_{{\cal H}_M}\leq 1$, this implies that $\sum_{j=1}^{\infty}g_j^2\leq \delta^2$ and $\sum_{j=1}^{\infty}g_j^2/\mu_j \leq 1$. Combine these two inequalities, we have
\begin{align}\label{lemA.2.5.3}
\sum_{j=1}^{\infty}\frac{g_j^2}{\min\{\mu_j,\delta^2\}}\leq 2.    
\end{align}
Then we have
\begin{equation}\label{lemA.2.5.4}
\begin{aligned}
&\left|\sum_{i=1}^{|D|}\sigma_i\big (f({\bfx}_i)-f_{M,D,\lambda}^{\diamond}({\bfx}_i)\big)\right|=    \left|\sum_{i=1}^{|D|}\sigma_i\sum_{j=1}^{\infty}g_j\psi_j(\bfx_i)\right|\\
=&\left|\sum_{j=1}^{\infty}\frac{g_j}{\sqrt{\min\{\mu_j,\delta^2\}}}\sqrt{\min\{\mu_j,\delta^2\}}\sum_{i=1}^{|D|}\sigma_i\psi_j(\bfx_i)\right|\leq \sqrt{2}\left\{\sum_{j=1}^{\infty}\min\{\mu_j,\delta^2\}\left(\sum_{i=1}^{|D|}\sigma_i\psi_j(\bfx_i)\right)^2\right\}^{1/2},
\end{aligned}    
\end{equation}
where the inequality is from Cauthy-Schwarz inequality and \eqref{lemA.2.5.3}. Plug \eqref{lemA.2.5.4} into \eqref{lemA.2.5.2}, we have 
\begin{align*}
E[A]&\leq \frac{4\sqrt{2}}{|D|}E\left\{\sum_{j=1}^{\infty}\min\{\mu_j,\delta^2\}\left(\sum_{i=1}^{|D|}\sigma_i\psi_j(\bfx_i)\right)^2\right\}^{1/2}\\
&\leq \frac{4\sqrt{2}}{|D|}\left\{\sum_{j=1}^{\infty}\min\{\mu_j,\delta^2\}E_{\bfx,\sigma}\left(\sum_{i=1}^{|D|}\sigma_i\psi_j(\bfx_i)\right)^2\right\}^{1/2}\\
&=\frac{4\sqrt{2}}{|D|}\left\{\sum_{j=1}^{\infty}\min\{\mu_j,\delta^2\}\sum_{i=1}^{|D|}E_{\bfx,\sigma}\left(\sigma_i^2\psi^2_j(\bfx_i)\right)\right\}^{1/2}\\
&\leq \frac{4\sqrt{2}C}{|D|}\sqrt{\sum_{j=1}^{\infty}\min\{\mu_j,\delta^2\}},
\end{align*}
where the second inequality follows from Jensen's inequality, and the first equality is from the fact that $E_{\bfx,\sigma}(\sigma_i\psi_j(\bfx_i))=0$ for each $i$. Thus we have 
\begin{align}\label{lemA.2.5.5}
E[A]\leq C\mathcal{R}_M(\delta).  
\end{align}

Next, we turn to bound $A-E[A]$, note that
\begin{align*}
|g(\bfx)|=\left|\sum_{j=1}^{\infty}g_j\psi_j(\bfx)\right|&\leq \sqrt{\sum_{j=1}^{\infty}\frac{g_j^2}{\min\{\mu_j,\delta^2\}}}\sqrt{\sum_{j=1}^{\infty}\min\{\mu_j,\delta^2\}\psi_j^2(\bfx)}\leq C\sqrt{\sum_{j=1}^{\infty}\min\{\mu_j,\delta^2\}}=C\sqrt{|D|}\mathcal{R}_M(\delta).
\end{align*}
Thus we have
$$
\big|\rho_{\tau}(y-f({\bfx}))-\rho_{\tau}(y-f_{M,D,\lambda}^{\diamond}({\bfx}))\big|\leq  \big|f({\bfx})-f_{M,D,\lambda}^{\diamond}({\bfx})\big|=|g(\bfx)|\leq C\sqrt{|D|}\mathcal{R}_M(\delta),
$$
and
$$
E\big[\rho_{\tau}(y-f({\bfx}))-\rho_{\tau}(y-f_{M,D,\lambda}^{\diamond}({\bfx}))\big]^2\leq  E\big(f({\bfx})-f_{M,D,\lambda}^{\diamond}({\bfx})\big)^2=E(g(\bfx))^2\leq C|D|\mathcal{R}^2_M(\delta).
$$

With these two inequalities, we use the Bousquet bound inequality in Lemma \ref{lemD.8} and set $t=C\sqrt{\frac{\log |D|}{|D|}}$, then 
\begin{align}\label{lemA.2.5.6}
A-E[A]\leq C\log|D| \mathcal{R}_M(\delta)    
\end{align}
  
holds with probability at least $1-n^{-C}$. 

Combine \eqref{lemA.2.5.5} and \eqref{lemA.2.5.6}, we can obtain the inequality in the lemma. Thus we complete the proof.

\end{proof}

According to Lemma \ref{lemA.2.5}, we can also get the following inequality by some normalized procedure,
\begin{equation}\label{lemA.2.5.normal}
\begin{aligned}
&\left|\frac{1}{|D|}\sum_{i=1}^{|D|}[\rho_{\tau}(y_i-f(\bfx_i))-\rho_{\tau}(y_i-f_{M,D,\lambda}^{\diamond}(\bfx_i))]-E[\rho_{\tau}(y-f(\bfx))-\rho_{\tau}(y-f_{M,D,\lambda}^{\diamond}(\bfx)]\right|\\
\leq& C\log |D| \mathcal{R}_{M}(\delta)\left(\frac{\|f-f_{M,D,\lambda}^{\diamond}\|_{\rho}}{\delta}+\|f-f_{M,D,\lambda}^{\diamond}\|_{{\cal H}_M}\right)   
\end{aligned}    
\end{equation}

\begin{lemma}\label{lemA.2.6}
Suppose that Assumptions \ref{ass4}-\ref{ass6} and \ref{ass5} and the conditions in Proposition \ref{propA.2.1} are satisfied, if $|D|\geq \widetilde{C}_3$, then with probability near to 1, there holds    
$$
\|f_{M,D,\lambda}-f_{M,D,\lambda}^{\diamond}\|_{\rho} \leq C |D|^{-\frac{r}{2r+\gamma}}\log |D|,
$$
where $C$ is a universal positive constant. 
\end{lemma}

\begin{proof}
Recall the definition of $f_{M,D,\lambda}$, we have 
$$
\frac{1}{|D|}\sum_{i=1}^{|D|}\rho_{\tau}(y_i-f_{M,D,\lambda}(\bfx_i))+\lambda\|f_{M,D,\lambda}\|_{{\cal H}_M} \leq \frac{1}{|D|}\sum_{i=1}^{|D|}\rho_{\tau}(y_i-f_{M,D,\lambda}^{\diamond}(\bfx_i))+\lambda\|f_{M,D,\lambda}^{\diamond}\|_{{\cal H}_M}.
$$

We can not directly obtain a upper bound of $E[\rho_{\tau}(y-f_{M,D,\lambda}(\bfx))-\rho_{\tau}(y-f_{M,D,\lambda}^{\diamond}(\bfx)]$ according to \eqref{lemA.2.5.normal} from Lemma \ref{lemA.2.5},  because $E[\rho_{\tau}(y-f_{M,D,\lambda}(\bfx))-\rho_{\tau}(y-f_{M,D,\lambda}^{\diamond}(\bfx)] \geq 0$ does not always hold. Thus we use Lemma \ref{lemA.2.4} and note that  $E[\rho_{\tau}(y-f_{M,D,\lambda}(\bfx))-\rho_{\tau}(y-f_{M,D,\lambda}^{\diamond}(\bfx)]+C\|f_{M,D,\lambda}^{\diamond}-f_{\tau}^*\|\geq 0$ holds. Thus combine Lemma \ref{lemA.2.4} and \eqref{lemA.2.5.normal} from Lemma \ref{lemA.2.5} (plus a $C |D|^{-\frac{r}{2r+\gamma}}$ term and minus the same term in the left of \eqref{lemA.2.5.normal}) , with probability near to $1$, we have
\begin{equation}\label{lemA.2.6.1}
\begin{aligned}
&E[\rho_{\tau}(y-f_{M,D,\lambda}(\bfx))-\rho_{\tau}(y-f_{M,D,\lambda}^{\diamond}(\bfx)] +C |D|^{-\frac{2r}{2r+\gamma}} \\
\leq & \lambda\|f_{M,D,\lambda}^{\diamond}\|_{{\cal H}_M}-\lambda\|f_{M,D,\lambda}\|_{{\cal H}_M}+C\log|D|\frac{\mathcal{R}_{M}(\delta)}{\delta}\|f-f_{M,D,\lambda}^{\diamond}\|_{\rho}+C\log|D|\mathcal{R}_{M}(\delta)\|f-f_{M,D,\lambda}^{\diamond}\|_{{\cal H}_M}+C |D|^{-\frac{2r}{2r+\gamma}}\\
= & -2\lambda \langle f_{M,D,\lambda}^{\diamond}, f_{M,D,\lambda}-f_{M,D,\lambda}^{\diamond}\rangle_{{\cal H}_M}-\lambda \|f_{M,D,\lambda}-f_{M,D,\lambda}^{\diamond}\|_{{\cal H}_M}^2+C\log|D|\frac{\mathcal{R}_{M}(\delta)}{\delta}\|f_{M,D,\lambda}-f_{M,D,\lambda}^{\diamond}\|_{\rho}\\
+& C\log|D|\mathcal{R}_{M}(\delta)\|f_{M,D,\lambda}-f_{M,D,\lambda}^{\diamond}\|_{{\cal H}_M}+C |D|^{-\frac{2r}{2r+\gamma}}\\
= & -2\lambda \langle f_{M,D,\lambda}^{\diamond}, f_{M,D,\lambda}-f_{M,D,\lambda}^{\diamond}\rangle_{{\cal H}_M}-\lambda \|f_{M,D,\lambda}-f_{M,D,\lambda}^{\diamond}\|_{{\cal H}_M}^2+C\lambda^r\log|D|\|f_{M,D,\lambda}-f_{M,D,\lambda}^{\diamond}\|_{\rho}\\
+& C\lambda^{r+1/2}\log|D|\|f_{M,D,\lambda}-f_{M,D,\lambda}^{\diamond}\|_{{\cal H}_M}+C \lambda^{2r},
\end{aligned}    
\end{equation}
where in the last equality, we choose $\delta$ satisfying that $\mathcal{R}_{M}(\delta)=\delta^{1+2r}$. Note that $\mathcal{R}_{M}(\delta)\asymp \mathcal{R}(\delta)\asymp \delta^{1-\gamma}/\sqrt{|D|}$ \citep{lian2022}, we can obtain that $\delta=|D|^{-\frac{1}{4r+2\gamma}}$, and $\lambda=\delta^2=|D|^{-\frac{1}{2r+\gamma}}$.

We now establish the bound of the term $-2\lambda \langle f_{M,D,\lambda}^{\diamond}, f_{M,D,\lambda}-f_{M,D,\lambda}^{\diamond}\rangle_{{\cal H}_M}$. Note that by the triangle inequality, we have 
\begin{equation}\label{lemA.2.6.2}
\begin{aligned}
&|\lambda \langle f_{M,D,\lambda}^{\diamond}, f_{M,D,\lambda}-f_{M,D,\lambda}^{\diamond}\rangle_{{\cal H}_M}|  \\
\leq& |\lambda \langle f_{M,D,\lambda}^{\diamond}-f_{M,\lambda}, f_{M,D,\lambda}-f_{M,D,\lambda}^{\diamond}\rangle_{{\cal H}_M}| +|\lambda \langle f_{M,\lambda}, f_{M,D,\lambda}-f_{M,D,\lambda}^{\diamond}\rangle_{{\cal H}_M}|. 
\end{aligned}    
\end{equation}
For the first term of the right side of \eqref{lemA.2.6.2}, we use the Cauchy–Schwarz inequality and Lemma \ref{lemA.2.3} and obtain that
\begin{align*}
|\lambda \langle f_{M,D,\lambda}^{\diamond}-f_{M,\lambda}, f_{M,D,\lambda}-f_{M,D,\lambda}^{\diamond}\rangle_{{\cal H}_M}| &\leq \lambda \|f_{M,D,\lambda}^{\diamond}-f_{M,\lambda}\|_{{\cal H}_M}\|f_{M,D,\lambda}-f_{M,D,\lambda}^{\diamond}\|_{{\cal H}_M}\\
&\leq \widetilde{C}_1\lambda^{r+1/2}\|f_{M,D,\lambda}-f_{M,D,\lambda}^{\diamond}\|_{{\cal H}_M}.
\end{align*}
For the second term in the right side of \eqref{lemA.2.6.2}, we consider the following two cases:
\begin{enumerate}
    \item[(i).] 
For the case when $r\in (0,1/2)$, recall the definition of $f_{M,\lambda}$ we get
\begin{align*}
\|f_{M,\lambda}\|_{{\cal H}_M}&=\|(L_M+\lambda I)^{-1}L_Mf_{\rho}\|_{{\cal H}_M}= \|(L_M+\lambda I)^{-1}L_ML_K^{r}h_{\rho}\|_{{\cal H}_M}\leq \|(L_M+\lambda I)^{-1}L_M\|\|L_K^{r}h_{\rho}\|_{\rho} \\
&\leq \|L_K^{r}h_{\rho}\|_{\rho} \leq \|L_K^{r}\|\|h_{\rho}\|_{\rho}\leq R\kappa^{2r},
\end{align*}
where the first and third inequality is from the fact that $(L_M+\lambda I)^{-1}L_M$ and $L_K^r$ are linear operators, the last inequality is from $\|L_K^r\|\leq \kappa^{2r}$ for $r\in (0,1/2) $ and $\|h_{\rho}\|_{\rho} \leq R$. Then by the Cauchy–Schwarz inequality, we have
\begin{align*}
|\lambda \langle f_{M,\lambda}, f_{M,D,\lambda}-f_{M,D,\lambda}^{\diamond}\rangle_{{\cal H}_M}|
&\leq \lambda \|f_{M,\lambda}\|_{{\cal H}_M}\|f_{M,D,\lambda}-f_{M,D,\lambda}^{\diamond}\|_{{\cal H}_M}\leq R\phi^{2r}\lambda\|f_{M,D,\lambda}-f_{M,D,\lambda}^{\diamond}\|_{{\cal H}_M}\\
&\leq R\kappa^{2r}\lambda^{r+1/2}\|f_{M,D,\lambda}-f_{M,D,\lambda}^{\diamond}\|_{{\cal H}_M}.
\end{align*}
\item[(ii).]
For the case when $r\in [1/2,1]$, we have 
\begin{align*}
&|\lambda \langle f_{M,\lambda}, f_{M,D,\lambda}-f_{M,D,\lambda}^{\diamond}\rangle_{{\cal H}_M}|\\
= &|\lambda \langle f_{M,\lambda}, L_M^{-1}(f_{M,D,\lambda}-f_{M,D,\lambda}^{\diamond})\rangle_{\rho}|\\
=&\lambda| \langle (L_M+\lambda I)^{-1}L_M L_M^rh_{\tau}^*, L_M^{-1}(f_{M,D,\lambda}-f_{M,D,\lambda}^{\diamond}) \rangle_{\rho}|\\
\leq &R\lambda^{r} \| \lambda^{1-r}(L_M+\lambda I)^{-1}L_M^r(f_{M,D,\lambda}-f_{M,D,\lambda}^{\diamond}) \|_\rho
\\
= &R\lambda^{r}\sqrt{ \langle f_{M,D,\lambda}-f_{M,D,\lambda}^{\diamond},   \lambda^{2-2r}L_M^{2r}(L_M+\lambda I)^{-2}(f_{M,D,\lambda}-f_{M,D,\lambda}^{\diamond}) \rangle_\rho}\\
\leq &R\lambda^{r}\sqrt{ \langle f_{M,D,\lambda}-f_{M,D,\lambda}^{\diamond},  ((2-2r)\lambda+ (2r-1)L_M)L_M(L_M+\lambda I)^{-2}(f_{M,D,\lambda}-f_{M,D,\lambda}^{\diamond}) \rangle_\rho}\\
\leq &R\lambda^{r}
\sqrt{ \langle f_{M,D,\lambda}-f_{M,D,\lambda}^{\diamond},  \lambda L_M(L_M+\lambda I)^{-2}(f_{M,D,\lambda}-f_{M,D,\lambda}^{\diamond}) \rangle_\rho}+\\
&R\lambda^{r}
\sqrt{ \langle f_{M,D,\lambda}-f_{M,D,\lambda}^{\diamond},  L_M^2(L_M+\lambda I)^{-2}(f_{M,D,\lambda}-f_{M,D,\lambda}^{\diamond}) \rangle_\rho}\\
=&R\lambda^{r+1/2}
\sqrt{ \langle f_{M,D,\lambda}-f_{M,D,\lambda}^{\diamond},   L_M^2(L_M+\lambda I)^{-2}(f_{M,D,\lambda}-f_{M,D,\lambda}^{\diamond}) \rangle_{{\cal H}_M}}+\\
&R\lambda^{r}
\sqrt{ \langle f_{M,D,\lambda}-f_{M,D,\lambda}^{\diamond},  L_M^2(L_M+\lambda I)^{-2}(f_{M,D,\lambda}-f_{M,D,\lambda}^{\diamond}) \rangle_\rho}\\
&\leq R\lambda^{r+1/2}\|f_{M,D,\lambda}-f_{M,D,\lambda}^{\diamond}\|_{{\cal H}_M}+R\lambda^r\|f_{M,D,\lambda}-f_{M,D,\lambda}^{\diamond}\|_{\rho},
\end{align*}
where we use the fact that $\|f\|_{\rho}=\|L_M^{1/2}f\|_{{\cal H}_M}$ for any $f \in L_{\rho_{\cal X}}^2$, the the second inequality uses the Young's inequality that $\lambda^{2-2r}L_M^{2r}\leq (2-2r)\lambda+2rL_M$ for the positive operator $L_M$, $\lambda>0$, and $r\in [1/2,1]$, the last inequality is from $\|(L_M+\lambda I)^{-1}L_M\|\leq 1$. This technical proof taking consideration of $r \in [1/2,1]$ is inspired from that of KQR in \citet{lian2022}.
\end{enumerate}

Plug the aforementioned two results (i) and (ii) into \eqref{lemA.2.6.2}, we have 
$$
|\lambda \langle f_{M,D,\lambda}^{\diamond}, f_{M,D,\lambda}-f_{M,D,\lambda}^{\diamond}\rangle_{{\cal H}_M}| \leq C\lambda^{r+1/2}\|f_{M,D,\lambda}-f_{M,D,\lambda}^{\diamond}\|_{{\cal H}_M}+C\lambda^r\|f_{M,D,\lambda}-f_{M,D,\lambda}^{\diamond}\|_{\rho},
$$
where $C=(\widetilde{C}_1+R(\phi^{2r}+1))+R$. Plug this result into \eqref{lemA.2.6.1}, we get 
\begin{align*}
&E[\rho_{\tau}(y-f_{M,D,\lambda}(\bfx))-\rho_{\tau}(y-f_{M,D,\lambda}^{\diamond}(\bfx)] \\
\leq & C\lambda^{r+1/2}\|f_{M,D,\lambda}-f_{M,D,\lambda}^{\diamond}\|_{{\cal H}_M}+C\lambda^r\|f_{M,D,\lambda}-f_{M,D,\lambda}^{\diamond}\|_{\rho}-\lambda \|f_{M,D,\lambda}-f_{M,D,\lambda}^{\diamond}\|_{{\cal H}_M}^2\\+& C\lambda^r\log|D|\|f_{M,D,\lambda}-f_{M,D,\lambda}^{\diamond}\|_{\rho}+
C\lambda^{r+1/2}\log|D|\|f_{M,D,\lambda}-f_{M,D,\lambda}^{\diamond}\|_{{\cal H}_M}+C \lambda^{2r}\\
\leq & C\lambda^{2r} +\frac{\lambda}{4}\|f_{M,D,\lambda}-f_{M,D,\lambda}^{\diamond}\|^2_{{\cal H}_M}+C\lambda^r\log|D|\|f_{M,D,\lambda}-f_{M,D,\lambda}^{\diamond}\|_{\rho}-\lambda \|f_{M,D,\lambda}-f_{M,D,\lambda}^{\diamond}\|_{{\cal H}_M}^2\\
+&C\lambda^{2r}\log^2|D| +\frac{\lambda}{4}\|f_{M,D,\lambda}-f_{M,D,\lambda}^{\diamond}\|^2_{{\cal H}_M}\\
\leq & C\lambda^{2r}\log^2|D| + C\lambda^r\log|D|\|f_{M,D,\lambda}-f_{M,D,\lambda}^{\diamond}\|_{\rho}=C|D|^{-\frac{2r}{2r+\gamma}}\log^2|D|+C|D|^{-\frac{r}{2r+\gamma}}\log|D|\|f_{M,D,\lambda}-f_{M,D,\lambda}^{\diamond}\|_{\rho}.
\end{align*}

By Lemma \ref{lemA.2.4}, we have
\begin{align*}
\|f_{M,D,\lambda}-f_{M,D,\lambda}^{\diamond}\|_{\rho}^2&\leq   C E[\rho_{\tau}(y-f_{M,D,\lambda}(\bfx))-\rho_{\tau}(y-f_{M,D,\lambda}^{\diamond}(\bfx)]+C  |D|^{-\frac{2r}{2r+\gamma}}\\
&\leq C|D|^{-\frac{2r}{2r+\gamma}}\log^2|D|+C|D|^{-\frac{r}{2r+\gamma}}\log|D|\|f_{M,D,\lambda}-f_{M,D,\lambda}^{\diamond}\|_{\rho}.
\end{align*}
Solve the above inequality we can finally obtain that 
$$
\|f_{M,D,\lambda}-f_{M,D,\lambda}^{\diamond}\|_{\rho} \leq C|D|^{-\frac{r}{2r+\gamma}}\log |D|.
$$
Thus we complete the proof.
\end{proof}

\subsection{Proofs of Theorems \ref{thm1} and \ref{thm2} and Corollary \ref{cor1}}
Now we are ready to prove Theorems \ref{thm1} and \ref{thm2} and Corollary \ref{cor1}.

\begin{proof}
By Proposition \ref{propA.2.1} and Lemma \ref{lemA.2.6}, if $r\in [0,1], \gamma\in[0,1], 2r+\gamma\geq 1$, and $\lambda=|D|^{-\frac{1}{2r+\gamma}}$, and the number of random features satisfies the following two inequalities
\begin{align*}
&M \gtrsim  |D|^{\frac{\alpha}{2r+\gamma}},\quad \text{for} \quad r\in (0,1/2),\\
&M \gtrsim  |D|^{\frac{(2r-1)(1+\gamma-\alpha)+\alpha}{2r+\gamma}},   \quad \text{for} \quad r\in [1/2,1],
\end{align*}
then with probability near to $1$, there holds 
$$
\|f_{M,D,\lambda}-f_{\tau}^*\|_{\rho}\leq \|f_{M,D,\lambda}-f_{M,D,\lambda}^{\diamond}\|_{\rho}+ \|f_{M,D,\lambda}^{\diamond}-f_{\tau}^*\|_{\rho} \leq (\widetilde{C}_2+C) |D|^{-\frac{r}{2r+\gamma}}\log|D|,
$$
so we have $\|f_{M,D,\lambda}-f_{M,D,\lambda}^{\diamond}\|_{\rho} \leq C|D|^{-\frac{r}{2r+\gamma}}\log^2 |D|$. Recall the knight inequality that 
\begin{align*}
&\rho_{\tau}(y-f_{M,D,\lambda}({\bfx}))-\rho_{\tau}(y-f_{\tau}^*(\bfx))  =-(f_{M,D,\lambda}(\bfx)-f_{\tau}^*(\bfx)) \big(\tau-I(y\leq f_{\tau}^*(\bfx))\big)\\
+&\int_{0}^{f_{M,D,\lambda}({\bfx})-f_{\tau}^*(\bfx)}\big(I(y\leq f_{\tau}^*(\bfx)+t)-I( y\leq f_{\tau}^*(\bfx))\big)dt.
\end{align*}
Taking the expectation and using Fubini's theorem, we obtain that
\begin{align*}
 &E\big[\rho_{\tau}(y-f_{M,D,\lambda}({\bfx}))-\rho_{\tau}(y-f_{\tau}^*(\bfx))\big] = -E\big[(f_{M,D,\lambda}(\bfx)-f_{M,D,\lambda}^{\diamond}(\bfx)) E((\tau-I(y\leq f_{\tau}^{*}(\bfx))|\bfx)\big]\\
 +&E\left[\int_{0}^{f_{M,D,\lambda}({\bfx})-f_{\tau}^*(\bfx)}\big[E(I(y\leq f_{\tau}^*(\bfx)+t)|\bfx)-E(I( y\leq f_{\tau}^*(\bfx))|\bfx)\big]dt\right].
\end{align*}   
The first term is 0 due to that fact that $E((\tau-I(y\leq f_{\tau}^{*}(\bfx))|\bfx)=0$ and the second term can be bounded by 
\begin{align*}
&E\left[\int_{0}^{f_{M,D,\lambda}({\bfx})-f_{\tau}^{*}(\bfx)}\big[E(I(y\leq f_{\tau}^{*}(\bfx)+t)|\bfx)-E(I( y\leq f_{\tau}^{*}(\bfx))|\bfx)\big]dt\right] \\
=& E\left[\int_{0}^{f_{M,D,\lambda}({\bfx})-f_{\tau}^{*}(\bfx)}\big[F_{y|\bfx}(f_{\tau}^{*}(\bfx)+t)-F_{y|\bfx}(f_{\tau}^{*}(\bfx))\big]dt\right]\\
\leq&c_1E\left[\int_{0}^{f_{M,D,\lambda}({\bfx})-f_{\tau}^{*}(\bfx)}tdt\right]=\frac{c_1}{2}\|f_{M,D,\lambda}-f_{\tau}^{*}\|^2_{\rho},
\end{align*}
where the inequality uses Assumption \ref{ass5} that $\sup_{t \in \mathbb{R}}f_{y|\bfx}(t) \leq c_1$. Therefore, we have
$$
{\cal E}(f_{M,D,\lambda})-{\cal E}(f_{\tau}^*) \leq c_1/2 \|f_{M,D,\lambda}-f_{\tau}^*\|_{\rho}^2 \leq C |D|^{-\frac{2r}{2r+\gamma}}\log^2 |D|.
$$

Thus we complete the proof of Theorem \ref{thm2}. By Theorem \ref{thm2} with $\alpha=1$ and $\alpha=\gamma$, we can estabilsh the proofs of Theorem \ref{thm1} and Corollary \ref{cor1}.

\end{proof}

\section{Extension to the Lipschitz Loss} \label{proof_lip}

In this section, we consider random feature method with Lipschitz continuous loss function $L(\cdot,\cdot)$. Similar to the check loss case in \eqref{kqr_rf}, we approximate $y_i$ with $f_{M,D,\lambda}^L=\widetilde{\bfu}\phi_M$ and formulate the following general learning problem
\begin{align*}
\widetilde{\bfu}= \argmin_{\bfu \in \mathbb{R}^M}\frac{1}{|D|} \sum_{i=1}^{|D|}L\big(y_i, \bfu^T\bfphi_{M}(\bfx_i)\big)+\lambda \bfu^T\bfu.
\end{align*}

The following theorem shows the capacity-dependent learning rates for the RF estimator with Lipschitz continuous loss function (Lip-RF), which is sharper than those of the existing literature \citep{li2021towards,li2022sharp} and can be applied to the agnostic setting.

\begin{theorem}\label{thm3}
Under Assumptions \ref{ass4}-\ref{ass6} and \ref{ass7}, if $r\in (0,1]$, $\gamma \in [0,1]$, $2r+\gamma\geq 1$, and $\lambda=|D|^{-\frac{1}{2r+\gamma}}$, when the number of random features satisfies 
\begin{align*}
&M \gtrsim  |D|^{\frac{\alpha}{2r+\gamma}},\quad \text{for} \quad r\in (0,1/2),\\
&M \gtrsim  |D|^{\frac{(2r-1)(1+\gamma-\alpha)+\alpha}{2r+\gamma}},   \quad \text{for} \quad r\in [1/2,1],   
\end{align*}
and $|D|$ is sufficiently large, then there holds
\begin{align*}
\|f^L_{M,D,\lambda}-f^*\|^2_{\rho} ={\cal O}(|D|^{-\frac{2r}{2r+\gamma}}\log |D|),
\end{align*}
with probability near to 1, where $f^*=\argmin_f {\cal E}_L(f)$.      
\end{theorem}

\begin{proof}
Similar to Lemma \ref{lemA.1.1}, we decompose the error for Lip-RF in the following
\begin{align}\label{eqB.25}
 \|f^L_{M,D,\lambda}-f^*\|_{\rho} \leq  \underbrace{\|f^L_{M,D,\lambda}-f_{M,D,\lambda}^{\diamond}\|_{\rho}}_{\text{LS-approximation error}} +\underbrace{  \|f_{M,D,\lambda}^{\diamond}-f_{M,\lambda}\|_{\rho}}_{\text{Empirical error}}+\underbrace{
 \|f_{M,\lambda}-f_{\lambda}\|_{\rho}}_{\text{RF error}}+\underbrace{ \|f_{\lambda}-f_{\tau}^*\|_{\rho}}_{\text{Approximation error}}.
\end{align}

For the last three error terms, we have established their upper bounds in Lemma \ref{lemA.2.1}-\ref{lemA.2.3}. So we only need to bound the first LS-approximation error term.

With the similar argument in the proof of Lemma \ref{lemA.2.4}, we have the similar (adaptive) local strongly convexity condition on $L$ with Assumption \ref{ass7},
\begin{align}\label{eqB.26}
 {\cal E}_{L}(f)-{\cal E}_{L}(f_{M,D,\lambda}^{\diamond})\geq c_3  \|f-f_{M,D,\lambda}^{\diamond}\|^2_{\rho},
\end{align}
or 
\begin{align}\label{eqB.27}
 {\cal E}_{L}(f)-{\cal E}_{L}(f_{M,D,\lambda}^{\diamond})+\|f_{M,D,\lambda}^{\diamond}-f^*\|^2_{\rho}\geq c_4  \|f-f_{M,D,\lambda}^{\diamond}\|^2_{\rho}.
\end{align}

Note that $L$ is Lipschitz continuous, we can also replace $\rho_{\tau}$ with $L$ and obtain a similar inequality with that in \eqref{lemA.2.5.normal}
\begin{equation}\label{eqB.28}
\begin{aligned}
&\left|\frac{1}{|D|}\sum_{i=1}^{|D|}[L(y_i-f(\bfx_i))-L(y_i-f_{M,D,\lambda}^{\diamond}(\bfx_i))]-E[L(y-f(\bfx))-L(y-f_{M,D,\lambda}^{\diamond}(\bfx)]\right|\\
\leq& C\log |D| \mathcal{R}_{M}(\delta)\left(\frac{\|f-f_{M,D,\lambda}^{\diamond}\|_{\rho}}{\delta}+\|f-f_{M,D,\lambda}^{\diamond}\|_{{\cal H}_M}\right). \end{aligned}    
\end{equation}

Using \eqref{eqB.26}-\eqref{eqB.28}, we perform a similar procedure in Lemma \ref{lemA.2.6} with $\rho_{\tau}$ replaced by $L$ and get the upper bound for the LS-approximation error term
\begin{align}\label{eqB.29}
\|f^L_{M,D,\lambda}-f_{M,D,\lambda}^{\diamond}\|_{\rho} \leq C |D|^{-\frac{r}{2r+\gamma}}\log |D|,    
\end{align}
with probability near to 1.

Combining Lemmas \ref{lemA.2.1}-\ref{lemA.2.3}, \eqref{eqB.25} and \eqref{eqB.29}, we have
$$
\|f^L_{M,D,\lambda}-f^*\|^2_{\rho}={\cal O}(|D|^{-\frac{2r}{2r+\gamma}}\log |D|),
$$
with probability near to 1.
\end{proof}

With $\alpha=1$ and $\alpha=\gamma$, we can derive the following corollaries for Lip-RF with uniformly sampling and data-dependent sampling strategies.

\begin{corollary}\label{cor2}
Under Assumptions \ref{ass4}-\ref{ass6} and \ref{ass7}, if random features are sampled according to the uniform strategy, $r\in (0,1]$, $\gamma \in [0,1]$, $2r+\gamma\geq 1$, and $\lambda=|D|^{-\frac{1}{2r+\gamma}}$, when the number of random features satisfies 
\begin{align*}
&M \gtrsim  |D|^{\frac{1}{2r+\gamma}},\quad \text{for} \quad r\in (0,1/2),\\
&M \gtrsim  |D|^{\frac{(2r-1)\gamma+1}{2r+\gamma}},   \quad \text{for} \quad r\in [1/2,1],   
\end{align*}
and $|D|$ is sufficiently large, then there holds
\begin{align*}
\|f^L_{M,D,\lambda}-f^*\|^2_{\rho} ={\cal O}(|D|^{-\frac{2r}{2r+\gamma}}\log |D|)
\end{align*}
with probability near to 1, where $f^*=\argmin_f {\cal E}_L(f)$.      
\end{corollary}

\begin{corollary}\label{cor3}
Under Assumptions \ref{ass4}-\ref{ass6} and \ref{ass7}, if random features are sampled according to the strategy in Example \ref{example1},  $r\in (0,1]$, $\gamma \in [0,1]$, $2r+\gamma\geq 1$, and $\lambda=|D|^{-\frac{1}{2r+\gamma}}$, when the number of random features satisfies 
\begin{align*}
&M \gtrsim  |D|^{\frac{\gamma}{2r+\gamma}},\quad \text{for} \quad r\in (0,1/2),\\
&M \gtrsim  |D|^{\frac{2r+\gamma-1}{2r+\gamma}},   \quad \text{for} \quad r\in [1/2,1],   
\end{align*}
and $|D|$ is sufficiently large, then there holds
\begin{align*}
\|f^L_{M,D,\lambda}-f^*\|^2_{\rho} ={\cal O}(|D|^{-\frac{2r}{2r+\gamma}}\log |D|),
\end{align*}
with probability near to 1, where $f^*=\argmin_f {\cal E}_L(f)$.      
\end{corollary}

\begin{remark}
Note that if we further pose an assumption which is widely used in the literature \citep{feng2024towards}
$$
{\cal E}_L(f)-{\cal E}_L(f^*) \leq C \|f-f^*\|_{\rho}^2,
$$
we can also establish the learning rates for the excess risk of Lip-RF as given by
$$
{\cal E}_L(f^L_{M,D,\lambda})-{\cal E}_L(f^*) \asymp \|f^L_{M,D,\lambda}-f^*\|_{\rho}^2 ={\cal O}(|D|^{-\frac{2r}{2r+\gamma}}\log^2 |D|).
$$    
\end{remark}

\section{Operator Similarities}
In this section, we provide some tight bounds of operator similarities. We first analyze the similarity between $L_K$ and $L_M$.

\begin{lemma}\label{lemC.1}
For any $\delta \in (0,1)$, under Assumption \ref{ass5}, there holds
\begin{align}\label{lemC.1.1}
\|(L_K+\lambda I)^{-1/2}(L_K-L_M)(L_K+\lambda I)^{-1/2}\|   \leq \frac{2({\cal N}_\infty(\lambda)+1)\log (2/\delta)}{M}+\sqrt{\frac{2{\cal N}_\infty(\lambda)\log (2/\delta)}{M}},
\end{align}    
with probability at least $1-\delta$.
\end{lemma}
\begin{proof}
We denote $\phi_{\bfomega}$ as the the function $\phi(\cdot,\bfomega)$ for any $\bfomega \in \Omega$. Note that 
\begin{align*}
L_M=\frac{1}{M}\sum_{i=1}^M\phi_{\bfomega_i}\otimes\phi_{\bfomega_i}  \quad \text{and} \quad L_K=\mathbb{E}_{\bfomega}[\phi_{\bfomega}\otimes\phi_{\bfomega}]. 
\end{align*} 
and Assumption \ref{ass5} that the inequality $\|(L_K+\lambda I)^{-1/2}\phi_{\bfomega}\|^2_{\rho} \leq {\cal N}_{\infty}(\lambda)$ holds almost everywhere. By Lemma \ref{lemD.5} with $Q=L_K$ and $v_i=\phi_{\bfomega_i}$ for $i=1,\ldots,M$ we can get \eqref{lemC.1.1} with  probability at least $1-\delta$. Thus we complete the proof.
\end{proof}

\begin{lemma}\label{lemC.2}
If the number of random features $M \geq 16({\cal N}_\infty(\lambda)+1)\log (2/\delta)$, then under Assumption \ref{ass5}, for any $\delta \in (0,1)$, there holds
\begin{align}\label{lemC.2.1}
\|(L_K+\lambda I)^{-1/2}(L_K-L_M)(L_K+\lambda I)^{-1/2}\|   \leq \frac{1}{2},  
\end{align}
and 
\begin{align}\label{lemC.2.2}
\|(L_M+\lambda I)^{-1/2}(L_K+\lambda I)^{1/2}\|\leq \sqrt{2},   
\end{align}
with probability at least $1-\delta$.
\end{lemma}
\begin{proof}
If  $M \geq 16({\cal N}_\infty(\lambda)+1)\log (2/\delta)$, from Lemma \ref{lemC.1}, we have 
$$
\|(L_K+\lambda I)^{-1/2}(L_K-L_M)(L_K+\lambda I)^{-1/2}\|   \leq \frac{2({\cal N}_\infty(\lambda)+1)\log (2/\delta)}{M}+\sqrt{\frac{2{\cal N}_\infty(\lambda)\log (2/\delta)}{M}}\leq \frac{1}{2}.
$$
By Lemma \ref{lemD.3} with $A=L_K$, $B=L_M$, and $\eta=1/2$, we have 
$$
\|(L_M+\lambda I)^{-1/2}(L_K+\lambda I)^{1/2}\|\leq \sqrt{2}.
$$
Thus we complete the proof.
\end{proof}

\begin{lemma}\label{lemC.2p}
For any $\delta \in (0,1)$, under Assumptions \ref{ass4} and  \ref{ass5}, there holds
\begin{align}\label{lemC.2p1}
 \|(L_K+\lambda I)^{-1/2}(L_K-L_M)\|\leq   \frac{4\kappa\sqrt{{\cal N}_\infty(\lambda)}\log (2/\delta)}{M}+\sqrt{\frac{4\kappa^2{\cal N}(\lambda)\log (2/\delta)}{M}}, 
\end{align}
with probability at least $1-\delta$.
\end{lemma}

\begin{proof}
Let $v_i=z_i=\phi_{\bfomega_i}$ for $i=1,\ldots, M$, then 
$$
T_{M}=\frac{1}{M}\sum_{i=1}^Mv_i\otimes z_i=\frac{1}{M}\sum_{i=1}^M \phi_{\bfomega_i}\otimes \phi_{\bfomega_i}=L_{M},
$$
and
$$
Q=T={\mathbb E}[v \otimes v]={\mathbb E} [v \otimes z]={\mathbb E}[\phi_{\bfomega}\otimes \phi_{\bfomega}]=L_K.
$$
Note that $\|v\|^2=\|\phi_{\bfomega}\|^2\leq \kappa^2$ from Assumption \ref{ass4},
and $\|(L_K+\lambda)^{-1/2}\phi_{\bfomega}\|^2\leq {\cal N}_{\infty}(\lambda)$ from Assumption \ref{ass5}, and ${\cal N}(\lambda)=\text{Tr}((L_K+\lambda I)^{-1}L_K)$, then by Lemma \ref{lemD.6}, we have 
$$
\|(L_K+\lambda I)^{-1/2}(L_K-L_{M})\| \leq  \frac{4\kappa\sqrt{{\cal N}_\infty(\lambda)}\log (2/\delta)}{M}+\sqrt{\frac{4\kappa^2{\cal N}(\lambda)\log (2/\delta)}{M}} .
$$
Thus we complete the proof.    
\end{proof}

\begin{lemma}\label{lemC.3}
Under Assumption \ref{ass5}, for any $\delta \in (0,1)$, there holds
\begin{align}\label{lemC.3.1}
{\cal R}_{M,D,\lambda}\leq \frac{2(\kappa^2\lambda^{-1}+1)\log (2/\delta)}{|D|}+\sqrt{\frac{2\kappa^2\lambda^{-1}\log (2/\delta)}{|D|}},
\end{align}
with probability at least $1-\delta$.
\end{lemma}
\begin{proof}
Recall that ${\cal R}_{M,D,\lambda}=\|(C_M+\lambda I)^{-1/2}(C_M-C_{M,D})(C_M+\lambda I)^{-1/2}\|$, we have 
\begin{align*}
C_{M,D}=\frac{1}{|D|}\sum_{\bfx \in D}{\bfphi}_M(\bfx)\otimes {\bfphi}_M(\bfx)  \quad \text{and} \quad C_M=\mathbb{E}_{\bfx}[{\bfphi}_M(\bfx)\otimes {\bfphi}_M(\bfx)].  
\end{align*} 
Note that 
\begin{equation}\label{lemC.3.2}
\begin{aligned}
\|(C_M+\lambda I)^{-1/2}{\bfphi}_M(\bfx)\|_2^2&\leq \frac{1}{\lambda}\sup_{\bfx \in {\cal X} }\|{\bfphi}_M(\bfx)\|_2^2=\frac{1}{\lambda M}\sup_{\bfx \in {\cal X} }\sum_{i=1}^M|\phi(\bfx,{\bfomega}_i)|^2\\
&\leq \frac{1}{\lambda M}\sum_{i=1}^M\sup_{\bfx \in {\cal X} }|\phi(\bfx,{\bfomega}_i)|^2\leq \frac{1}{\lambda M}\sum_{i=1}^M\sup_{\bfx \in {\cal X}, \bfomega \in \Omega }|\phi(\bfx,{\bfomega}_i)|^2 \leq \kappa^2\lambda^{-1}.
\end{aligned}    
\end{equation}
By Lemma \ref{lemD.5} with $Q=C_M$ and $v_i={\bfphi}_M({\bfx}_i)$ for ${\bfx}_i \in D(\bfx)$, we can get \eqref{lemC.1.1} with  probability at least $1-\delta$. Thus we complete the proof.    
\end{proof} 

\begin{lemma}\label{lemC.4}
If the number of sample $|D| \geq 16(\kappa^2\lambda^{-1}+1)\log (2/\delta)$, then under Assumption \ref{ass5}, for any $\delta \in (0,1)$, there holds
\begin{align}\label{lemC.4.1}
{\cal R}_{M,D,\lambda}  \leq \frac{1}{2},  
\end{align}
and 
\begin{align}\label{lemC.4.2}
{\cal Q}_{M,D,\lambda} \leq \sqrt{2},    
\end{align}
with probability at least $1-\delta$.
\end{lemma}
\begin{proof}
If $|D| \geq 16(\kappa^2\lambda^{-1}+1)\log (2/\delta)$, from Lemma \ref{lemC.3}, we have 
$$
\|(C_M+\lambda I)^{-1/2}(C_M-C_{M,D})(C_M+\lambda I)^{-1/2}\|   \leq \frac{2(\kappa^2\lambda^{-1}+1)\log (2/\delta)}{|D|}+\sqrt{\frac{2\kappa^2\lambda^{-1}\log (2/\delta)}{|D|}}\leq \frac{1}{2}.
$$
By Lemma \ref{lemD.3} with $A=C_M$, $B=C_{M,D}$, and $\eta=1/2$, we have 
$$
\|(C_{M,D}+\lambda I)^{-1/2}(C_M+\lambda I)^{1/2}\|\leq \sqrt{2}.
$$
Thus we complete the proof.
\end{proof}

\section{Technical Lemmas}
\begin{lemma}[Cordes Inequality \citep{furuta2001invitation}]\label{lemD.1}
Let $A$ and $B$ be positive bounded linear operators on a separable Hilbert space. Then, for any $0<\tau\leq 1$, we have 
$$
\|A^{\tau}B^{\tau}\|\leq \|AB\|^{\tau}.    
$$    
\end{lemma}

\begin{lemma}[Proposition 9 in \cite{rudi2017generalization}]\label{lemD.2}
Let $\cal H, \cal K$ be two separable Hilbert spaces and $X, A$ be bounded linear operators, with $X:\mathcal{H}\to\mathcal{K}$ and $A:\mathcal{H} \to \mathcal{H}$ be positive semi-definite, then there holds
$$
\|XA^{s}\|\leq \|X\|^{1-s}\|XA\|^s, \quad \forall s \in [0,1].
$$
\end{lemma}

\begin{lemma}[Lemma E.2 in \cite{blanchard2010optimal}]\label{lemD.3}
For any self-adjoint and positive semi-definite operators $A$ and $B$, if there exists some $\eta \in [0,1]$ such that 
$$
\|(A+\lambda I)^{-1/2}(B-A)(A+\lambda I)^{-1/2}\|\leq 1-\eta,
$$
then we have 
$$
\|(A+\lambda I)^{1/2}(B+\lambda I)^{-1/2}\|\leq \frac{1}{\sqrt{\eta}}.
$$
\end{lemma}

\begin{lemma}[Bernstein’s inequality for sum of random vectors (Proposition 2 in \cite{rudi2017generalization})]\label{lemD.4}
Let $\xi_1,\ldots,\xi_n$ be a sequence of i.i.d random variables on a separable Hilbert space $\cal H$, if there exists $\widetilde{\sigma},\widetilde{B} \geq 0$ such that 
\begin{align}\label{lemD.4.1}
\mathbb{E}\|\xi_i-\mathbb{E}\xi_i\|_{\cal H}^p\leq \frac{1}{2}p!\widetilde{\sigma}^2\widetilde{B}^{p-2}, \quad \forall p\geq2,    
\end{align}
for any $0\leq i\leq n$, then for any $\delta \in (0,1]$, there holds 
$$
\left\|\frac{1}{n}\sum_{i=1}^n\xi_i-\mathbb{E}\xi_i\right\|_{\cal H} \leq \frac{2\widetilde{B}\log (2/\delta)}{n}+\sqrt{\frac{2\widetilde{\sigma}^2\log (2/\delta)}{n}},
$$
with probability at least $1-\delta$. Particularly, \eqref{lemD.4.1} is satisfied if 
$$
\|\xi\|_{\cal H}\leq \frac{\widetilde{B}}{2}, \text{a.s.}\quad \text{and}\quad \mathbb{E}\|\xi\|_{\cal H}^2 \leq \widetilde{\sigma}^2, \quad \text{or}\quad \|\xi-\mathbb{E}\xi\|_{\cal H}\leq \widetilde{B}, \text{a.s.}\quad \text{and}\quad \mathbb{E}\|\xi-\mathbb{E}\xi\|_{\cal H}^2 \leq \widetilde{\sigma}^2.
$$
\end{lemma}

\begin{lemma}[Proposition 5 in \cite{rudi2017generalization}]\label{lemD.6}
Let $\cal H$ and $\cal K$ be two separable Hilbert spaces and $(v_1,z_1),\ldots,(v_n,z_n) \in {\cal H} \times {\cal K}$ for $n\geq 1$ be i.i.d. random variables such that there exists some constant $\tau$ such that $\|v\|_{\cal H}\leq \tau$ and  $\|z\|_{\cal H}\leq \tau$ almost everywhere. Let $Q=\mathbb{E}v\otimes v$ and $T=\mathbb{E}v\otimes z$ and $T_n=\frac{1}{n}\sum_{i=1}^nv_i\otimes z_i$,  then for any $\delta \in (0,1]$, the following holds with probability at least $1-\delta$, 
$$
\left\|(Q+\lambda I)^{-1/2}(T-T_n)\right\|_{HS}\leq \frac{4\tau\sqrt{{\cal Q}_\infty(\lambda)}\log (2/\delta)}{n}+\sqrt{\frac{4\tau^2{\cal Q}(\lambda)\log (2/\delta)}{n}},
$$
\end{lemma}
where ${\cal Q}_\infty(\lambda)=\sup_{v \in {\cal H}}\|(Q+\lambda I)^{-1/2}v\|^2$ and ${\cal Q}(\lambda)=\text{Tr}((Q+\lambda I)^{-1/2}Q)$.

\begin{lemma}[Proposition 6 in \cite{rudi2017generalization}]\label{lemD.5}
Let $v_1,\ldots,v_n$ be a sequence of i.i.d random variables on a separable Hilbert spaces $\cal H$ such that $Q=\mathbb{E}v\otimes v$ is trace class, and for any $\lambda>0$ there exists a constant ${\cal Q}_\infty(\lambda)<\infty$ such that $\left\langle v,(Q+\lambda I)^{-1}v\right\rangle\leq{\cal Q}_\infty(\lambda)$ almost everywhere. Let $Q_n=\frac1n\sum_{i=1}^n v_i\otimes v_i$,  then for any $\delta \in (0,1]$, the following holds with probability at least $1-\delta$,
$$
\left\|(Q+\lambda I)^{-1/2}(Q-Q_n)(Q+\lambda I)^{-1/2}\right\|\leq \frac{2({\cal Q}_\infty(\lambda)+1)\log (2/\delta)}{n}+\sqrt{\frac{2{\cal Q}_\infty(\lambda)\log (2/\delta)}{n}}.
$$
\end{lemma}

\begin{lemma}[Bousquet Inequality]\label{lemD.8}
Let $Z_1, \ldots, Z_n$ be independent random elements taking values in some space $\mathcal{Z}$ and let $\Xi$ be a class of real-valued functions on $\mathcal{Z}$, if we have 
\begin{align*}
\|\xi\| \leq \eta_n \quad \text { and } \quad \frac{1}{n} \sum_{i=1}^n \Var \left(\xi\left(Z_i\right)\right) \leq \zeta_n^2, \quad \forall \xi \in \Xi.
\end{align*}
Define $\boldsymbol{Z}:=\sup _{\xi \in \Xi}\left|\frac{1}{n} \sum_{i=1}^n\left(\xi\left(Z_i\right)-{E} \xi\left(Z_i\right)\right)\right|$. Then for $t>0$
\begin{align*}
P\left(\boldsymbol{Z} \geq E(\boldsymbol{Z})+t \sqrt{2\left(\zeta_n^2+2 \eta_n E(\boldsymbol{Z})\right)}+\frac{2 \eta_n t^2}{3}\right) \leq \exp \left(-n t^2\right) .
\end{align*}    
\end{lemma}

\begin{lemma}[Proposition 10 in \cite{rudi2017generalization}]\label{lemD.7}
If the number of random features $M\geq (4+18 {\cal N}_{\infty}(\lambda))\log(12\kappa^2/\lambda\delta)$, then under Assumption \ref{ass4}, for any $\delta \in (0,1)$, there holds
$$
 |{\cal N}_M(\lambda)-{\cal N}(\lambda)|\leq 1.55 {\cal N}(\lambda),   
$$
with probability at least $1-\delta$.
\end{lemma}

\section{Additional numerical Experiments}\label{sec:exp}

In this section, we conduct some additional numerical experiments on both simulated and real-world data to demonstrate the effectiveness of random features in large kernel quantile learning tasks.

\subsection{Simulated Data} \label{sim_data}
For the simulated data, we consider the following  two data-generating schemes that

\noindent \textbf{(\romannumeral1) Homoscedastic case:}
\begin{align*}
y_{i}=\exp(-x_{i1}+x_{i2})-x_{i2}x_{i3}+\bar{x}_{i}+\epsilon_{i}, \quad i=1,2, \ldots, N.
\end{align*}
\noindent \textbf{(\romannumeral2) Heteroscedastic case:}
\begin{align*}
y_{i}=\sum_{j=1}^3 \beta_i\sin(2\pi x_{ij})+ (1+\bar{x}_{i})\big (\epsilon_{i}-F_{\epsilon}^{-1}(\tau)\big ), \quad i=1,2, \ldots, N.
\end{align*}
In both cases, $x_{ij} \sim U(0,1)$ with $\bar{x}_{i}=\frac{1}{p}\sum_{j=1}^p x_{ij}$, and
$\epsilon_{i}$ follows the standard normal distribution. Moreover, in the heteroscedastic case, $\beta_i \sim U(0, 1)$ and $F^{-1}_\epsilon$ denotes the  quantile function of $\epsilon$. Clearly,  the $\tau$-th conditional quantile of $y$ given $\bfx$ is given by $f_{\tau}^*(\bfx)=f(\bfx)+F_{\epsilon}^{-1}(\tau)$ with $f(\bfx)=\exp(-x_{i1}+x_{i2})-x_{i2}x_{i3}+\bar{x}_{i}$; and in the homoscedastic case, the $\tau$-th conditional quantile of $y$ given $\bfx$ is $f_{\tau}^*(\bfx)=\sum_{j=1}^3 \beta_i\sin(2\pi x_{ij})$.  

\textbf{Parameters setting.} In our simulation,
several scenarios are considered by varying quantile level $\tau$  from $ \{0.1,0.25,0.5,0.75,0.9\}$. In all the scenarios, we employ the standard Gaussian kernel $K(\bfx,\bfx^{\prime})=\exp(-\|\bfx-\bfx^{\prime}\|^2/2)$. As suggested by \citet{rahimi2007random} and \citet{rudi2017generalization}, the corresponding random features are taken as $\phi(\bfx,\bfomega)=\sqrt{2}\cos(\bfomega^T\bfx+b)$, where $\bfomega \sim N(0,I)$ and $b \sim U(0,2\pi)$. The regularization parameter $\lambda$ is selected via a grid search based on a validation set with 1000 samples, where the grid is set as $\{10^{0.5s}: s = -20, -19,..., 2\}$. 

\textbf{Performance evaluation.} To assess the numerical performance of KQR-RF, we use the predicted quantile error (PQE) defined on a testing dataset with $n_{te}=10000$ samples $\{{\bfx}_{te}^i,y_{te}^i\}$ as follows
$$
\widehat{\cal E}_{\tau}(\hat{f})=\frac{1}{n_{te}}\sum_{i=1}^{n_{te}}\rho_{\tau} \big (y_{te}^i-\widehat{f}({\bfx}_{te}^i) \big ).
$$
Our investigation delves into evaluating the influence of several critical factors including the number of random features, the sample size, and the sampling type of random features. All the reported numerical results are obtained from an average of 50 independently repeated experiments.

\subsubsection{Effect of Random feature Size}\label{effect_rf}
In this part, we evaluate the performance of KQR-RF by using various random features. Specifically, we consider the fixed setting that $N=1000$ and vary the number of random features $M \in \{2,5,10,50,100,150\}$. The results of KQR-RF  are summarized in Figure \ref{fig2}.

As indicated in Figure \ref{fig2}, we can conclude that the PQE of  KQR-RF tends to be smaller if a larger number of random features are used. The curves of PQEs become relatively flat as the number of random features reaches $50$, which implies that the marginal gain from those increased random features is limited as enough features have been taken into account.
\begin{figure}
    \centering
    \subfigure[$\tau=0.1$]{
	\includegraphics[width=0.3\linewidth]{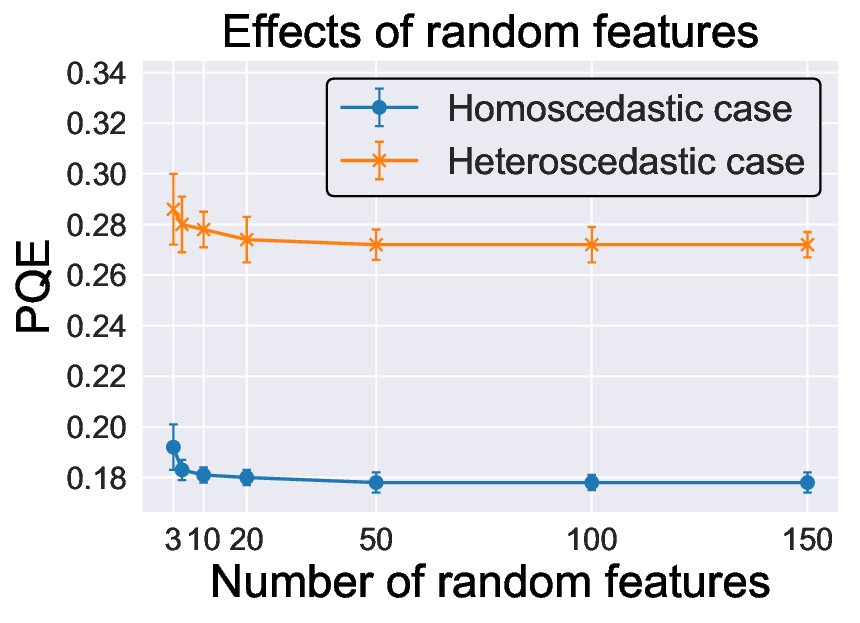}
        \label{label_for_cross_ref_16}
    }
    \subfigure[$\tau=0.25$]{
	\includegraphics[width=0.3\linewidth]{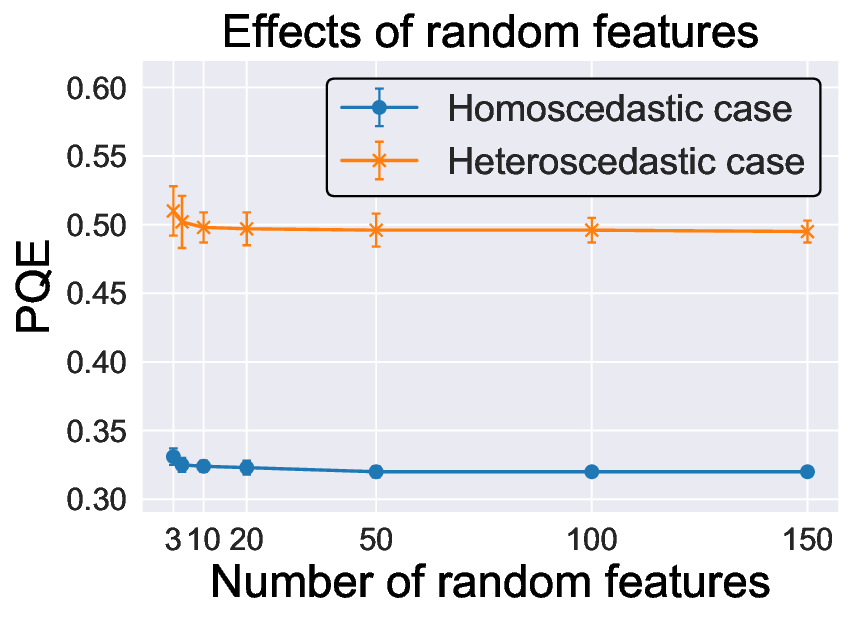}
        \label{label_for_cross_ref_17}
    }
    \subfigure[$\tau=0.5$]{
    	\includegraphics[width=0.3\linewidth]{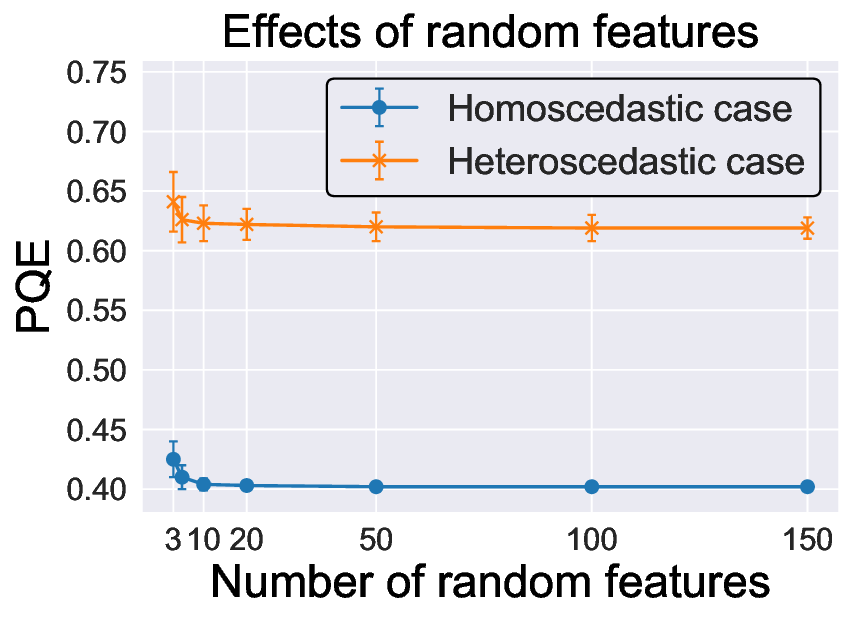}
        \label{label_for_cross_ref_18}
    }
         \subfigure[$\tau=0.75$]{
    	\includegraphics[width=0.3\linewidth]{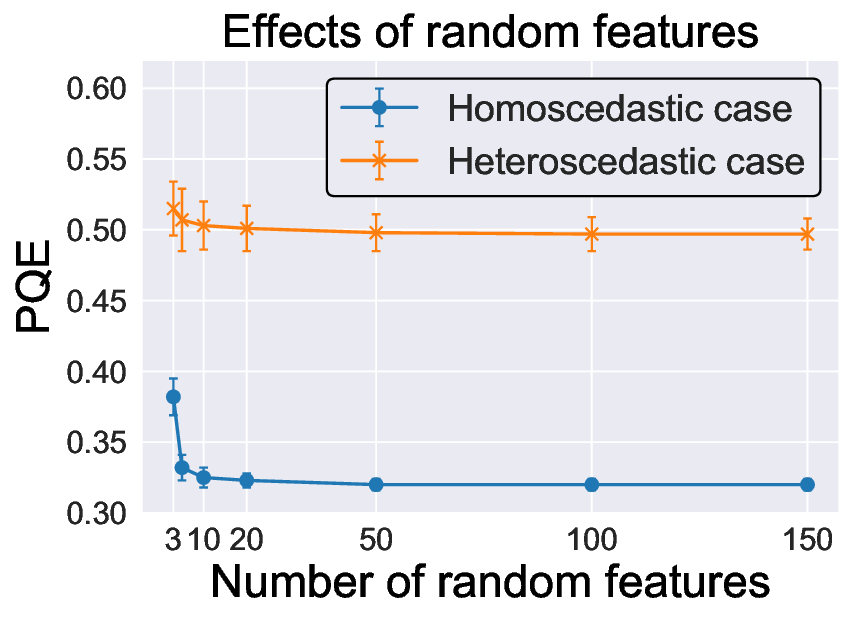}
        \label{label_for_cross_ref_19}
    }
      \subfigure[$\tau=0.9$]{
    	\includegraphics[width=0.3\linewidth]{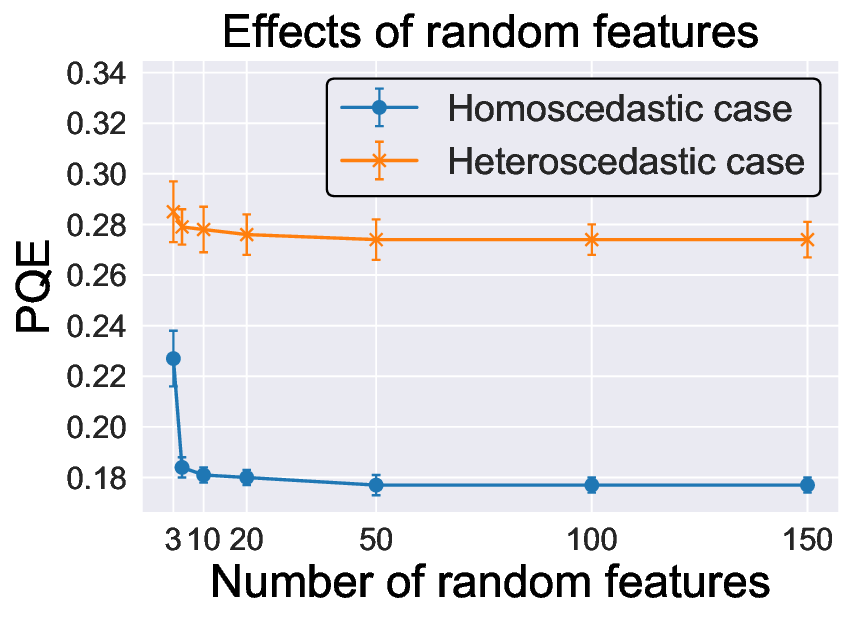}
        \label{label_for_cross_ref_20}
    }

    \caption{Averaged PQE and its standard deviation against the number of random features used in KQR-RF under various scenarios.}
    \label{fig2}
\end{figure}

\subsubsection{Effect of sample size }
In this part, we investigate how the performance of KQR-RF is affected by the sample size $N$, and we also compare it with the exact kernel quantile regression (KQR without random feature).  Specifically, all the settings are exactly the same as those in Section \ref{effect_rf} except that we set $M=50$ and vary  $N \in \{500,1000,2000,5000,10000\}$, respectively.

It is clear from Table \ref{tab.2} that under the homoscedastic case,  the PQE of KQR-RF decreases as $N$ is increased at all the quantile levels, which is consistent with the theoretical result given in Theorem \ref{thm2}. Moreover, the performance of KQR-RF is near to that of KQR as the sample size increases, which shows the consistency of random feature approximation. Similar conclusions can also be drawn in the heteroscedastic case as indicated in Table \ref{tab.3}.

\begin{table}
\centering
\footnotesize
\label{tab.2}
\caption{Averaged PQE and its standard deviation  against  $N$  of different methods in the homoscedastic case.}
\begin{tabular}{ccccccc}
\hline
\multirow{2}{*}{Quantile level  } & \multirow{2}{*}{Method} & \multicolumn{5}{c}{Sample size } \\ \cline{3-7}&& $N=500$ & $N=1000$ & $N=2000$ & $N=5000$ & $N=10000$\\
\hline
\multirow{2}{*}{$\tau=0.1$} &                         KQR-RF         & 0.193(0.014) & 0.185(0.010) & 0.178(0.007) & 0.176(0.005) & 0.174(0.003)
\\ \cline{2-7}
&KQR         & 0.184(0.005) & 0.175(0.004) & 0.168(0.001) & 0.164(0.001) & 0.162(0.001)\\
\hline
\multirow{2}{*}{$\tau=0.25$} &                         KQR-RF         & 0.332(0.007) & 0.325(0.005) & 0.321(0.002) & 0.319(0.002) & 0.317(0.002)
\\ \cline{2-7} 
&KQR         & 0.323(0.002) & 0.315(0.002) & 0.312(0.001) & 0.308(0.001) & 0.306(0.001)\\
\hline
\multirow{2}{*}{$\tau=0.5$} &                         KQR-RF         & 0.419(0.008) & 0.407(0.006) & 0.402(0.002) & 0.400(0.002) & 0.398(0.001)
\\ \cline{2-7} 
&KQR         & 0.405(0.002) & 0.395(0.002) & 0.391(0.002) & 0.388(0.001) & 0.386(0.001)\\
\hline
\multirow{2}{*}{$\tau=0.75$} &                         KQR-RF         & 0.338(0.010) & 0.327(0.008) & 0.322(0.004) & 0.319(0.002) & 0.316(0.002)
\\ \cline{2-7} 
&KQR         & 0.321(0.003) & 0.313(0.003) & 0.306(0.002) & 0.304(0.001) & 0.302(0.001)\\
\hline
\multirow{2}{*}{$\tau=0.9$} &                         KQR-RF         & 0.191(0.016) & 0.183(0.011) & 0.177(0.005) & 0.175(0.004) & 0.174(0.003)
\\ \cline{2-7}
&KQR         & 0.183(0.005) & 0.173(0.004) & 0.168(0.001) & 0.166(0.001) & 0.165(0.001)\\
\hline
\end{tabular}
\end{table}

\begin{table}
\centering
\footnotesize
\label{tab.3}
\caption{Averaged PQE and its standard deviation  against  $N$ and $n$ of different methods in the  heteroscedastic case.}
\begin{tabular}{ccccccc}
\hline
\multirow{2}{*}{Quantile level  } & \multirow{2}{*}{Method} & \multicolumn{5}{c}{Sample size } \\ \cline{3-7}&& $N=500$ & $N=1000$ & $N=2000$ & $N=5000$ & $N=10000$\\
\hline
\multirow{2}{*}{$\tau=0.1$} &                         KQR-RF         & 0.284(0.009) & 0.275(0.008) & 0.268(0.005) & 0.266(0.004) & 0.264(0.003)
\\ \cline{2-7}
&KQR         & 0.271(0.005) & 0.266(0.004) & 0.261(0.001) & 0.259(0.001) & 0.256(0.001)\\
\hline
\multirow{2}{*}{$\tau=0.25$} &                         KQR-RF         & 0.502(0.007) & 0.495(0.006) & 0.487(0.003) & 0.483(0.002) & 0.479(0.002)
\\ \cline{2-7} 
&KQR         & 0.492(0.002) & 0.484(0.002) & 0.479(0.001) & 0.472(0.001) & 0.470(0.001)\\
\hline
\multirow{2}{*}{$\tau=0.5$} &                         KQR-RF         & 0.621(0.009) & 0.614(0.006) & 0.605(0.003) & 0.601(0.002) & 0.598(0.001)
\\ \cline{2-7} 
&KQR         & 0.609(0.002) & 0.601(0.002) & 0.597(0.002) & 0.594(0.001) & 0.591(0.001)\\
\hline
\multirow{2}{*}{$\tau=0.75$} &                         KQR-RF         & 0.501(0.009) & 0.495(0.008) & 0.482(0.004) & 0.479(0.002) & 0.476(0.002)
\\ \cline{2-7} 
&KQR         & 0.493(0.003) & 0.483(0.003) & 0.474(0.002) & 0.467(0.001) & 0.463(0.001)\\
\hline
\multirow{2}{*}{$\tau=0.9$} &                         KQR-RF         & 0.295(0.013) & 0.283(0.011) & 0.270(0.006) & 0.266(0.004) & 0.264(0.003)
\\ \cline{2-7}
&KQR         & 0.284(0.005) & 0.275(0.004) & 0.261(0.001) & 0.256(0.001) & 0.254(0.001)\\
\hline
\end{tabular}
\end{table}

\subsubsection{Effect of sampling strategy}\label{effect_sampling}

In this part, we compare the PQE of KQR-RF with different sampling strategies. Specifically, we consider the following two strategies,
\begin{enumerate}
    \item \textit{Uniform RF:} We generate the random features with uniform sampling strategy in \eqref{in_re}, so the corresponding random features are $\phi(\bfx,\bfomega_i)=\sqrt{2}\cos(\bfomega^T\bfx+b)$, where $\bfomega_i \sim N(0,I)$ and $b \sim U(0,2\pi)$ for $i=1,\ldots,M$.
    \item \textit{Leverage scores RF:} We generate the random features with leverage scores sampling strategy in Example \ref{example1}. By adopting the idea of \citet{sun2018but} and \citet{lioptimalRF2023}, we consider the importance ratio $q(\bfomega_i)=r_i/\sum_{i=1}^Mr_i$, where $\{r_i\}_{i=1}^M$ is the the diagonal of 
    $$
    \bfphi_{M}(\bfX)^T\bfphi_{M}(\bfX)\big(\bfphi_{M}(\bfX)^T\bfphi_{M}(\bfX)+\lambda N I\big)^{-1},
    $$
    with $\bfphi_{M}(\bfX)=(\bfphi_{M}({\bfx}_1),\ldots,\bfphi_{M}({\bfx}_{N}))^{T} \in \mathbb{R}^{N\times M}$. The corresponding random features are then given as $\phi_l(\bfx,\bfomega_i)=[q(\bfomega_i)]^{-1/2}\phi(\bfx,\bfomega_i)$ where $\phi(\bfx,\bfomega_i)$ is the uniform RF.
\end{enumerate}
Specifically, we consider the same settings as those in Section \ref{effect_rf} except that we additionally consider the data-dependent sampling strategy. The results of KQR-RF  are summarized in Figures \ref{fig3}-\ref{fig4} for the homoscedastic and heteroscedastic cases, respectively.

From the results in Figures \ref{fig3}-\ref{fig4}, we can see that both uniform random features and leverage scores random features can achieve better performance as the number of random features increases. The data-dependent sampling strategy is more effective than the uniform sampling strategy with a fixed number of random features, which confirms our theoretical findings in Theorem \ref{thm1} and Corollary \ref{cor1}.

\textbf{Theoretical and empirical leverage scores sampling.} Example \ref{example1} consider a leverage scores sampling strategy by using an importance ratio denoted as $q(\boldsymbol{\omega})=l_{\lambda}(\boldsymbol{\omega})/\int_{\boldsymbol{\omega}}l_{\lambda}(\boldsymbol{\omega})d\pi(\boldsymbol{\omega})$, where $l_{\lambda}(\boldsymbol{\omega})=\|(L_K+\lambda I)^{-1 / 2} \psi(\cdot, \boldsymbol{\omega})\|_{\rho_{\cal X}}^2$. The corresponding parameters $\boldsymbol{\omega}$ are sampled from distribution $\pi_{l}(\boldsymbol{\omega})=q(\boldsymbol{\omega})\pi(\boldsymbol{\omega})$, and random features are $\phi_{l}(\boldsymbol{x},\boldsymbol{\omega})=[q(\boldsymbol{\omega})]^{-1/2}\phi(\boldsymbol{x},\boldsymbol{\omega})$. This reweighted sampling ensure that $K(\boldsymbol{x},{\boldsymbol{x}}^{\prime})=E_{\boldsymbol{\omega} \sim \pi_{l}(\boldsymbol{\omega})}[\langle\phi_{l}(\boldsymbol{x},\boldsymbol{\omega}),\phi_{l}({\boldsymbol{x}}^{\prime},\boldsymbol{\omega})\rangle]$. Here, we want to emphasize that Example \ref{example1} is a theoretical construction, and the data dependent sampling scheme is implicit due to the integral operator $L_K$ and the expectation with respect to $\boldsymbol{x}$ in $\|(L_K+\lambda I)^{-1 / 2} \psi(\cdot, \boldsymbol{\omega})\|_{\rho_{\cal X}}^2$.  However, in the literature, there are a lot of empirical leverage score sampling strategies that highly depend on the data. For example, the empirical random features leverages scores $\widehat{l}_{\lambda}(\boldsymbol{\omega})=\widehat{\Xi}(\boldsymbol{\omega})^T(\mathbf{K}+\lambda I)^{-1}\widehat{\Xi}(\boldsymbol{\omega})^T$, with $\widehat{\Xi}(\boldsymbol{\omega}) \in \mathbb{R}^{|D|}$, $(\widehat{\Xi}(\boldsymbol{\omega}))_i=\phi_M(\boldsymbol{x}_i)$ and $\mathbf{K}=\{k(\boldsymbol{x}_i,\boldsymbol{x}_j)\}_{ij}$ is the data kernel matrix (see Remark 4 in  \citet{rudi2017generalization}). There are also some approximate leverage score sampling strategies to save the computation cost, see in \citet{sun2018but,li2021towards}. 

\begin{figure}
    \centering
    \subfigure[$\tau=0.1$]{
	\includegraphics[width=0.3\linewidth]{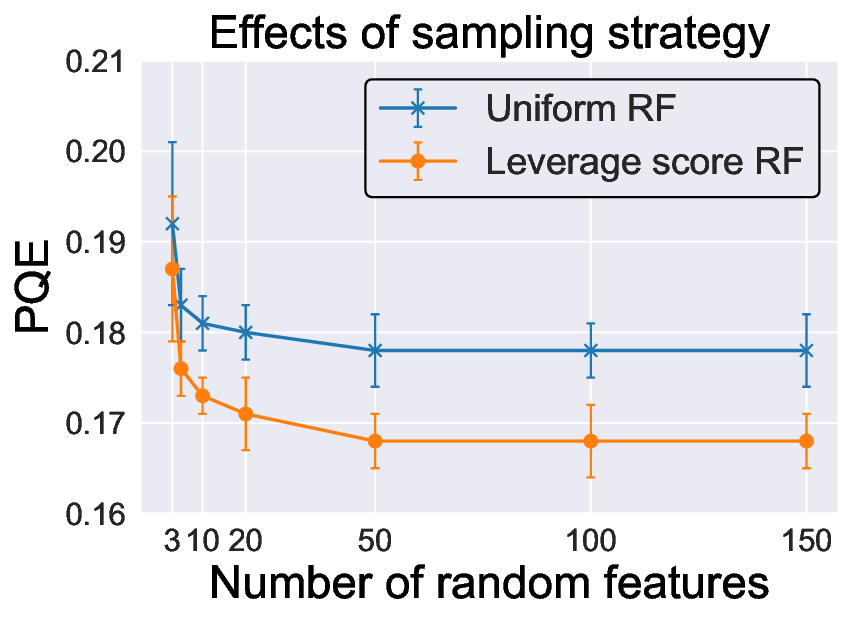}
        \label{label_for_cross_ref_26}
    }
    \subfigure[$\tau=0.25$]{
	\includegraphics[width=0.3\linewidth]{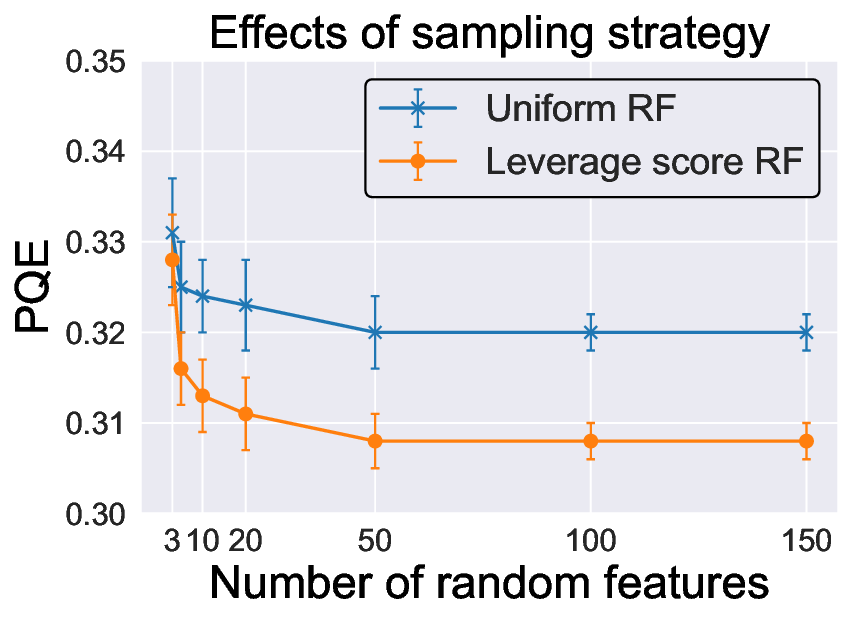}
        \label{label_for_cross_ref_27}
    }
    \subfigure[$\tau=0.5$]{
    	\includegraphics[width=0.3\linewidth]{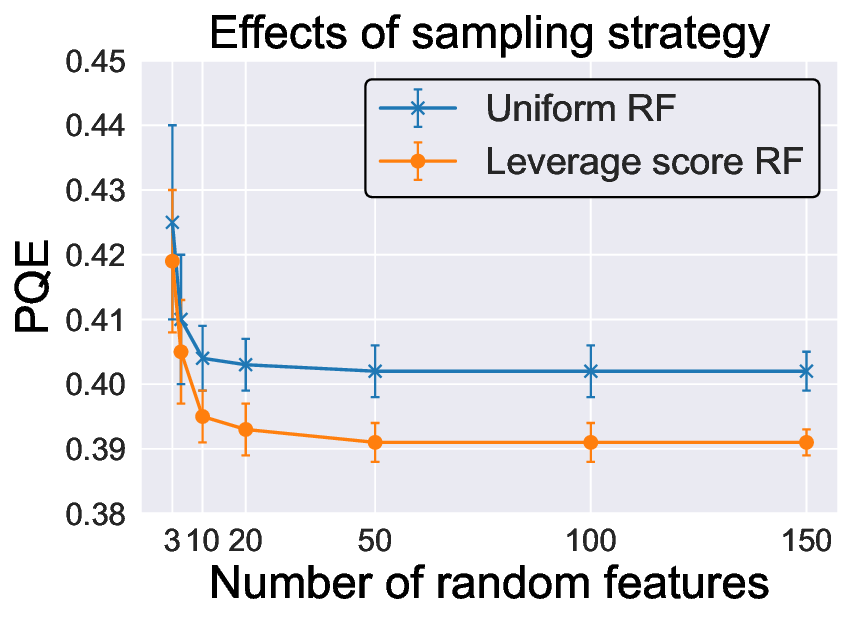}
        \label{label_for_cross_ref_28}
    }
         \subfigure[$\tau=0.75$]{
    	\includegraphics[width=0.3\linewidth]{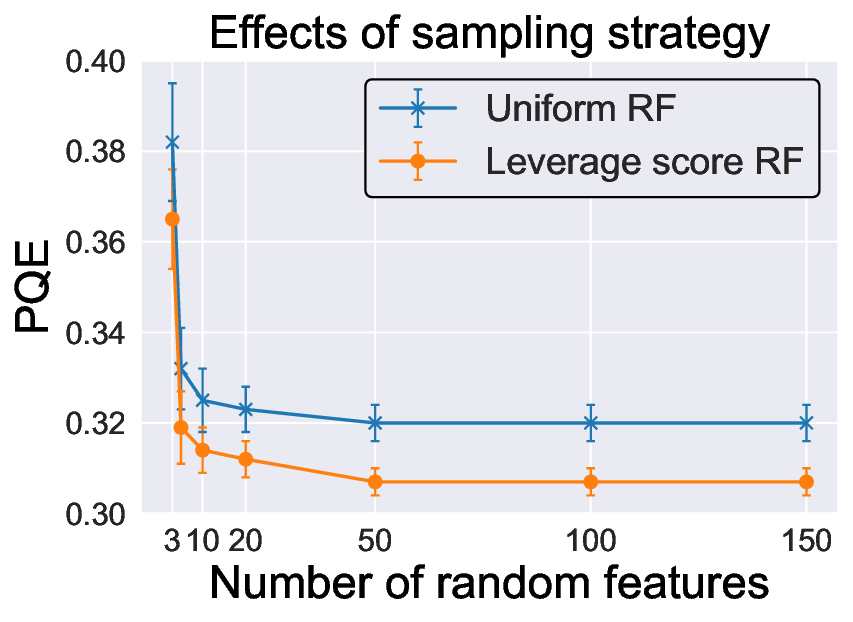}
        \label{label_for_cross_ref_29}
    }
      \subfigure[$\tau=0.9$]{
    	\includegraphics[width=0.3\linewidth]{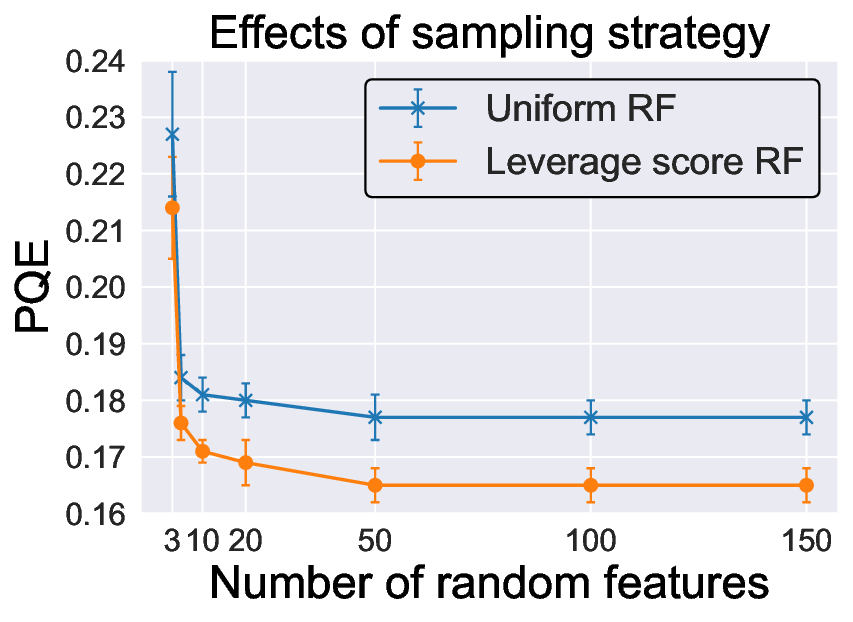}
        \label{label_for_cross_ref_30}
    }

    \caption{Averaged PQE and its standard deviation against the number of random features used in KQR-RF for different sampling strategies in the homoscedastic case.}
    \label{fig3}
\end{figure}

\begin{figure}
    \centering
    \subfigure[$\tau=0.1$]{
	\includegraphics[width=0.3\linewidth]{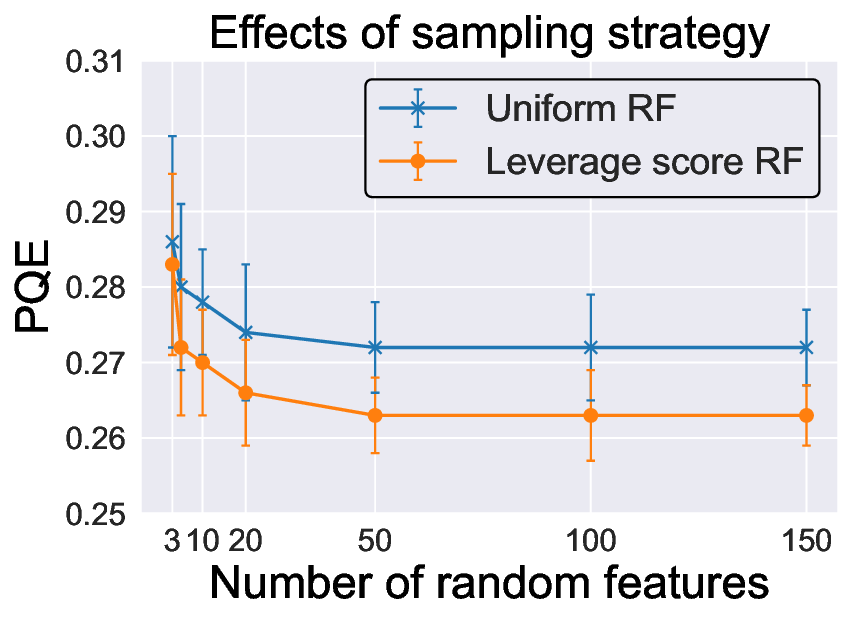}
        \label{label_for_cross_ref_21}
    }
    \subfigure[$\tau=0.25$]{
	\includegraphics[width=0.3\linewidth]{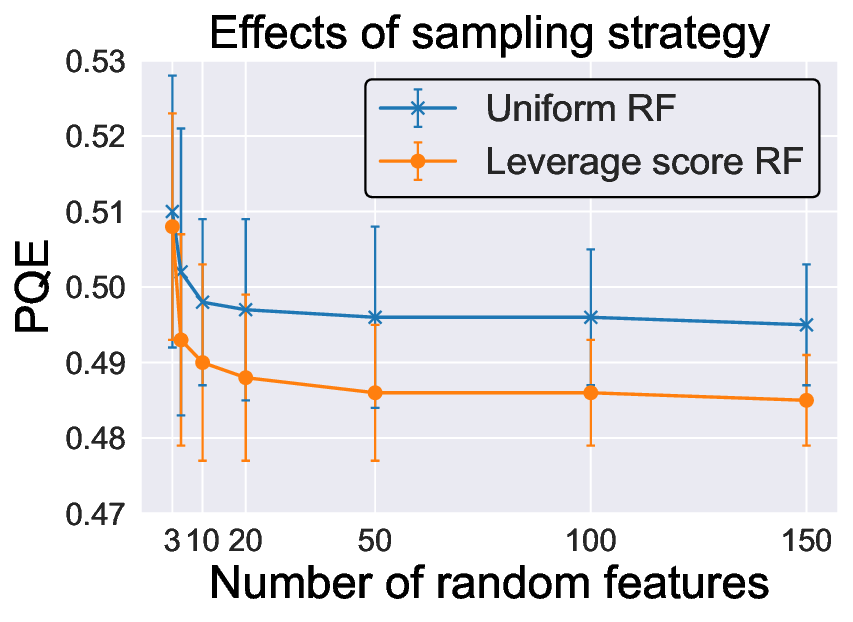}
        \label{label_for_cross_ref_22}
    }
    \subfigure[$\tau=0.5$]{
    	\includegraphics[width=0.3\linewidth]{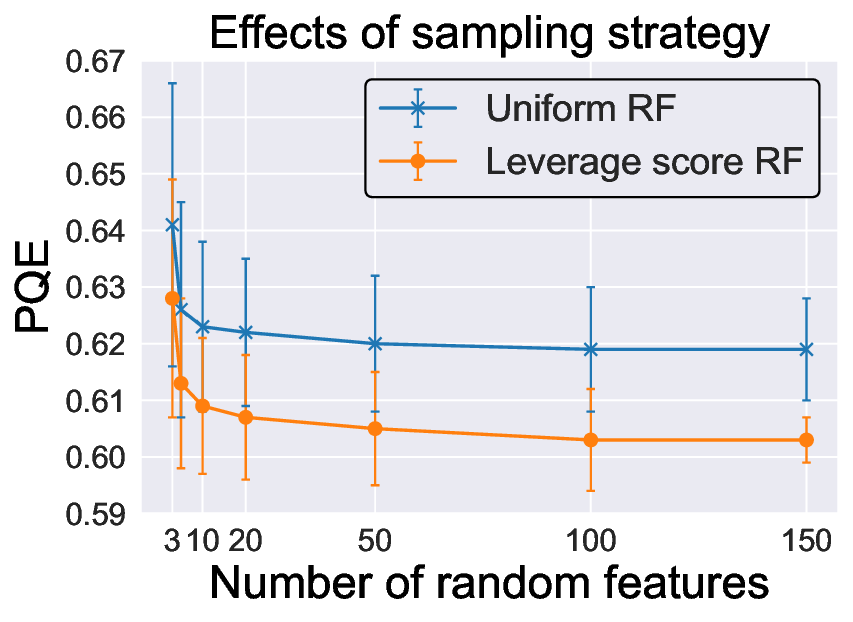}
        \label{label_for_cross_ref_23}
    }
         \subfigure[$\tau=0.75$]{
    	\includegraphics[width=0.3\linewidth]{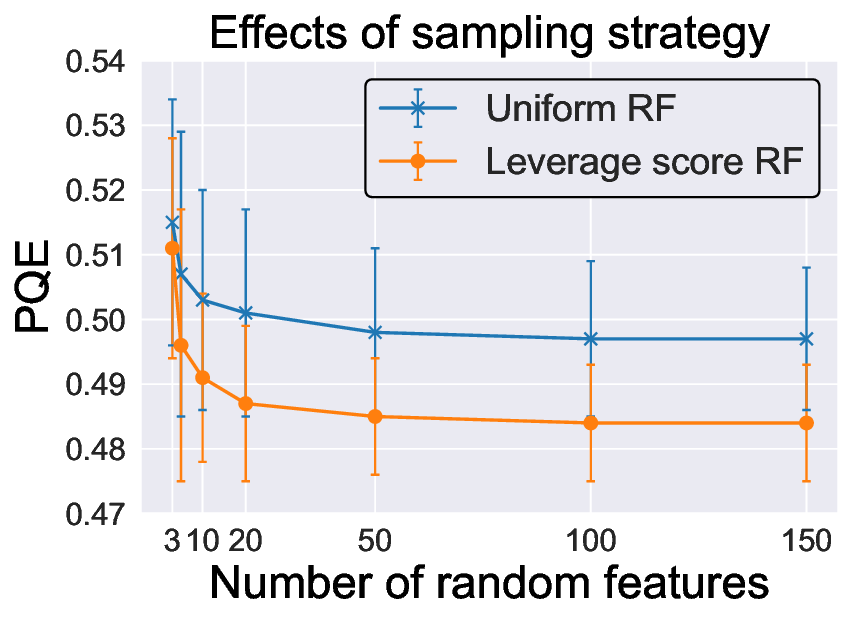}
        \label{label_for_cross_ref_24}
    }
      \subfigure[$\tau=0.9$]{
    	\includegraphics[width=0.3\linewidth]{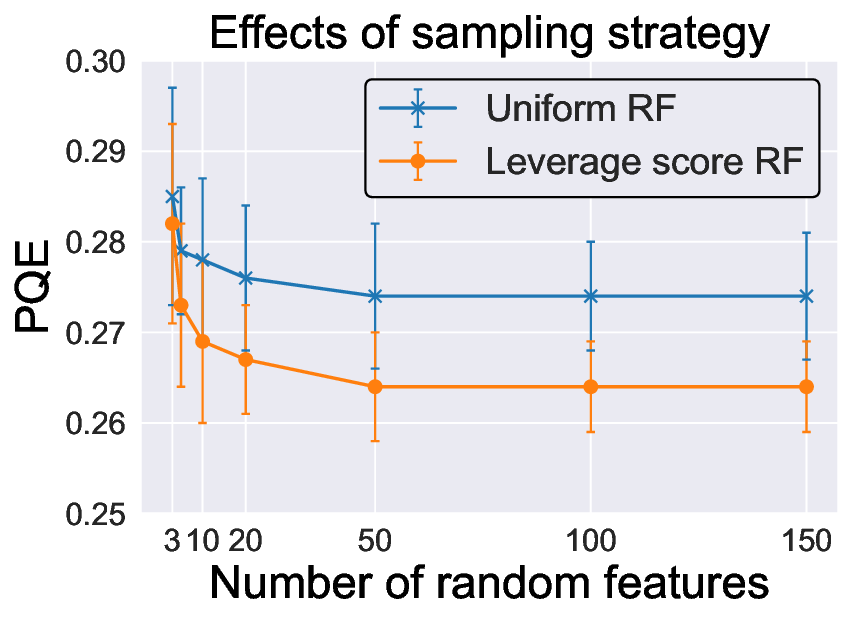}
        \label{label_for_cross_ref_25}
    }

    \caption{Averaged PQE and its standard deviation against the number of random features used in KQR-RF for different sampling strategies in the heteroscedastic case.}
    \label{fig4}
\end{figure}

\subsection{Real Case Study}
In this study, we consider the UK used car prices dataset from Kaggle (\url{https://www.kaggle.com/datasets/kukuroo3/used-car-price-dataset-competition-format}).  In the raw dataset, there are $N=7632$ samples after excluding those with missing values. The response is the price of each used car, while the covariates include crucial information about the used cars, such as the registration year, mileage, road tax, miles per gallon (mpg), and engine size. We mean to predict the prices of used cars, thereby assisting car buyers in making optimal purchasing decisions.

\begin{figure}[!htbp]
    \centering
    \subfigure[Raw price]{
	\includegraphics[width=0.4\linewidth]{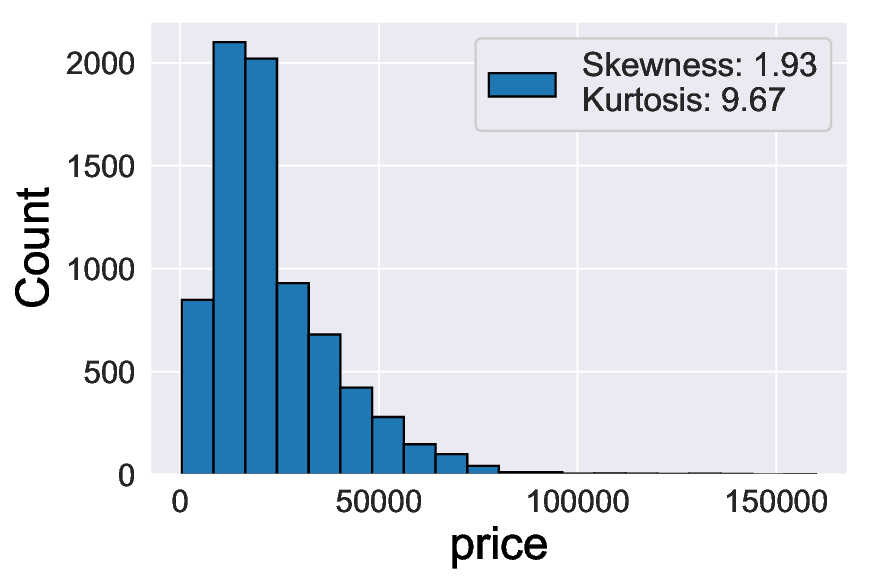}
        
    }
    \subfigure[Log-transformed price]{
    	\includegraphics[width=0.4\linewidth]{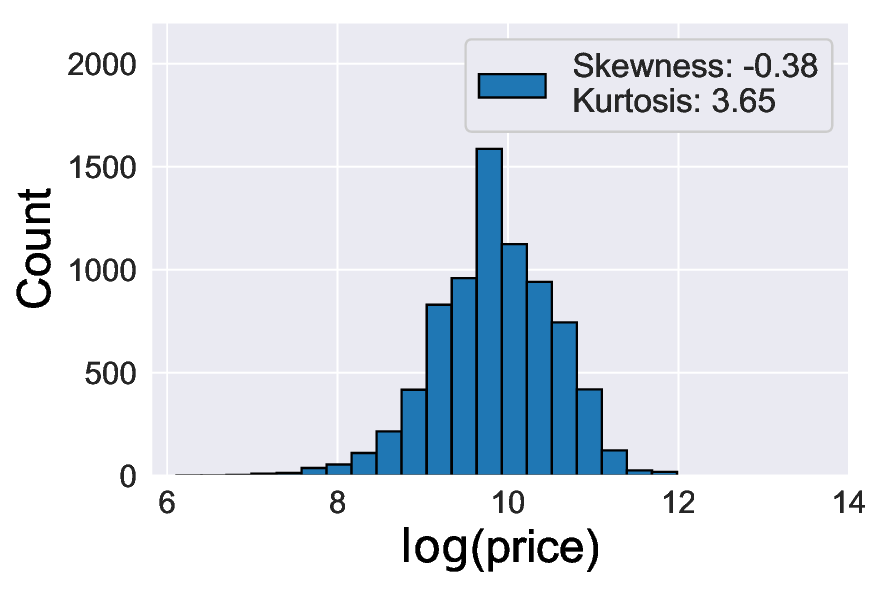}
 
    }

    \caption{The Histogram of the used car price and the log-transformed price.}
    \label{price_fig}
\end{figure}

In Figure \ref{price_fig}, we plot the histogram of the used car price, as well as the skewness and kurtosis, revealing a notable right-skew in the raw price distribution. Although this skewness is partially mitigated by applying a log transformation to the price, there is still a tail dragging on the right. To obtain more robust estimates, we consider the quantile regression with the log-transformed price as the response. Note that the sample size is large, thus it is natural to use random features to save the computing cost. 
In our experimental setup, we randomly choose $N_{tr}=5000$ samples as the training data and assume they are randomly distributed, $N_{va}=1000$ samples as the validation data, and the rest as the testing data. For each dataset, we perform the min-max normalization for each covariate, i.e., $x_{ij}$  is rescaled by $x_{ij}=(x_{ij}-x_{min}^i)/(x_{max}^i-x_{min}^i)$, where $x_{min}^i$ and $x_{max}^i$ is the minimum and maximum values of the $i$-th covariate within the entire dataset. Following the suggestion in Section \ref{effect_rf}, we select the number of random features $M=100$, and the choices of regularization parameters $\lambda$ are the same as that in Section  \ref{sim_data}.

Considering that car buyers primarily focus on the average price of used cars, we only consider the quantile level $\tau=0.5$. Table \ref{tab.4} depicts the averaged PQE and its standard deviation (50 repeats) of three methods, including the exact KQR, KQR-RF with uniform random features, and KQR-RF with leverage scores random features. Clearly, two random feature methods exhibit close performance compared to the exact KQR, especially when we use the data-dependent random features. These results further support the effectiveness of  KQR-RF and substantiate the theoretical results provided in the main text. 

\begin{table}
\centering
\footnotesize
\label{tab.4}
\caption{Averaged PQE and its standard deviation for $\tau=0.5$ of different methods in used car price dataset.}
\begin{tabular}{cccc}
\hline
Method& Exact KQR& KQR-RF(Uniform RF)&KQR-RF(Leverage scores RF)\\
\hline
PQE& 0.094(0.000) & 0.105(0.004) & 0.098(0.003)
\\
\hline
\end{tabular}
\end{table}

\subsection{True quantile functions in the simulation of the main text}

To graphically show the quantile function at different quantile levels in the simulation of the main text , we consider three different settings: (1) worst case ($r=0,\gamma=1$); (2) general case ($r=1/2, \gamma=1$); (3) most benign case ($r=1, \gamma=0$). We generate data with size $N=200$, and plot the quantile curves for $\tau \in \{0.25,0.5,0.75\}$ in Figure \ref{fig_1}.

\begin{figure}
    \centering 
   \setlength{\abovecaptionskip}{0.cm}    \includegraphics[width=1\textwidth]{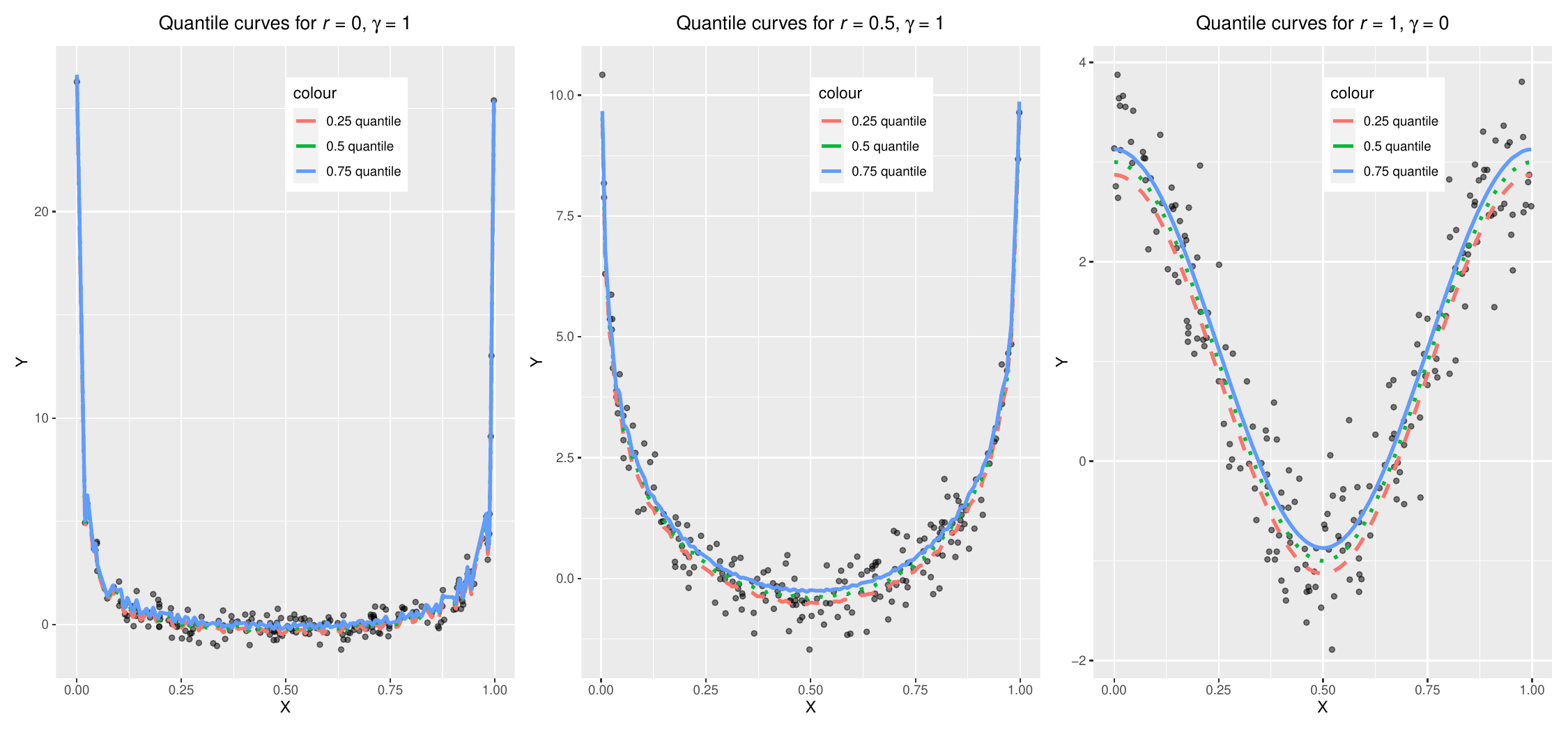} 
    \caption{True quantile curves for $r=0,\gamma=1$ (left), $r=1/2, \gamma=1$ (middle), and $r=1, \gamma=0$ (right). }
    \label{fig_1}
\end{figure}

\end{document}